\documentclass[final,12pt,3p]{elsarticle}

\usepackage[colorlinks=true]{hyperref}

\usepackage{amssymb}

\usepackage{lipsum}
\usepackage[labelsep=period,singlelinecheck=off,textformat=period]{caption}
\usepackage{subcaption}
\usepackage{fancyhdr} 
\usepackage{epsfig}
\usepackage{timesmt}
\usepackage{mathtools,nccmath}
\usepackage{mathrsfs}
\usepackage{wrapfig}
\usepackage[noabbrev]{cleveref}
\creflabelformat{equation}{#2\textup{#1}#3}

\usepackage{pdflscape}
\usepackage{afterpage}
\usepackage{float}
\usepackage{url}
\usepackage{enumitem}
\usepackage{accents}
\usepackage{cuted}
\usepackage{dirtytalk}
\usepackage{multirow}
\usepackage{multicol}
\usepackage{booktabs}
\usepackage{array}
\usepackage{makecell}
\usepackage{colortbl}
\usepackage{soul}
\newcolumntype{n}{>{\columncolor{blue!5}}c}

\usepackage{pbox}
\usepackage[ruled,commentsnumbered,linesnumbered,boxed]{algorithm2e}
\SetArgSty{textnormal}
\SetKwComment{Comment}{/* }{ */}

\usepackage[flushleft]{threeparttable}
\usepackage{algpseudocode}
\usepackage{setspace}
\usepackage{courier}
\usepackage{psfrag,pstool}

\newcommand{\comment}[1]{}

\Crefname{equation}{Eq.}{Eqs.}
\Crefname{figure}{Fig.}{Figs.}
\Crefname{tabular}{Tab.}{Tabs.}
\crefname{algocf}{alg.}{algs.}
\Crefname{algocf}{Algorithm}{Algorithms}
\Crefname{fct}{Fact}{Facts }

\usepackage{amsmath}
\usepackage{amssymb}
\usepackage{amsthm}
\usepackage{mathrsfs}

\usepackage{bm}
\usepackage{mathtools}
\usepackage{nicefrac}

\theoremstyle{plain}
\newtheorem{thm}{Theorem}[section]

\newtheorem{prop}[thm]{Proposition}

\theoremstyle{definition}

\theoremstyle{remark}
\newtheorem{rem}{Remark}

\newcommand{\prob}[1]{\rho_{#1}}

\newcommand{\latin}[1]{\textit{#1}}
\renewcommand{\Re}{\mathbb{R}}

\newcommand{\bmm}{{\mathbf b}}

\newcommand{\jm}{\bm{j}}

\newcommand{\uu}{{\mathbf u }}

\newcommand{\vm}{{\mathbf v }}

\newcommand{\x}{{\mathbf x }}

\newcommand{\X}{{\mathbf X}}

\newcommand{\y}{{\mathbf y}}
\newcommand{\Y}{{\mathbf Y}}

\newcommand{\z}{{\mathbf z}}
\newcommand{\Z}{{\mathbf Z}}

\newcommand{\xii}{\boldsymbol{\xi}}

\newcommand{\ie}{i.e., }

\newcommand{\supth}[1]{\ensuremath{{#1}^{\text{th}}}}

\usepackage{pgfplots}
\pgfplotsset{compat=1.16}
\usepackage{tikz}

\usetikzlibrary{shapes.misc}
\usetikzlibrary{arrows,arrows.meta,calc,backgrounds}

\tikzstyle{block} = [draw,rectangle,thick,minimum height=2em,minimum width=2em]
\tikzstyle{sum} = [draw,circle,inner sep=0mm,minimum size=2mm]
\tikzstyle{connector} = [->,thick]
\tikzstyle{line} = [thick]
\tikzstyle{branch} = [circle,inner sep=0pt,minimum size=1mm,fill=black,draw=black]
\tikzstyle{guide} = []
\tikzset{>=latex}

\makeatletter
\newcommand{\hathat}[1]{%
	\begingroup%
	\let\macc@kerna\z@%
	\let\macc@kernb\z@%
	\let\macc@nucleus\@empty%
	\hat{\raisebox{.35ex}{\vphantom{\ensuremath{#1}}}\smash{\hat{#1}}}%
	\endgroup%
}
\makeatother

\creflabelformat{equation}{#2(#1)#3}
\crefrangelabelformat{equation}{#3(#1)#4 to #5(#2)#6}

\labelcrefformat{subequation}{#2(#1)#3}
\labelcrefrangeformat{subequation}{#3(#1)#4 to #5(#2)#6}

\renewcommand{\mathbf}[1]{\bm{#1}}

\graphicspath{{./tikzpics/}{./figures/}}

\usepackage{lineno}

\journal{Computer Methods in Applied Mechanics and Engineering}

\begin{document}

\begin{frontmatter}



\title{Conditional flow matching for physics-constrained inverse problems with finite training data}


\author[label1,label2]{Agnimtra Dasgupta} 
\author[label1]{Ali Fardisi}
\author[label1]{Mehrnegar Aminy}
\author[label1]{Brianna Binder}
\author[label1]{Bryan Shaddy}
\author[label1]{Saeed Moazami}
\author[label1]{Assad A Oberai\corref{cor1}}
\cortext[cor1]{Corresponding author}
\ead{aoberai@usc.edu}

\affiliation[label1]{organization={{Department of Aerospace \& Mechanical Engineering, University of Southern California}},
city={Los Angeles},
postcode={90089}, 
state={California},
country={USA}}

\affiliation[label2]{organization={{Optimization and Uncertainty Quantification, Sandia National Laboratories}},
city={Albuquerque},
postcode={87123}, 
state={New Mexico},
country={USA}}

\begin{abstract}
This study presents a conditional flow matching framework for solving physics-constrained Bayesian inverse problems. In this setting, samples from the joint distribution of inferred variables and measurements are assumed available, while explicit evaluation of the prior and likelihood densities is not required. We derive a simple and self-contained formulation of both the unconditional and conditional flow matching algorithms, tailored specifically to inverse problems. In the conditional setting, a neural network is trained to learn the velocity field of a probability flow ordinary differential equation that transports samples from a chosen source distribution directly to the posterior distribution conditioned on observed measurements. This black-box formulation accommodates nonlinear, high-dimensional, and potentially non-differentiable forward models without restrictive assumptions on the noise model. We further analyze the behavior of the learned velocity field in the regime of finite training data. Under mild architectural assumptions, we show that overtraining can induce degenerate behavior in the generated conditional distributions, including variance collapse and a phenomenon termed selective memorization, wherein generated samples concentrate around training data points associated with similar observations. A simplified theoretical analysis explains this behavior, and numerical experiments confirm it in practice. We demonstrate that standard early-stopping criteria based on monitoring test loss effectively mitigate such degeneracy. The proposed method is evaluated on a range of problems, including conditional density estimation benchmarks, an inverse problem motivated by data assimilation for the Lorenz-63 system, physics-based inverse problems governed by partial differential equations, and problems where the data is measured experimentally. We investigate the impact of different choices of source distributions, including Gaussian and data-informed priors, and quantify performance using optimal transport–based metrics, posterior statistics, and sampling efficiency. Across these examples, conditional flow matching accurately captures complex, multimodal posterior distributions while maintaining computational efficiency. 
\end{abstract}






\end{frontmatter}


\section{Introduction}
Inverse problems are ubiquitous across a wide range of scientific and engineering disciplines. They are typically ill-posed and are often driven by measurements contaminated with noise. A principled way to address these challenges is through a Bayesian formulation of the inverse problem \cite{stuart2010inverse,calvetti2018inverse}. Within this framework, a prior distribution is specified for the quantities to be inferred, and this distribution is updated {to obtain the posterior upon observing measurements via the likelihood. The likelihood incorporates the forward model, which maps the unknown quantities to the observed measurements and encodes the assumed measurement noise model.

Solving Bayesian inverse problems becomes particularly challenging when the dimensions of the inferred variables and measurements are large, when either the prior or the likelihood is non-Gaussian, when the forward model is nonlinear and computationally complex, when the prior distribution is available only through samples, (\ie the prior is \emph{data-driven}), or when either the prior or likelihood cannot be evaluated but sampled from (\ie the prior or likelihood are \emph{intractable}). Under such conditions, posterior approximation techniques encounter significant difficulties. These include methods that rely on Gaussian approximations of the posterior distribution~\cite{tierney1986accurate}, as well as sampling-based approaches such as Markov chain Monte Carlo (MCMC) and its variants \cite{brooks1998markov,neal2011mcmc}, which may suffer from poor scalability and slow convergence in high-dimensional settings.

To address these challenges, a class of novel methods has emerged that casts Bayesian inverse problems as conditional generative modeling tasks. These approaches adapt and leverage deep generative models, including generative adversarial networks (GANs) \cite{dimakis2022deep, lunz2018adversarial, ray2022efficacy, duff2024regularising, patel2022solution,ray2023solution}, normalizing flows (NFs) \cite{whang2021composing, hagemann2022stochastic,dasgupta2024dimension}, diffusion models \cite{daras2024survey,chung2022improving,wang2024dmplug, batzolis2021conditional, jacobsen2025cocogen, dasgupta2025conditional,dasgupta2026unifying}, and, more recently, methods based on flow matching and stochastic interpolants \cite{zhang2024flow,pourya2025flower, tauberschmidt2025physics,utkarsh2025physics}. In all of these approaches, samples are first drawn from a simple source distribution and are then transformed so as to approximate samples from the target posterior distribution. 

In GAN-based methods \cite{goodfellow2014generative, arjovsky2017wasserstein}, the source distribution is typically a low-dimensional Gaussian or uniform distribution, and the transformation is implemented via a single pass through a generator network. The generator is trained through a min–max optimization problem, which can render the training process unstable and difficult to interpret.  In contrast, normalizing flow–based methods \cite{kobyzev2020normalizing} usually transform a Gaussian distribution of the same dimensionality as the inferred variables. Discrete normalizing flows~\cite{dinh2014nice,dinh2016density} require specially designed invertible networks for such transformations, which carry large memory footprints in high-dimensional problems. Alternatively, continuous normalizing flows~\cite{chen2018neural} transform samples from a Gaussian distribution by integrating an ODE, where a non-invertible neural network approximates the vector field. However, the objective used to train continuous normalizing flows requires evaluating the log-determinant of the Jacobian of the transformation and its gradients (using backpropagation), which are both computationally expensive in high dimensions. Diffusion-based models \cite{ho2020denoising, sohl2015deep, song2020score} also transform samples from a Gaussian distribution, but do so by integrating an ODE or stochastic differential equation (SDE) whose drift term involves a neural network approximation of the score function. This network can be implemented using standard architectures and, unlike continuous normalizing flows, is trained using a regression-based loss, which makes diffusion models relatively easy to train and scalable to high-dimensional settings. Flow matching–based methods \cite{lipman2022flow, liu2022flow,albergo2022building} share several similarities with diffusion models, including simple network architectures, regression-based training objectives, and ODE-based sampling procedures. However, instead of learning a score function, flow matching methods learn the velocity field that transports samples from the source distribution to the target distribution. An important advantage of these methods is their flexibility in the choice of source distribution: rather than requiring a specific parametric form, they only require access to samples from the source.  This work investigates flow matching for Bayesian inference in physics-constrained inverse problems.

There are two broad strategies for applying flow matching methods to probabilistic inverse problems. In the first approach \cite{zhang2024flow,pourya2025flower, tauberschmidt2025physics, utkarsh2025physics,parikh2026d}, an unconditional flow matching algorithm is used to learn a velocity field that maps a Gaussian source distribution to the prior distribution. During posterior sampling, a modification to the velocity field corresponding to the likelihood term is incorporated. This likelihood-induced velocity field is typically not learned via a neural network; instead, it is derived analytically by penalizing the residual of the forward operator or by imposing it as a hard constraint. A key advantage of this approach is that once the prior velocity field has been learned, it can be reused for arbitrary likelihoods and forward models. However, this approach relies on restrictive assumptions about the noise model and requires the forward operator to be differentiable and sufficiently simple to permit efficient computation of its derivatives.

The second approach, which is adopted in this paper and  remains relatively unexplored for physics-based inverse problems~\cite{wildberger2023flow}, 
employs the conditional variant of the flow matching algorithm. In this setting, a neural network is trained to learn the velocity field that maps samples from the source distribution directly to the distribution of the inferred variables conditioned on the observed measurements. This approach does not impose explicit assumptions on the measurement noise model and interfaces with the forward model in a black-box manner, allowing it to accommodate highly complex and potentially non-differentiable forward models. 
Its primary limitation is that the learned velocity field is specific to the forward model used during training; consequently, changes to the forward model necessitate retraining the conditional flow matching network, even if the prior distribution remains unchanged.

The contributions of the present work, which are rooted in the second approach, extend the existing literature in the following ways:
\begin{enumerate}
\item We consider the application of the conditional version of the flow matching algorithm to a variety of physics-driven inverse problems, and quantify the performance of this approach.  Our presentation includes a simple and self-contained derivation of the conditional flow matching algorithm inspired by the general framework introduced in \cite{albergo2022building} but tailored specifically to inverse problems.

\item We provide a theoretical analysis of the conditional flow matching algorithm in the regime of finite training data. We show that, for velocity networks with sufficient expressive capacity and prolonged training, the learned velocity field can induce degenerate behavior in the generated samples. This includes vanishing variance in some cases and what we refer to as conditional memorization in some other cases. We also present numerical evidence illustrating the emergence of such degenerate behavior in practice. 
\item We demonstrate empirically that terminating training the velocity field according to standard early-stopping criteria used to prevent overfitting in regression problems can effectively mitigate this degeneracy.
\item In addition to applying the proposed method to complex, high-dimensional inverse problems, we also consider simplified settings for which the true posterior distribution is known, enabling a quantitative assessment of the performance of conditional flow matching. 
\end{enumerate}

The remainder of this paper is organized as follows. In \Cref{sec:flow-matching-intro}, we introduce the probabilistic inverse problem, and present a simple and self-contained derivation of the flow-matching algorithm and its conditional variant. In 
\Cref{sec:overfitting}, we further analyze the velocity field learned by the algorithm under finite-data constraints and examine the impact of this learned velocity on the generated samples. In \Cref{sec:results}, we apply the proposed method to a range of inverse problems. These include both synthetic examples, for which the true distribution is known, and more complex cases driven by real-world experimental data. Through these studies, we investigate the effects of overfitting the velocity field and demonstrate that monitoring the test loss provides an effective strategy for mitigating this issue. We conclude in \Cref{sec:conclusions} with a summary of our findings and directions for future work.

\section{Derivation of the flow matching algorithm}\label{sec:flow-matching-intro}

We denote the vector of variables to be inferred by $\X \in \Re^d$ and the vector of measurements by $\Y \in \Re^D$. Both are treated as random vectors. The prior density of $\X$ is denoted by $\prob{\X}(\x)$. The forward model is allowed to be probabilistic and defines the conditional distribution of $\Y$ given $\X = \x$, \ie it specifies the conditional density $\prob{\Y|\X}(\y \mid \x)$. We do not assume explicit knowledge of the functional forms of $\prob{\X}$ or $\prob{\Y|\X}$. Instead, we impose weaker assumptions. For the prior, we assume access to samples drawn from $\prob{\X}$. For the forward model, we assume that, given any value $\X = \x$, we can generate a sample of the measurement from $\prob{\Y|\X}$. Since the joint density satisfies
\begin{equation}
\prob{\X\Y}(\x, \y) = \prob{\Y|\X}(\y \mid \x)\prob{\X}(\x),
\end{equation}
these assumptions imply that we can generate samples from the joint distribution.

The inverse problem we consider is therefore the following: given samples from the joint distribution of $(\X, \Y)$, construct an algorithm capable of generating samples from the conditional density $\prob{\X|\Y}(\x \mid \y)$. Moreover, if the trained algorithm can subsequently be applied to arbitrary measurement values $\Y = \y$, then the training cost is amortized across multiple inference tasks.

In the following sections, we show that a conditional variant of the flow-matching algorithm can be used to solve similar inverse problems. The presentation proceeds in two stages. In \Cref{subsec:derive-flow-matching}, we derive the flow-matching algorithm for the unconditional generative problem: given samples of $\X$ drawn from an unknown density $\prob{\X}$, generate additional samples from the same distribution. Although this derivation is well established in the literature \cite{lipman2022flow, liu2022flow, albergo2022building}, we provide a simple and self-contained treatment for completeness. Subsequently, in \Cref{subsec:conditional-flow-matching}, we extend the method to the conditional setting. Specifically, given samples of $(\X, \Y)$ drawn from an unknown joint density $\prob{\X\Y}(\x, \y)$, we construct an algorithm that generates samples of $\X$ from the conditional density $\prob{\X|\Y}(\x \mid \y)$ for arbitrary values $\Y = \y$.

\subsection{A simple derivation of the flow matching loss}\label{subsec:derive-flow-matching}


In flow matching the generative problem is solved by generating samples of $\Z$ from a simple source density $\prob{\Z}$ and transforming these to samples from $\prob{\X}$. It begins with the definition of a stochastic interpolant given by 
\begin{equation}
    \X_t = \bm{I}_t(\Z,\X) \label{eq:sint},
\end{equation}
where $t \in (0,1)$, $\bm{I}_0 (\Z,\X) = \Z$ and $\bm{I}_1 (\Z,\X) = \X$. That is, at any given time $t$, the random vector $\X_t$ is a mixture of the random variables $\Z$ and $\X$ with probability densities $\prob{\Z}$ and $\prob{\X}$, respectively. Due to the interpolating property of $\bm{I}_t$, we have $\X_0 \sim  \prob{\Z}$, and $\X_1 \sim \prob{\X}$.

Let $\prob{t}$ be the probability density associated with the random vector $\X_t$. Since $\prob{t}$ defines a time-dependent probability density, it can be shown to satisfy the continuity equation and the velocity that appears in this equation (denoted by $\vm_t$ here) is such that for any $t \in (0,1)$, if $\X_0$ is sampled from $\prob{\Z}$, and $\X_t$ is evaluated by integrating the probability flow ode,
\begin{equation}\label{eq:ode}
    \frac{\mathrm{d} \X_t}{\mathrm{d}t } = \vm_t(\X_t), 
\end{equation}
then $\X_t \sim \prob{t}$. In particular 
then at $t = 1$ this yields $\X_1 \sim \prob{\X}$. This is the essence of the generative procedure in flow  matching. That is, first generate samples from an easy to sample source density $\prob{\Z}$, and then evolve them according to \Cref{eq:ode} to transform them to samples from $\prob{\X}$. The task is to compute this velocity field $\vm_t(\x)$ using  samples from $\prob{\Z}$ and $\prob{\X}$. This task is achieved in two steps. First, find an expression for the velocity field using the continuity equation. Second, define a loss function that can be computed using samples from $\prob{\Z}$ and $\prob{\X}$ whose minimizer is equal to this velocity field.

\begin{prop} The density $\prob{t}$ for the random vector $\X_t$ is defined as 
\begin{eqnarray}\label{eq:defrhot}
     \prob{t} (\xii)  &=& \int_{\Re^d \times \Re^d} \delta (\xii - \bm{I}_t(\z,\x))  \prob{\Z}(\z) \prob{\X} (\x) \mathrm{d}\z \mathrm{d}\x, 
\end{eqnarray}
and this density satisfies the continuity equation, wherein the velocity field is given by 
\begin{equation}\label{eq:defvel}
    \vm_t(\xii) = \frac{\jm_t(\xii)}{\prob{t} (\xii)}, 
\end{equation}
with the flux $\jm_t$ defined as 
\begin{equation}\label{eq:defflux}
    \jm_t(\xii) = \int_{\Re^d \times \Re^d} \delta (\xii - \bm{I}_t(\z,\x))  \frac{\partial \bm{I}_t(\z,\x)}{\partial t}  \prob{\Z}(\z) \prob{\X}(\x) \mathrm{d}\z \mathrm{d}\x. 
\end{equation}
\end{prop}

\begin{proof}
The overall approach is to derive an expression for the density $\rho_t$, use this to derive the continuity equation, and identify the velocity field in this equation. Let $\rho_t$ be the density for $\X_t$. Now consider an arbitrary function $\phi(\xii)$. Then 
\begin{equation}\label{eq:temp1}
    \mathbb{E}[\phi(\X_t)] = \int_{\Re^d} \phi (\xii) \prob{t} (\xii) \mathrm{d} \xii. 
\end{equation}
However, since $\X_t$ is related to $\Z$ and $\X$ through \Cref{eq:sint}, we also have 
\begin{eqnarray}
    \mathbb{E}[\phi(\X_t)] &=& \mathbb{E}[\phi(\bm{I}_t(\Z,\X)] \nonumber \\
        &=& \int_{\Re^d \times \Re^d} \phi (\bm{I}_t(\z,\x)) \prob{\Z}(\z) \prob{\X} (\x)  \mathrm{d}\z \mathrm{d}\x \nonumber \\
        &=& \int_{\Re^d} \int_{\Re^d \times \Re^d} \delta (\xii - \bm{I}_t(\z,\x)) \phi(\xii) \prob{\X}(\z) \prob{\X} (\x) \mathrm{d}\z \mathrm{d}\x \mathrm{d}\xii \nonumber \\
        &=& \int_{\Re^d} \phi(\xii) \Big( \int_{\Re^d \times \Re^d} \delta (\xii - \bm{I}_t(\z,\x))  \prob{\X}(\z) \prob{\X}(\x) \mathrm{d}\z \mathrm{d}\x \Big) \mathrm{d}\xii \label{eq:temp2}
\end{eqnarray}
Comparing \Cref{eq:temp1} and \Cref{eq:temp2}, we arrive at \Cref{eq:defrhot}.
Taking the time derivative of $\prob{t}$ in the \Cref{eq:defrhot} we have 
\begin{eqnarray}
\frac{\partial \prob{t} (\xii)}{\partial t}  &=& - \int_{\Re^d \times \Re^d} \nabla \delta (\xii - \bm{I}_t(\z,\x)) \cdot \frac{\partial \bm{I}_t(\z,\x)}{\partial t}  \prob{\Z}(\z) \prob{\X}(\x) \mathrm{d}\z \mathrm{d}\x \nonumber \\
    &=& - \nabla \cdot \Big( \int_{\Re^d \times \Re^d} \delta (\xii - \bm{I}_t(\z,\x))  \frac{\partial \bm{I}_t(\z,\x)}{\partial t}  \rho_Z(\z) \prob{\X} (\x) \mathrm{d}\z \mathrm{d}\x \Big) \nonumber \\
    &=& - \nabla \cdot  \jm_t(\xii), \label{eq:drhodt}
\end{eqnarray}
where we have used the definition of flux from \Cref{eq:defflux} in the last step. Finally, using the definition of the velocity field from \Cref{eq:defvel}, we arrive at the continuity equation
\begin{equation}
\frac{\partial \prob{t} (\xii)}{\partial t} +  \nabla \cdot  ( \rho_t (\xii) \vm_t(\xii)) = 0, \label{eq:continuity}    
\end{equation}
which completes our proof. 
\end{proof}
%



An important question is whether we can learn this velocity field purely from samples of $Z$ and $X$. This is answered next. 

\begin{prop}
The minimizer of the loss function, 
\begin{equation}
L(\bmm_t) = \int_0^1 \mathbb{E} \bigg[ \Big\lvert \bmm_t(\bm{I}_t(\Z,\X))  -   \frac{\partial \bm{I}_t(\Z,\X)}{\partial t} \Big\rvert^2 \bigg] \mathrm{d}t, \label{eq:loss1}
\end{equation}
is equal to $\vm_t$. 
\end{prop}


\begin{proof}
Let $\bm{\phi}(\xii)$ be any vector valued function. Then using the definition of the flux vector from \Cref{eq:defflux} we have,
\begin{equation}\label{eq:temp3}
    \int_{\Re^d}  \bm{\phi}(\xii) \cdot \jm_t(\xii) \, \mathrm{d}\xii = \int_{\Re^d \times \Re^d}  \bm{\phi}(\bm{I}_t(\z,\x)) \cdot  \frac{\partial \bm{I}_t(\z,\x)}{\partial t}  \prob{\Z}(\z) \prob{\X}(\x) \, \mathrm{d}\z \mathrm{d}\x. 
\end{equation}
We will use this relation later. 

Next, we write the loss function (\ref{eq:def-loss}) as a sum of two integrals,
\begin{equation}\label{eq:loss2}
L(\bmm_t) = \int_0^1 \mathbb{E} \bigg[ \Big\lvert\bmm_t(\bm{I}_t(\Z,\X))\Big\rvert^2  - 2 \bmm_t(\bm{I}_t(\Z,\X)) \cdot   \frac{\partial \bm{I}_t(\Z,\X)}{\partial t} \bigg] \mathrm{d}t  + \int_0^1 \mathbb{E}\bigg[ \Big\lvert\frac{\partial \bm{I}_t(\Z,\X)}{\partial t}\Big\rvert^2 \bigg] \mathrm{d}t . 
\end{equation}
Consider the first term in the first integral. From \Cref{eq:temp1}, we have 
\begin{equation}
\mathbb{E}\bigg[ \Big\lvert\bmm_t (\bm{I}_t(\Z,\X))\Big\rvert^2 \bigg] = \int_{\Re^d}  \lvert\bmm_t (\xii)\rvert^2 \prob{t}(\xii) \mathrm{d}\xii. 
\end{equation}
Consider the second term in the first integral. From (\ref{eq:temp3}) and (\ref{eq:defvel}), we have,
\begin{equation}
 \mathbb{E}\bigg[ \bmm_t(\bm{I}_t(\Z,\X)) \cdot \frac{\partial \bm{I}_t(\Z,\X)}{\partial t} \bigg]  = \int_{\Re^d} \bmm_t(\xii) \cdot \jm_t(\xii) \mathrm{d}\x = \int_{\Re^d} \bmm_t(\xii) \cdot \vm_t(\xii) \prob{t} (\xii) \mathrm{d}\xii . 
\end{equation}
Finally, the second integral does not depend on $b_t$, and we set it equal to a ``constant''  $C_1$.

Using these expressions in the loss function we have 
\begin{equation}
L(\bmm_t) = \int_0^1 \int_{\Re^d} \big( |\bmm_t(\xii)|^2 - 2 \bmm_t(\xii) \cdot \vm_t(\xii) \big) \prob{t}(\xii) \mathrm{d}\xii + C_1. 
\end{equation}
Setting the variations with respect to $\bmm_t$ equal to zero we have 
\begin{equation}
2 \int_0^1 \int_{R^d} \delta \bmm_t(\xii) \cdot \big( \bmm_{t}^*(\xii) -  \vm_t(\xii) \big) \prob{t}(\xii) \mathrm{d}\xii = 0, \qquad \forall \;\; \delta\bmm_t(\xi) 
\end{equation}
which gives us $\bmm_{t}^* = \vm_t$ and completes the proof. 
\end{proof}

\subsection{Extension to conditional densities}\label{subsec:conditional-flow-matching}

Consider the case where the target density $\prob{\X}(\x)$ is replaced by a conditional density $\prob{\X|\Y}(\x|\y)$. That is, we want to generate samples from a target density which is conditioned on the random vector $\Y = \y$. We note that in this case, for a fixed value of $\y$ the results of the previous section still apply. We write them below for completeness. 

The interpolant is still given by \Cref{eq:sint}, where as before $\Z \sim \prob{\Z}(\z)$, however, now $\X \sim \prob{\X|\Y}(\x|\y)$. Further, if we sample $\X_0 \sim \prob{\Z}(\z)$, and then evolve according to 
\begin{equation}\label{eq:ode-con}
    \frac{\mathrm{d}\X_t}{\mathrm{d}t} = \vm_t(\X_t,\y), 
\end{equation}
then we are guaranteed that $\X_1 \sim \prob{\X|\Y}(\x|\y)$. 

The appropriate loss function is given by 
\begin{equation}\label{eq:loss-con}
L(\bmm_t) = \int_0^1 \int_{\Re^d \times \Re^d } \Big\lvert \bmm_t(\bm{I}_t(\z,\x),\y)  -   \frac{\partial \bm{I}_t(\z,\x)}{\partial t} \Big\rvert^2 \prob{Z}(\z) \prob{\X|\Y}(\x|\y) \mathrm{d}\z \mathrm{d}\x \mathrm{d}t, 
\end{equation}
Further, we know that the minimizer of this loss function is equal to the desired velocity field $\vm_t$.

However, we note that this loss cannot be computed because we do not have samples from $\prob{\X|\Y}(\x|\y)$. We address this by defining a different loss,
\begin{equation}\label{eq:loss-con2}
\hat{L}(\bmm_t) = \int_0^1 \int_{\Re^d \times \Re^D \times \Re^d } \Big\lvert \bmm_t(\bm{I}_t(\z,\x),\y)- \frac{\partial \bm{I}_t(\z,\x)}{\partial t} \Big\rvert^2 \prob{\Z}(\z) \prob{\X\Y}(\x,\y) \mathrm{d}\x \mathrm{d}\y \mathrm{d}\z  \mathrm{d}t, 
\end{equation}
which is defined in terms of expectations over the densities $\prob{\Z}(\z)$ and $\prob{\X\Y}(\x,\y)$, and therefore it can be approximated by a Monte Carlo sum as long as we have samples from these densities. 

It is easy to show that for a given value of $\y$, the minimizer of $\hat{L}(\bmm_t)$ is also the minimizer of $L(\bmm_t)$. To see this, we use $\prob{\X\Y}(\x,\y) = \prob{\X|\Y}(\x|\y) \prob{\Y}(\y)$ in \Cref{eq:loss-con2} and change the order of integrations to write 
\begin{equation}\label{eq:loss-con3}
\hat{L}(\bmm_t) = \int_{\Re^D} \bigg[ \int_0^1 \int_{\Re^d \times \Re^d } \Big\lvert \bmm_t(\bm{I}_t(\z,\x),\y)- \frac{\partial \bm{I}_t(\z,\x)}{\partial t} \Big\rvert^2 \prob{\Z}(\z) \prob{\X|\Y}(\x|\y) \mathrm{d}\z \mathrm{d}\x \mathrm{d}t \bigg] \prob{\Y}(\y) \mathrm{d}\y. 
\end{equation}
Assuming $\prob{\Y}(\y) > 0 \;\; \forall \;\y$, from the equation above we conclude that for any given $\y$, the minimizer of $\hat{L}(\bmm_t)$ must minimize the term within the parenthesis. However, this term is precisely the loss $L(\bmm_t)$. Therefore the minimizer of $\hat{L}(\bmm_t)$ is equal to the minimizer of $L(\bmm_t)$, which is in turn equal to the desired velocity field $\vm_t$.

To summarize, in the conditional case, we use samples from $\prob{\Z}(\z)$ and $\prob{\X\Y}(\x,\y)$ in \Cref{eq:loss-con2} to compute the velocity field that minimizes $\hat{L}(\bmm_t)$ and use it in \Cref{eq:ode-con} to transport samples from the source density to the conditional density. Note that once the velocity field is learned it can be used for any value of $\Y = \y$.

\section{Consequences of overfitting the velocity field}\label{sec:overfitting}

In both the conditional and unconditional settings, the learning task can be formulated as a nonlinear regression problem for the velocity field. When only a finite amount of training data is available, it is standard practice to mitigate overfitting by monitoring the test loss and terminating training once the test loss ceases to decrease. In the present setting, this strategy prevents overfitting of the learned velocity field. However, an important question is how overfitting at the level of the velocity field manifests in the samples generated by that field.

In this section, we investigate this question for the conditional case using a simplified theoretical analysis together with illustrative numerical experiments. The simplification rests on assuming certain types approximations for the velocity field along the observation coordinates. The universal approximation property of neural network ensures that these approximations can be learned in practice. We demonstrate that overfitting the velocity field can lead to incorrect conditional distributions being generated. In certain regimes, the generated distribution collapses onto the inferred value from the training data whose corresponding observation vector is closest to the conditioning observation. We refer to this behavior as selective memorization. In other regimes, the generated distribution may instead collapse to a single point that interpolates the inferred variable between multiple training samples whose observation vectors are close to the conditioning observation. In both scenarios, the distribution produced by the flow matching algorithm deviates substantially from the true conditional distribution. This highlights the necessity of carefully selecting the training termination point.

\subsection{Problem setting} 
We analyze the learning problem under the assumption that only a finite amount of training data is available. As a consequence, the joint density $\prob{\X\Y}(\x,\y)$ appearing in the loss function is replaced by its empirical approximation. At the same time, we assume access to an infinite number of samples from the source distribution and from the pseudo-time variable. Under this assumption, the integrals over $z$ and $t$ in the loss are computed exactly. In practice, this regime is approximated by generating fresh i.i.d. samples of $\Z$ and $t$ in every minibatch during training. Finally, we assume that the neural network has sufficient capacity and that optimization is allowed to proceed indefinitely. These conditions ensure that, aside from the replacement of $\prob{\X\Y}(\x,\y)$ with its empirical counterpart, no additional approximations are introduced when solving the minimization problem.

For simplicity, we consider the   linear interpolant 
\begin{eqnarray}
    \X_t = \bm{I}_t(\Z,\X) = \Z (1-t) + t\X \label{eq:def-interpolant}
\end{eqnarray}
where $\Z \sim \prob{\Z}(\z)$ and $\X \sim \prob{\X|\Y}(\x|\y)$. This yields
\begin{equation}\label{eq:def-time-der}
    \frac{\partial \bm{I}_t(\X,\Z)}{\partial t} = \X - \Z.
\end{equation}
The loss function for the velocity field is given by \Cref{eq:loss-con2}. Using \Cref{eq:def-interpolant,eq:def-time-der} in \Cref{eq:loss-con2}, and changing the variable from $\z$ to $\xii = \bm{I}_t(\z,\x)$, we arrive at
\begin{equation}\label{eq:def-loss}
    \hat{L}(\bmm_t) = \frac{1}{(1-t)^d} \int_0^1 \int_{\Re^d \times \Re^D \times \Re^d }   \Big\lvert \bmm_t(\xii,\y) - \frac{\x-\xii}{1-t} \Big\rvert^2 \prob{\X\Y}(\x,\y) \prob{\Z}\left(\frac{\xii - \x t}{1-t}\right) \mathrm{d}\x \mathrm{d}\y \mathrm{d}\xii \mathrm{d}t.
\end{equation}
Since we are looking for the velocity field that minimizes this loss function, we consider arbitrary variations $\delta \bmm_t(\xii,\y)$ and set the change in the loss function to zero. This yields
\begin{equation}\label{eq:loss-variation}
    \int_0^1 \int_{\Re^d \times \Re^D \times \Re^d }  \delta \bmm_t \cdot \Big( \bmm_t^* - \frac{\x-\xii}{1-t} \Big) \prob{\X\Y}(\x,\y) \prob{\Z}\left(\frac{\xii - \x t}{1-t}\right) \mathrm{d}\x \mathrm{d}\y \mathrm{d}\xii \mathrm{d}t = 0, \;\forall \;\delta \bmm_t. 
\end{equation}
Here $\bmm_t^* (\xii,\y) $ denotes the velocity field that minimizes the loss function. 

In order to make further progress, we need to make some assumptions regarding the explicit form of $\bmm_t^*$ and $\delta \bmm_t$ as a function of $\y$. We assume that they are of the following form,
\begin{eqnarray}
    \bmm_t^*(\xii,\y) &=& \sum_k \phi_k(\y) \bmm_{t,k}^*(\xii) \label{eq:vel-expansion} \\
    \delta \bmm_t(\xii,\y) &=& \sum_j \phi_j(\y) \delta \bmm_{t,j}(\xii). \label{eq:del-vel-expansion}
\end{eqnarray}
Note that this type of architecture is often used in approximating neural operators and is referred to as the DeepONet \cite{lu2021learning}. In the context of a DeepONet, the functions $\phi_j$ are formed by the trunk network, and the coefficients $\bmm_{t,j}$ are formed by the branch network. Using \Cref{eq:vel-expansion,eq:del-vel-expansion} in \Cref{eq:loss-variation}, we arrive at 
\begin{align}
    &\int_0^1 \int_{\Re^d \times \Re^D \times \Re^d }  \sum_j (\delta \bmm_{t,j} \phi_j(\y)) \cdot \Bigg( \left(\sum_k \bmm_{t,k}^* \phi_k(\y)\right) - \frac{\x-\xii}{1-t} \Bigg) \prob{\X\Y} (\x,\y) \nonumber\\
    &\hphantom{\int_0^1 \int_{\Re^d \times \Re^D \times \Re^d }  \sum_j (\delta \bmm_{t,j} \phi_j(\y)) \cdot \Bigg( \left(\sum_k \bmm_{t,k}^* \phi_k(\y)\right)} \prob{\Z}\left(\frac{\xii - \x t}{1-t}\right) \mathrm{d}\x \mathrm{d}\y \mathrm{d}\xii \mathrm{d}t = 0, \forall \delta \bmm_{ t,j}. \label{eq:loss-variation-1}
\end{align}

In order to understand how the velocity behaves with limited data, we replace the joint density with its empirical approximation 
\begin{equation}\label{eq:def-empirical}
    \prob{\X\Y}(\x,\y) = \frac{1}{N} \sum_i \delta(\x-\x^{(i)}) \delta(\y - \y^{(i)}), 
\end{equation}
where the pair $(\x^{(i)}, \y^{(i)}), i = 1, \cdots, N$ are the training data. Using this in \Cref{eq:loss-variation-1}, we arrive at 
\begin{equation}
    \int_0^1 \int_{\Re^d}  \sum_i \sum_{j} (\delta \bmm_{t,j} \phi_j(\y^{(i)})) \cdot \Bigg( \Big(\sum_k \bmm_{t,k}^* \phi_k(\y^{(i)})\Big) - \frac{\x^{(i)}- \xii}{1-t} \Bigg) \prob{\Z}\left(\frac{\xii - \x^{(i)} t}{1-t}\right) \mathrm{d}\xii \mathrm{d}t = 0, \forall \delta \bmm_{t,j}. \label{eq:loss-variation-2}
\end{equation}
The Euler-Lagrange equations corresponding to this variational form are
\begin{equation}\label{eq:euler-lagrange}
    \sum_i   \phi_j(\y^{(i)})  \Bigg( \left(\sum_k \bmm_{t,k}^*(\xii) \phi_k(\y^{(i)})\right) - \frac{\x^{(i)}-\xii}{1-t} \Bigg) \prob{\Z}\left(\frac{\xii - \x^{(i)} t}{1-t}\right)  = \bm{0},  \quad \forall j. 
\end{equation}

At this point, we need to make assumptions regarding the functions $\phi_j$ to make progress. We consider two cases: 

\subsection{Case 1} \label{sec: zero variance} We assume that $\phi_j$ are a partition of unity, that is,
\begin{equation}\label{eq:pu}
    \sum_j \phi_j(\y) = 1, 
\end{equation}
and that they are interpolatory at the training data points $y^{(i)}$. That is,
\begin{equation}\label{eq:y-interpolant}
    \phi_j(\y^{(i)}) = \delta_{ij}. 
\end{equation}
We note that these conditions are often satisfied by the basis functions used in the finite element method (also see \Cref{fig:overfit-case-1}). We also note that when $\y$ is one-dimensional ($D = 1$) this basis can be generated by a neural network with a single layer of ReLU activations (see \cite{ray2024deep}).
\begin{figure}[t]
    \centering
    \includegraphics[width=0.8\linewidth]{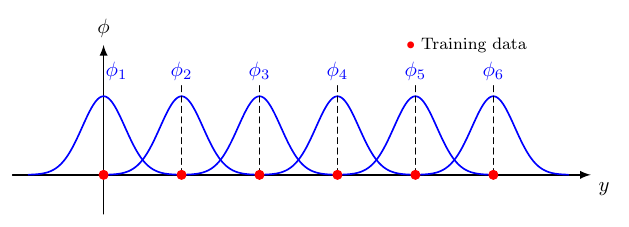}
    \caption{Visualizing Case 1 corresponding to \Cref{eq:pu}}
    \label{fig:overfit-case-1}
\end{figure}

Using \Cref{eq:y-interpolant} in \Cref{eq:euler-lagrange}, we arrive at 
\begin{equation}\label{eq:loss-variation-31}
     \Big(  \sum_{k} \bmm_{t,k}^*(\xii)\phi_k(\y^{(j)})  - \frac{\x^{(j)}-\xii}{1-t} \Big) \prob{\Z}\left(\frac{\xii - \x^{(j)} t}{1-t}\right)  = \bm{0}, \;\forall \; j.
\end{equation}
Using \Cref{eq:y-interpolant} once again in the above equation yields, 
\begin{equation}\label{eq:loss-variation-3}
     \Big(  \bmm_{t,j}^*(\xii) - \frac{\x^{(j)}-\xii}{1-t} \Big) \prob{\Z}\left(\frac{\xii - \x^{(j)} t}{1-t}\right)  = \bm{0}, \;\forall \; j.
\end{equation}
Assuming $\prob{\Z} >0$, the solution to this equation is given by 
\begin{equation}
    \bmm_{t,j}^*(\xii) = \frac{\x^{(j)}-\xii}{1-t}.
\end{equation}
Using this in \Cref{eq:vel-expansion}, we arrive at 
\begin{equation}
    \bmm_t^*(\xii,\y) = \sum_j \frac{\x^{(j)}-\xii}{1-t} \phi_j(\y).
\end{equation}
Using the partition of unity property \Cref{eq:pu}, we can further simplify this to 
\begin{equation}\label{eq:overfit-case-1}
    \bmm_t^*(\xii,y) =  \frac{\bar{\x}(\y)-\xii}{1-t},
\end{equation}
where 
\begin{equation}
    \bar{\x}(\y) = \sum_j \x^{(j)} \phi_j(\y),
\end{equation}
is a weighted sum of the training points $\x^{(j)}$, and where the weights are given by $\phi_j(\y)$. It is expected that the training data points that are closer to a given value of the observation vector, $\y$, will contribute more to this sum. 

We use the expression above for the velocity field in the probability flow ODE \Cref{eq:ode-con}, 
\begin{eqnarray}
    \frac{\mathrm{d} \X_t}{\mathrm{d}t} = \frac{\bar{\x}(\y)- \X_t}{1-t}
\end{eqnarray}
and integrate it from $t = 0$ to $t = 1$, with the initial condition at $t = 0$, $\X_t = \X_0$ to arrive at 
\begin{equation}
    \X_t = \X_0 (1-t) + \bar{\x}(\y) t.
\end{equation}
That is for any value of $\X_0$, at $t = 1$ all samples arrive to the point $\bar{\x}(\y)$, which in turn implies that generated conditional density converges to the Dirac measure at $\bar{\x}(\y)$.

\subsection{Case 2} 

\begin{figure}[!b]
    \centering
    \includegraphics[width=0.8\linewidth]{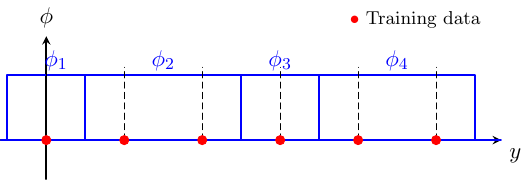}
    \caption{Visualizing Case 2 where $\phi_j$-s are non-overlapping indicator functions}
    \label{fig:overfit-case-2}
\end{figure}
Let $\phi_j(\y)$ be non-overlapping indicator functions that are unity on $\Omega_j \subset \Re^D$ and zero on the complement (see \Cref{fig:overfit-case-2}). This form for the basis functions is chosen to represent functions whose values diminish rapidly as they approach the boundary of their support. Further let $S_j$ be the set of training data points for which $\y$-coordinate is contained in $\Omega_j$. Using this in \Cref{eq:euler-lagrange}, we arrive at 
\begin{equation}\label{eq:euler-lagrange10}
    \sum_{i \in S_j}    \Big( \big(\sum_k \bmm_{t,k}^*(\xii) \phi_k(\y^{(i)})\big)   - \frac{\x^{(i)}-\xii}{1-t} \Big) \prob{\Z}\left(\frac{\xii - \x^{(i)} t}{1-t}\right)  = \bm{0},  \quad \forall \; j. 
\end{equation}
Since the functions $\phi_k$ are non-overlapping, and since $i \in S_k$, $\phi_k(\y^{(i)})$ is non-zero only when $k = j$. This yields,
\begin{equation}\label{eq:euler-lagrange1}
    \sum_{i \in S_j}    \Big(  \bmm_{t,j}^*(\xii)  - \frac{\x^{(i)}-\xii}{1-t} \Big) \prob{\Z}\left(\frac{\xii - \x^{(i)} t}{1-t}\right)  = \bm{0},  \quad \forall \; j. 
\end{equation}

This in turn implies, 
\begin{equation}
    \bmm_{t,j}^* (\xii)   =  \sum_{i \in S_j} \tilde{\rho}_{\Z}\left(\frac{\xii - \x^{(i)} t}{1-t}\right) \times \left( \frac{\x^{(i)}-\xii}{1-t}\right) \quad  \forall \; j,
\end{equation}
where 
\begin{equation}
    \tilde{\rho}_{\Z}\left(\frac{\xii - \x^{(i)} t}{1-t}\right) = \frac{\prob{\Z}\left(\frac{\xii - \x^{(i)} t}{1-t}\right)}{\sum_{k \in S_j} \prob{\Z}\left(\frac{\xii - \x^{(k)} t}{1-t}\right)}.
\end{equation}
It is easily verified that for $i \in S_j$, using the definition above, and recognizing that $\prob{\Z} > 0$ everywhere, the set of weights $\tilde{\rho}_{\Z}$ form a positive partition of unity. 

To proceed further, we need to make a choice for the source distribution. Based on what is the most popular choice, we select it to be the standard normal distribution of appropriate dimension. That is, $\prob{\Z}(\z) = \mathcal{N}(\z;\bm{0},\mathbb{I}_d)$\footnote{We use the notation $\mathbb{I}_d$ to denote the identity matrix of size $d \times d$}. With this choice, as shown in \cite{baptista2025memorization}, as $t \to 1$ the $\tilde{\rho}_{\Z}\left(\frac{\xii - \x^{(i)} t}{1-t}\right) $ tend to an indicator function for the Voronoi cell corresponding to the \supth{i} training data point. Therefore, we have 
\begin{equation}
    \bmm_{t,j}^*(\xii)   = \sum_{i \in S_j} \mathcal{I}^{(i)}(\xii) \frac{\x^{(i)}-\xii}{1-t} \quad  \forall \; j,
\end{equation}
where $\mathcal{I}^{(i)}(\xii)$ is the indicator function. Using this in the expression for the expansion for the velocity from \Cref{eq:vel-expansion},
\begin{equation}
    \bmm_{t}^*(\xii,y)   = \sum_j \left( \sum_{i \in S_j} \mathcal{I}^{(i)}(\xii) \frac{\x_1^{(i)}-\xii}{1-t}\right) \phi_j(\y).
\end{equation}
For a given value of the observation vector $\y$, let $j$ denote the index of the function $\phi_j(\y)$ which is non-zero. Since these functions form a non-overlapping partition of unity, there is only one such function and its value is equal to unity. As a result, the above expression reduces to 
\begin{equation}\label{eq:overfit-case-2}
    \bmm_{t}^*(\xii,\y)   =  \sum_{i \in S_j} \mathcal{I}^{(i)}(\xii) \frac{\x_1^{(i)}-\xii}{1-t}.
\end{equation}
With this expression for the velocity field, it was shown in \cite{baptista2025memorization} that for any initial condition $\X_0$, all trajectories of $\X_t$ will terminate at one of the points $\x^{(i)}$ contained in $S_j$. These points are a subset of the samples in the training data set. Thus, in this case, we observe a phenomenon that we refer to as selective memorization. This means that for a given value of the measurement vector, the sampling process will yield samples of the inferred vector that form a subset of the samples contained in the training data.

\subsection{Numerical evidence of the effect of finite data}\label{subsec:overfitting-experiment}
In this section, we demonstrate the effect of finite data on the estimation of the conditional density for a simple problem. In particular, we consider the forward model
\begin{equation}
    Y = X + \eta, \label{eq:forward1}
\end{equation}
where $X, Y \in \mathbb{R}$ and $\eta \sim \mathcal{N}(0, 0.25^2)$. Furthermore, we assume that the prior density of $X$, denoted by $\rho_X(x)$, is uniform between -1 and 1. In this case, the conditional distribution of X conditioned on $Y = \hat{y}$ will be a normal distribution truncated between -1 and 1 with the mean shifted to $\hat{y}$ and variance nearly equal to 0.25\textsuperscript{2}. For training, we generate five samples from $\rho_X$, and for each sample we generate the corresponding $Y$ value using \Cref{eq:forward1}. These samples are plotted in \Cref{fig:conditional_memorization_data_and_loss}(a), together with 1000 test samples and the curve $Y = X$. 
\begin{figure}[!b]
    \centering
    \includegraphics[width=1.0\linewidth]{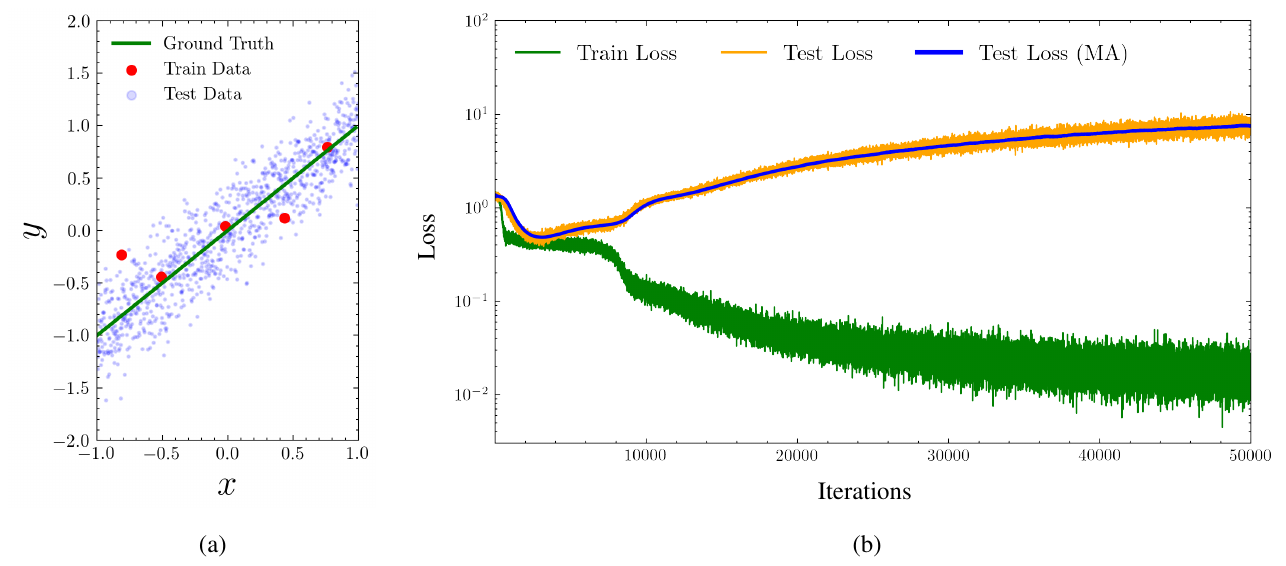}
    \caption{(a) Train and test data for the toy example used to illustrate effects of overfitting. (b) Train and test loss curve for the velocity network trained using the training data shown in (a). The blue curve above shows the moving average (MA) of the test loss. The MA is computed over a window of size 500}
    \label{fig:conditional_memorization_data_and_loss}
\end{figure}

We then use the five training samples to train a conditional flow matching algorithm to generate samples from the conditional density $\rho_{X \mid Y}$. The velocity field in the flow matching algorithm is parameterized by a neural network with Swish activation function and 3 hidden layers, each of width 32. The source distribution is chosen to be the standard normal distribution, i.e., $Z \sim \mathcal{N}(0,1)$. The training and test losses for the velocity network are shown in \Cref{fig:conditional_memorization_data_and_loss}(b). We observe that the test loss initially decreases, attains a minimum at approximately 3{,}000 iterations, and subsequently increases, exhibiting the classical overfitting behavior typical of regression problems. 

In \Cref{fig:conditional_memorization_kde}, we plot the kernel density estimate of the distribution of the samples generated by the learned velocity field at different stages of training for $Y = 0.6$. We observe that the network trained for 1,000 iterations produces a distribution that remains close to the source distribution, indicating that the network has not yet learned the true velocity field. The network trained for 3,000 iterations (corresponding to the minimum training loss) generates a distribution that is close to the true conditional distribution. In contrast, the network trained for 15,000 iterations, corresponding to significant overtraining, produces a distribution with 
substantially underestimated variance. When the network is severely overtrained (50,000 iterations), the variance becomes even smaller. These observations are consistent with the analysis presented in Case~1 of the previous section.
\begin{figure}[t]
    \centering
    \includegraphics[width=0.9\linewidth]{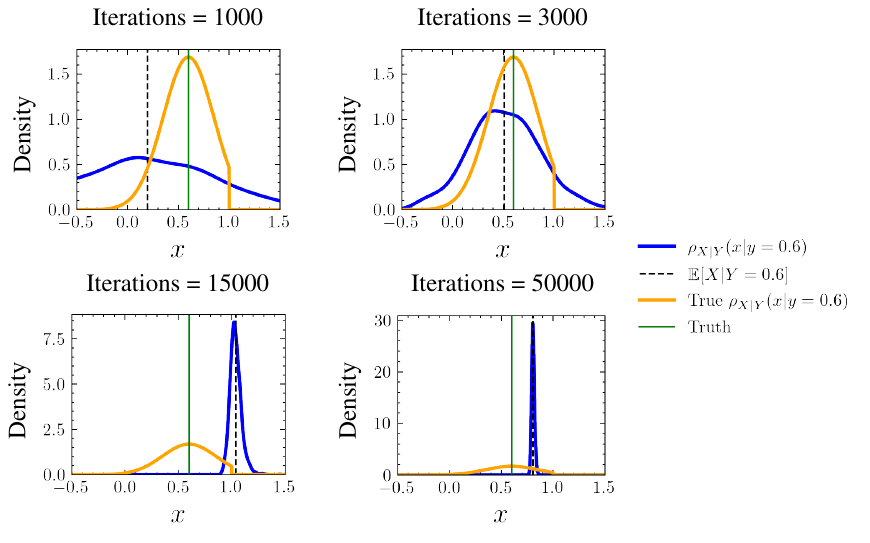}
    \caption{Kernel density estimates of the conditional distribution $\rho_{X|Y}(x \mid y = 0.6)$ estimated from samples generated using the trained velocity network at different stages of training}
    \label{fig:conditional_memorization_kde}
\end{figure}

To demonstrate that this phenomenon occurs for all values of $Y$, we consider 100 different values of $Y$ sampled uniformly in the range $[-1, 1]$. For each value of $Y$, we used the velocity network (trained for a fixed number of iterations) and evaluate the samples generated by the conditional flow matching algorithm. In \Cref{fig:conditional_memorization_posterior}, we plot the mean and the one-standard-deviation interval of these samples as functions of $Y$. We clearly observe that as the number of training iterations increases, the standard deviation decreases for all values of $Y$, thereby validating the analysis in Section~\ref{sec: zero variance}. We also observe that, as expected, increasing overtraining leads to a more complex (and incorrect) functional dependence of $\mathbb{E}[X \mid Y]$ on $Y$.
\begin{figure}[!t]
    \centering
    \includegraphics[width=\linewidth]{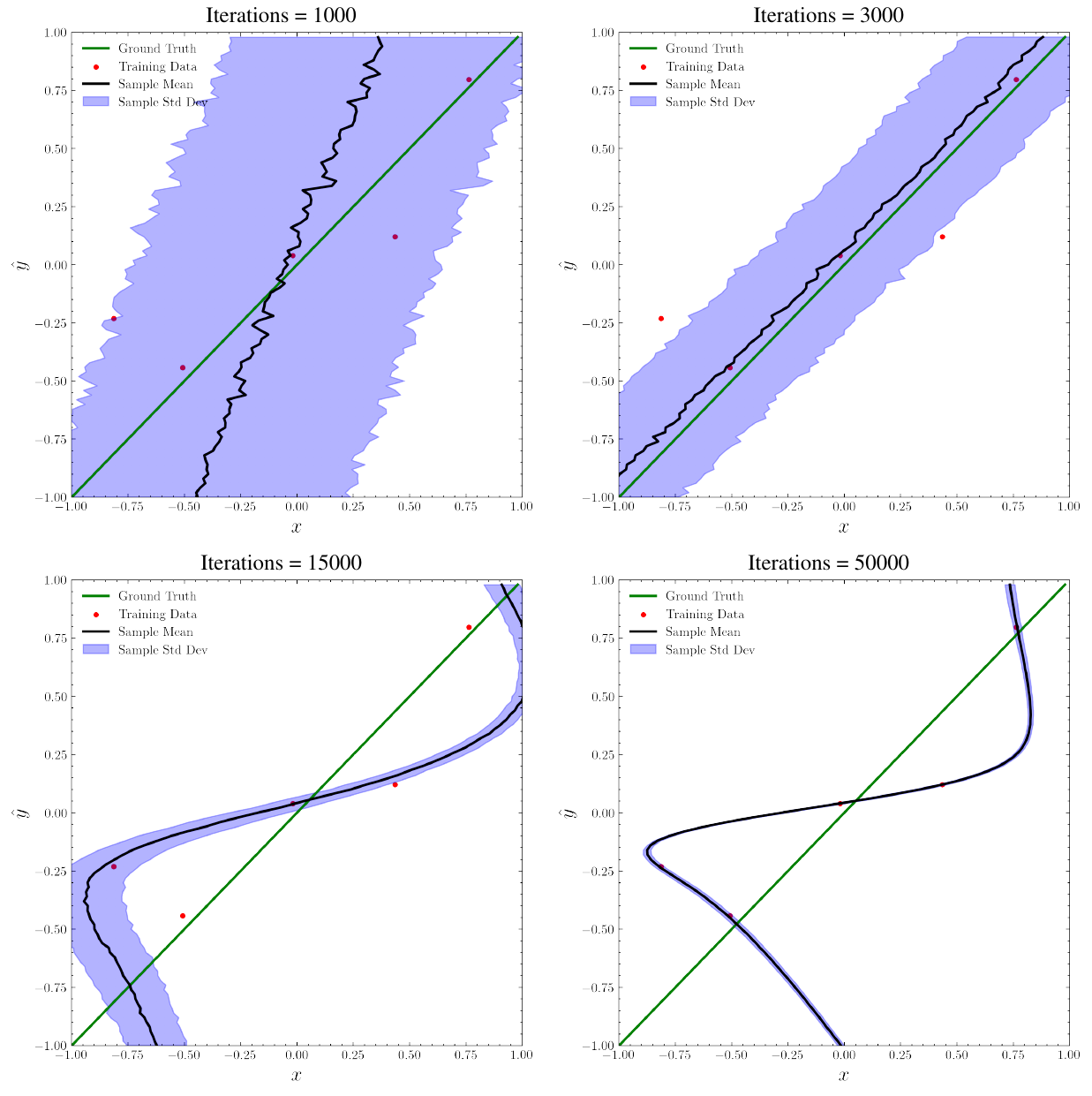}
    \caption{Mean and one-standard-deviation interval of the conditional distribution $\rho_{X|Y}$ estimated using samples generated by the trained velocity network at different stages of training}
    \label{fig:conditional_memorization_posterior}
\end{figure}

We further investigate the effects of increasing the capacity of the velocity network and the amount of training data (albeit still limited) on the performance of conditional flow matching in \ref{app:finite-data}. In summary, the results above and in \ref{app:finite-data} validate the analysis presented earlier in this section. They also suggest that a practical strategy to avoid degeneracy associated with finite data (or to avoid overtraining) is to stop training the velocity network once the test loss ceases to decrease. We adopt this strategy when training the velocity network for all problems considered in the following section.

\section{Results}\label{sec:results}

In this section, we present results for solving multiple inverse problems, including problems motivated by conditional density estimation (\Cref{subsec:cond-density}), data assimilation (\Cref{subsubsec:one-step-DA}), and physics-based inverse problems arising in fluid mechanics (\Cref{subsubsec:ADR}) and inverse elasticity (\Cref{subsubsec:sci,subsubsec:elasto2,subsubsec:TS}), 
using conditional flow matching. Across all the examples that we consider, the goal is to generate samples from the posterior distribution $\prob{\X|\Y}(\cdot|\hat{\y})$, for an observation or measurement $\hat{\y}$, using conditional flow matching given training data from the joint distribution $\prob{\X\Y}$.

\paragraph{Details of implementation, training, and data normalization} We perform all numerical experiments using PyTorch~\cite{pytorch2019}. We model the velocity field using a multilayer perceptron (MLP), for the examples in \Cref{subsec:cond-density,subsubsec:one-step-DA,subsubsec:ADR} , or a DDPM-inspired U-Net~\cite{ho2020denoising}, for the examples in \Cref{subsubsec:sci,subsubsec:elasto2,subsubsec:TS}. 
We relegate additional details regarding the architectures of the velocity network and time conditioning for the various problems to \ref{app:exp-settings}. To optimize the parameters of the velocity network, we minimize the loss function \Cref{eq:loss-con2} using the Adam and AdamW optimizer \cite{kingma2015adam} for the MLP- and U-Net-based velocity networks, respectively. We estimate the loss function \Cref{eq:loss-con2} using a mini-batch sampled from the training dataset at every iteration. We report training hyper-parameters, such as the learning rate and batch size, in \ref{app:exp-settings}. 

Before training the velocity network, we min-max normalize the training and test data between $[-1,1]$. As we discuss in \Cref{sec:overfitting}, in the case of limited data, overfitting and selective memorization can occur if the velocity field is trained for too long. So, we monitor the moving average of the test loss (over a window of size 500) during training. We also maintain the exponential moving averages (EMA) of the weights~\cite{dasgupta2025conditional} and use them to estimate the test loss over a held-out test set. Using the EMA of the weights helps dampen fluctuations in the test loss across consecutive iterations. Following the findings in \Cref{subsec:overfitting-experiment}, we present results using checkpoints close to where the moving average of the test loss saturates. In some cases, we simply terminate training after training the velocity network for an \latin{a priori} fixed number of iterations when the moving average of the test loss does not appear to saturate. In these cases, the total number of iterations reflects the compute budget available for training. 

\paragraph{Sampling using the trained velocity network}
At the chosen checkpoint, we use the EMA of the weights to sample the target posterior. We use an adaptive explicit Runge-Kutta method of order 5(4), available through \texttt{SciPy}'s~\cite{2020SciPy-NMeth} \texttt{solve\textunderscore ivp} routine to integrate \Cref{eq:ode-con}. After sampling, we re-normalize the generated samples to the appropriate physical units, and then estimate posterior statistics. We also report the number of steps 
taken by the integrator to produce realizations from the posterior, which we herein refer to as the number of sampling steps. A smaller number of steps means fewer evaluations of the velocity network and more efficient sampling. 


\paragraph{Source distributions}
In the experiments to follow, we explore the flexibility offered by the conditional flow matching framework and consider two different source distributions. We consider a multivariate standard normal distribution of $d$-dimensions, which we herein refer to as the Gaussian source. Additionally, we consider the prior distribution $\prob{\X}$, or the appropriate marginal of $\X$, as a second source distribution. The goal is to assess the benefits of data-informed source distributions. We use the accuracy of the posterior statistics (where possible), and the number $N_{\text{step}}$ of steps taken to compare the two types of source distribution. When working with $\prob{\X}$ as the source distribution, we assume that our knowledge of $\prob{\X}$ is limited to realizations of $\X$ in the training data. Accordingly, we sample a mini-batch of joint realizations of $\X$ and $\Y$ from the training dataset, and obtain realizations of $\Z$ by scrambling the realizations of $\X$ in the mini-batch to avoid instances where $\partial \bm{I}_t(\X,\Z) / \partial t = \bm{0}$. 

\paragraph{Comparison with noise-conditioned diffusion models} We adapt the numerical example in \Cref{subsubsec:sci} and applications in \Cref{subsubsec:elasto2,subsubsec:TS} 
from a previous study by \citet{dasgupta2026unifying}, where a discrete formulation of diffusion models was used to solve the related inverse problems. In these cases, we also highlight the benefits of using conditional flow matching. 

\subsection{Conditional density estimation benchmark}\label{subsec:cond-density}


In this example, which we refer to as the spiral problem, we consider the following two-dimensional joint distribution $\prob{XY}(x,y)$ adapted from \cite{dasgupta2026unifying}:
\begin{equation}\label{eq:spiral}
X = 0.1 (W \sin(W) +  C_3), \qquad Y = 0.1 (W \cos(W) + C_4),
\end{equation}
where $W = 1.5 \pi (1+ 2H)$, $H \sim \mathcal{U}(0,1)$, $C_3 \sim \mathcal{N}(0,1)$, and $C_4 \sim \mathcal{N}(0,1)$. In this case, $d = D = 1$, \ie both $X$ and $Y$ are one-dimensional. We sample 800 realizations from the joint distribution $\prob{XY}(x,y)$ to form the training dataset shown in \Cref{fig:spiral-data}.
\begin{figure}[H]
    \centering
    \includegraphics[height=2.4in]{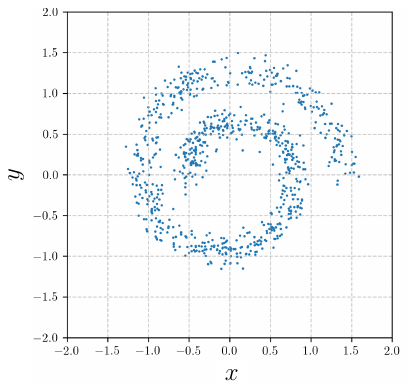}
    \caption{Training dataset for the spiral problem}
    \label{fig:spiral-data}
\end{figure}

\begin{figure}[!t]
    \centering
    \includegraphics[height=3.2in]{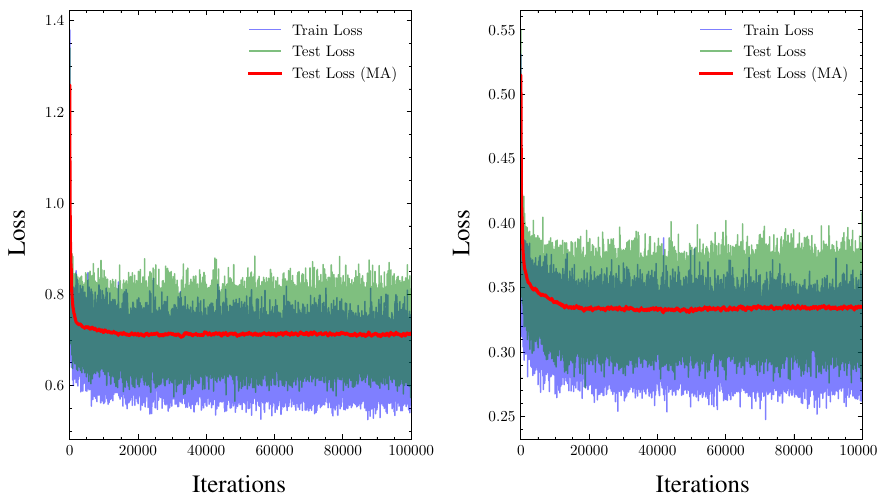}
    \caption{Training loss, test loss, and the moving average of the test loss for the velocity network trained on the spiral dataset with $Z \sim \mathcal{N}(0,1)$ (left) and $Z \sim \rho_X$ (right) as the source distributions, respectively} 
    \label{fig:spiral-loss}
    \vspace{1em}
\end{figure}
Next, we train two velocity networks on this training data corresponding to two different source distributions: $Z \sim \mathcal{N}(0, 1)$ and $Z \sim \prob{X}$, where $\prob{X}$ corresponds to the true marginal of $X$. 
\Cref{fig:spiral-loss} shows the training and test loss for the two source distributions. In both cases, we use the velocity network trained for 20,000 iterations (when the test loss saturates) to generate 10,000 samples each from the conditional distribution $\prob{X \vert Y}(x|y = \hat{y})$ for $\hat{y} \in\{-0.5, 0.0, 0.5, 1.0\}$. We start sampling with independent realizations of $\prob{X}$, not part of either the training or test set, as the initial condition for the ODE in \Cref{eq:ode-con} for the conditional flow matching model trained with $Z \sim \prob{X}$ as the source distribution.



\Cref{fig:spiral-kde} shows the histograms of samples from the conditional distribution generated using the trained velocity network for $\hat{y} \in\{-0.5, 0.0, 0.5, 1.0\}$. We also include histograms of samples obtained from the `\emph{true}' conditional distribution. We obtain these samples by retaining points in a band of width 0.1 around the
specified value of $\hat{y}$ from a test set containing 100,000 realizations from the joint density. \Cref{fig:spiral-kde} qualitatively shows that the conditional flow matching approach can be used to sample the conditional distribution with multiple modes. The Sinkhorn-Knopp algorithm~\cite{NIPS2013_af21d0c9} for computing the regularized optimal transport (OT) distance between the generated samples and the reference conditional density, which we herein refer to as the Sinkhorn distance, further quantifies the accuracy of conditional flow matching models. We tabulate the Sinkhorn distance between the true and estimated conditional distributions, averaged over the four $\hat{y}$ values, for the conditional flow matching models with $Z \sim \mathcal{N}(0, 1)$ and $Z \sim \prob{X}$ in the second column in \Cref{tab:spiral-results_comparison}. In comparison, we note that the best conditional diffusion model trained in \cite{dasgupta2026unifying} yields a Sinkhorn distance of 0.041 using a much larger number of training data points (10,000). We also report the number of sampling steps, averaged across the generated samples and four different values of $\hat{y}$, in \Cref{tab:spiral-results_comparison}. Overall, the results suggest that the choice of the source distribution has a negligible effect for this problem, \ie choosing $Z \sim \mathcal{N}(0, 1)$ or $Z \sim \prob{X}$ performs equally well in approximating the conditional distribution for different values of $Y$. \Cref{tab:spiral-results_comparison} also shows that the sampling efficiency of the conditional flow matching models with either source distribution is similar because the number of sampling steps is similar. 
\begin{figure}[!t]
\centering
    \includegraphics[width=\linewidth]{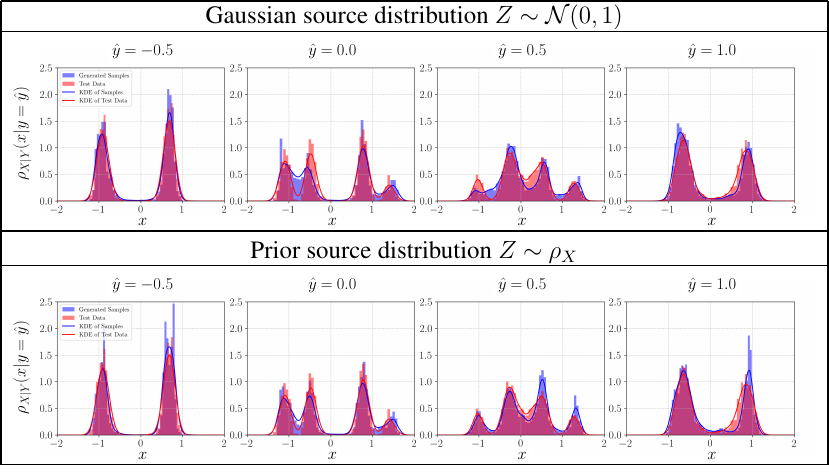}
    \caption{Histograms of samples generated using the trained velocity field compared to the samples from the true conditional distribution $\rho_{X|Y}(x|y = \hat{y})$ for $\hat{y} \in \{0.5, 0, 0.5, 1\}$ with different source distributions. The kernel density estimates of the conditional density are also shown}
\label{fig:spiral-kde}
\end{figure}
\begin{table}[!t]
    \centering
    \caption{Comparison of the two different source distributions for the spiral dataset}
    \label{tab:spiral-results_comparison}
    \begin{tabular}{lcc}
    \toprule
    \makecell[l]{Source\\distribution} & \makecell{Avg. Sinkhorn\\ distance} & \makecell{Avg. number of\\sampling steps} \\
    \midrule
    Gaussian & 0.051 & 17 \\
    Prior    & 0.056 & 19 \\
    \bottomrule
    \end{tabular}
\end{table}

\subsection{Inverse problems with synthetic data}

\subsubsection{One-step data assimilation problem}\label{subsubsec:one-step-DA}

Data assimilation, specifically filtering, is a sequential inference process that improves the predictive capability of numerical models by incorporating noisy and sparse observations made at continuous intervals. When using the Bayes filter, the data assimilation problem reduces to the solution of a probabilistic inverse problem at each instance an observation is obtained. Further, when sample based methods (like the ensemble Kalman Filter~\cite{evensen2003ensemble}, or the particle filter~\cite{doucet2000sequential}) are used, the prior and posterior distributions are both represented by samples. Samples for the prior distribution are obtained by applying forward dynamics operator to the samples from the previous assimilated state, and the posterior is approximated by generating samples conditioned on the most recent measurement. In this section, we consider a problem which corresponds to a single step in a data assimilation problem solved using a sample-based Bayes filter method. The problem is motivated by the Lorenz 63 system \cite{DeterministicNonperiodicFlow} and its details are described in \ref{app:da-details}. 

Succinctly, the dimension of the vector $\X$ to be inferred is $d = 3$, and the observation, which is a scalar ($D = 1$), is a noisy version of one of the components of this vector,  
\begin{equation}\label{eq:DA-obs_op}
    Y = X_3 + \epsilon,
\end{equation}
where $\epsilon \sim \mathcal{N}(0, 0.5^2)$ denotes the measurement noise. Since the vector $\bm{X}$ is partially observed, the posterior distribution can exhibit significant non-Gaussian structures like bimodality. This study focuses on a single data assimilation step chosen to yield a nontrivial transformation where the prior distribution is unimodal while the conditioning on the observation induces a bimodal posterior.

To obtain accurate approximations of the prior and posterior distributions, we employ the Sequential Importance Resampling (SIR) filter~\cite{doucet2000sequential} with an ensemble size of 100,000 to solve the first three steps of the data assimilation problem. For sufficiently large ensemble sizes, the SIR approximation converges to the true Bayesian posterior. The prior distribution for the inverse problem is obtained by sub-sampling 1,500 realizations from the prior distribution for the SIR approximation at the third step. The observation is given by \Cref{eq:DA-obs_op}. Once the inverse problem is solved using the conditional flow matching algorithm, the samples from the posterior distribution are compared with the 100,000 samples from the SIR approximation of the posterior distribution, which is treated as the true reference solution.


\begin{figure}[!b]
    \centering
    \includegraphics[width=\linewidth]{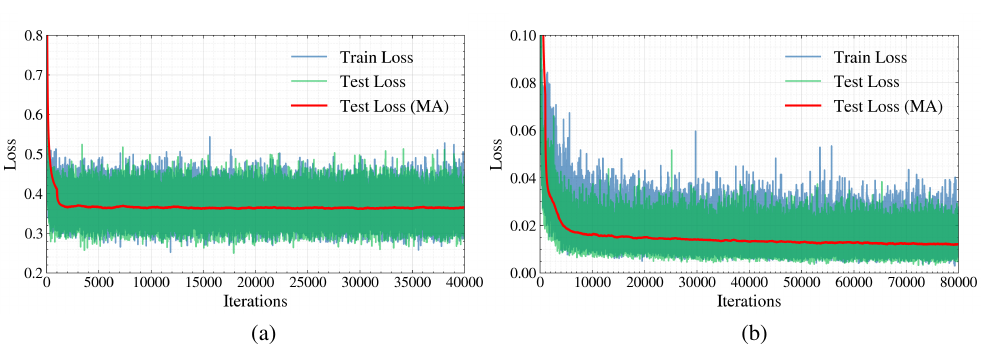}
    \caption{Training loss, test loss, and the moving average of the test loss for the velocity network trained on the one-step data assimilation example with (a) Gaussian and (b) the prior $\prob{X}$ as the source distribution}
    \label{fig:DA-loss_plots}
\end{figure}
\begin{figure}[!t]
    \centering
    \includegraphics[width=0.78\linewidth]{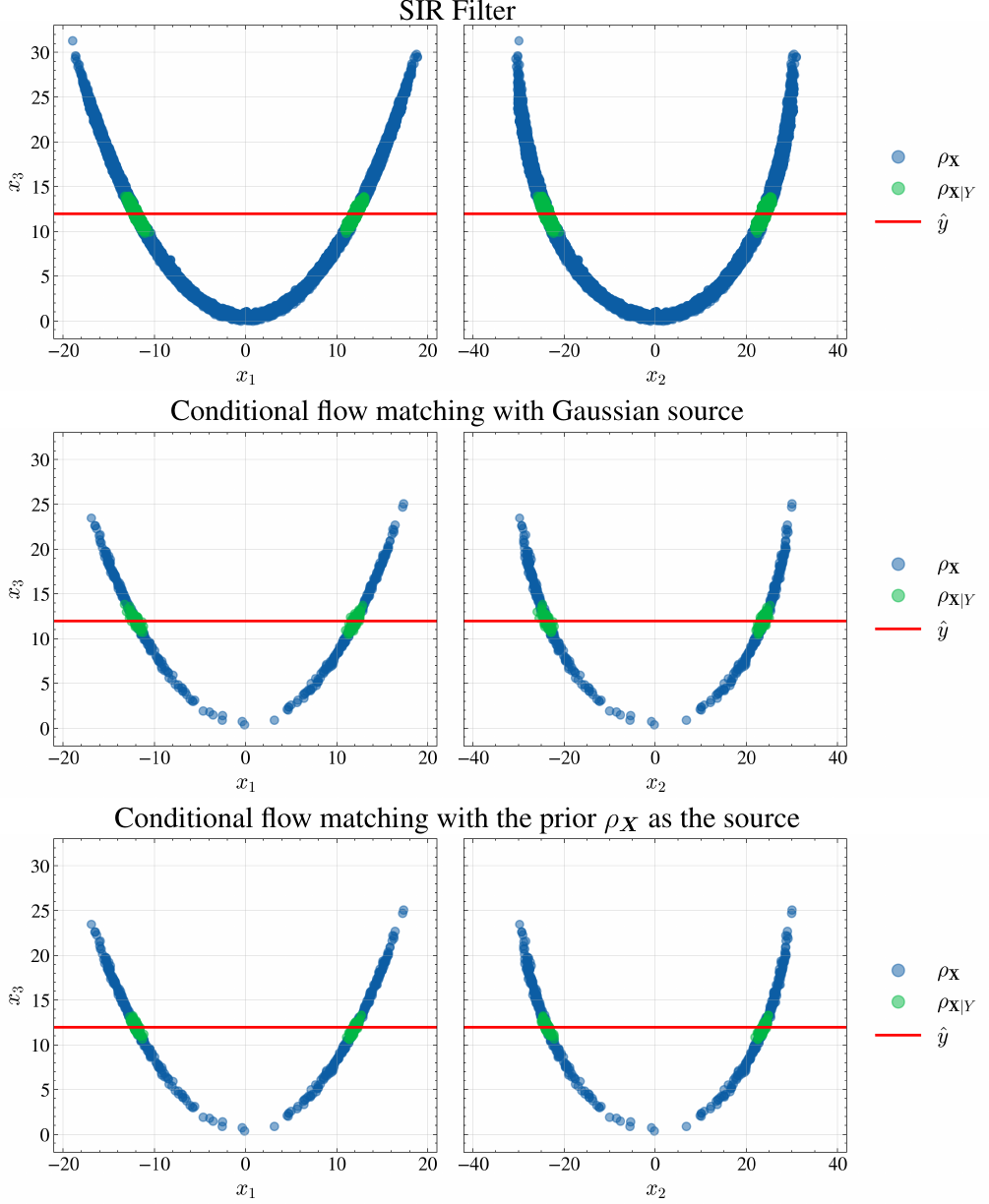}
    \definecolor{MyBlue}{RGB}{0, 127, 255}
    \definecolor{MyGreen}{RGB}{0, 184, 69}    
    \caption{Particles from the prior $\rho_{\X}$ (\textcolor{MyBlue}{$\bullet$}) and posterior $\rho_{\X|Y}$ (\textcolor{MyGreen}{$\bullet$}), and observation $\hat{y}$ ({\color{red}--}) on the $(x_1,x_3)$ and $(x_2,x_3)$ planes for the one-step data assimilation problem. First row: reference solution obtained using the SIR filter with 100,000 particles. Second and third row: 1,000 particles from $\rho_{\X}$ (part of the training dataset), and samples from $\rho_{\X|Y}$ generated by the conditional flow matching model with a Gaussian and prior source distributions, respectively}
    \label{fig:DA-distributions}
    \label{fig:DA-distributions}
\end{figure}
We train two conditional flow matching models to generate samples from the posterior distribution. The models differ only in the choice of source distribution: a multivariate Gaussian (Gaussian source) and the prior distribution itself. For both models, 1,000 training pairs and 500 testing pairs are used to learn the velocity field, and 500 samples are generated conditioned on a fixed observation $\hat{y}$. \Cref{fig:DA-loss_plots} shows the training and test loss for the two source distributions. The velocity network trained for 5,000 iterations is chosen for the Gaussian source, while the velocity network trained for 10,000 iterations is chosen when the prior is the source distribution. 
These checkpoints are selected because the moving average of the test loss transitions from an initial rapid decrease to either a near-constant level or a regime of slow, nearly constant descent.

\begin{table}[t]
    \centering
    \caption{Comparison of the two different source distributions on the one-step data assimilation example}
    \label{tab:DA-distribution_comparison}
    \begin{tabular}{lcc}
    \toprule
    Source distribution & Sinkhorn distance & Average number of sampling steps\\
    \midrule
    Gaussian & 0.045 & 17 \\
    Prior    & 0.062 & 28 \\
    \bottomrule
    \end{tabular}
\end{table}
\Cref{fig:DA-distributions} shows that both source distributions generate posterior samples that qualitatively capture the disjoint and bimodal structure of the reference solution. However, quantitative metrics in \Cref{tab:DA-distribution_comparison} reveal a meaningful difference in performance. The model trained with the multivariate Gaussian source achieves a lower Sinkhorn distance to the SIR reference posterior, indicating the generated samples are closer to the true distribution. Moreover, it requires fewer integration steps to generate samples. 
Overall, these results suggest that, for this problem, the conditional flow matching model with the Gaussian source provides a relatively more accurate approximation of the target posterior distribution compared to using the prior as the source distribution. Also, sampling the posterior distribution using the conditional flow matching model with the Gaussian as source is more efficient. This insight may help in designing recursive filters for data assimilation using conditional flow matching, which we intend to explore in a future study.

\subsubsection{Advection diffusion reaction}
\label{subsubsec:ADR}


In this section, we present the results for a nonlinear advection diffusion reaction problem defined on the rectangular domain shown in Figure~\ref{fig:dataset_description}, with dimensions $16 \times 4$ units, where the concentration of a chemical species is observed. The concetration field is governed by the following partial differential equation
\begin{eqnarray}
\label{eq:advection_diffusion_reaction}
\nabla \cdot (\bm{a} u) - \kappa \nabla^2 u  - u(r-u) = 0.
\end{eqnarray}
The velocity field, $\bm{a}$, is characterized by a horizontally directed parabolic velocity profile, attaining a maximum magnitude of 12 units. The diffusion coefficient and reaction parameter are set to $\kappa = 8$ and $r = 2$, respectively, which corresponds to a Péclet number of 6. The reaction term follows logistic growth dynamics and acts to stabilize the concentration around the value $r = 2$~\cite{lam2020}. A zero concentration boundary condition is imposed on the left boundary, while a zero-flux condition is prescribed on the right boundary. The top and bottom boundaries are allowed to have non-zero flux values, extending from the left corners up to 7 units into the domain. This flux is parameterized as a piecewise constant function over 15 segments, each of length 0.47 units. More details about this problem are available in \cite{dasgupta2026unifying}.

Collectively, the fluxes on the top and bottom boundaries (i.e., the upper and lower walls, respectively) are represented by a 30-dimensional vector $\X$ ($d = 30$), which constitutes the unknown parameter to be inferred. The observation vector $\Y$ consists of concentration measurements collected by 30 equally spaced sensors ($D = 30$), positioned 0.5 units from the top and bottom boundaries (see \Cref{fig:dataset_description}). The measurements are corrupted by additive, independent Gaussian noise with zero mean and variance $\sigma^2 = 0.01^2$.
\begin{figure}[t]
\centering
\captionsetup[subfigure]{justification=centering}
\begin{subfigure}[t]{0.48\linewidth}
    \centering
    \includegraphics[width=\linewidth]{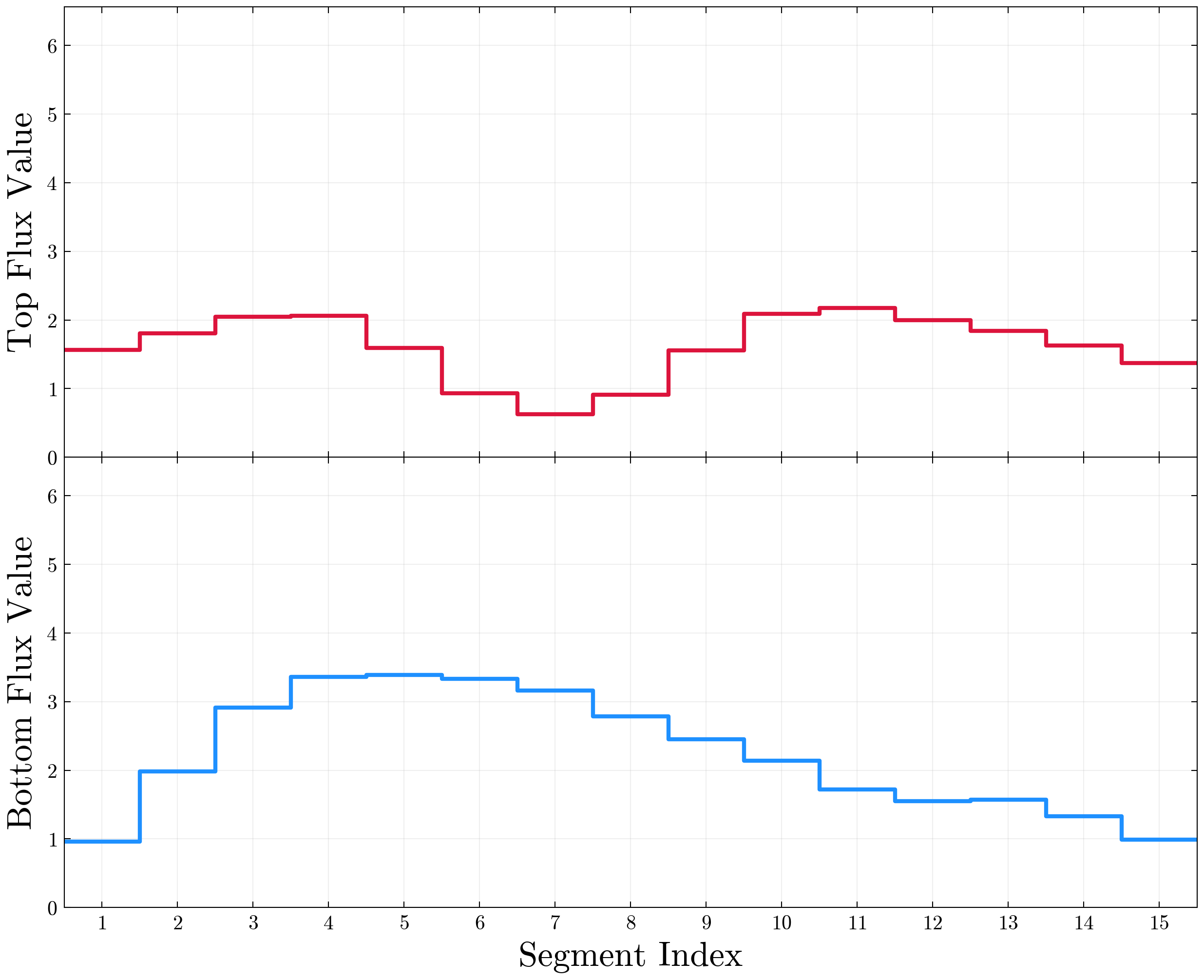}
    \caption{}
\end{subfigure}
\hfill
\begin{subfigure}[t]{0.48\linewidth}
    \centering
    \includegraphics[width=\linewidth]{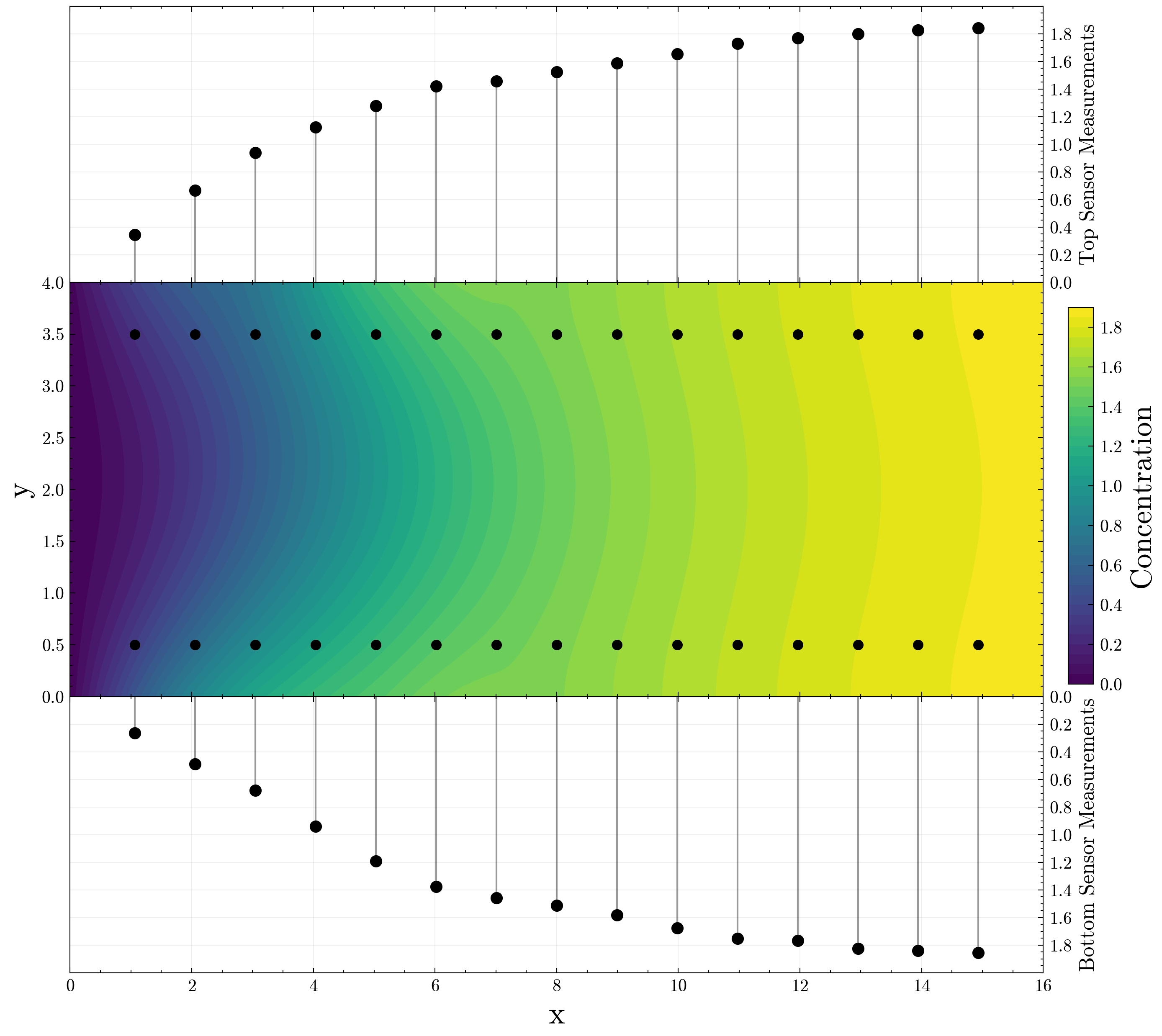}
    \caption{}
\end{subfigure}

\caption{A realization from the advection-diffusion-reaction dataset
(a) Piecewise-constant top and bottom wall fluxes. 
(b) Corresponding concentration field obtained from solving \Cref{eq:advection_diffusion_reaction} and sensor measurements (black dots)}
\label{fig:dataset_description}
\end{figure}

The prior distribution of $\bm{X}$ is modeled as a Gaussian process. The flux values associated with each segment are drawn from a multivariate normal distribution whose covariance matrix is defined by a radial basis function (RBF) kernel. The kernel depends on the Euclidean distance between segment locations and uses a length scale of 2 units, promoting stronger correlations between neighboring segments while still allowing variability across the boundary. The samples are initially generated with zero mean and subsequently shifted by adding 2 to enforce a positive mean flux. Any negative sampled values are clipped to zero to ensure non-negativity of the flux.

To compute the concentration field,  \Cref{eq:advection_diffusion_reaction} is solved numerically using the {FEniCS} \cite{Fenics2012} finite element framework on an $850 \times 200$ mesh comprising 340{,}000 P1 elements. The resulting solutions, shown in \Cref{fig:dataset_description}, indicate that the reaction mechanism drives the concentration toward the equilibrium value of two as it progresses in the downstream direction. Mesh convergence is verified to ensure that the numerical solution provides an accurate approximation of the true solution.

To construct the dataset, we generate 4,000 realizations of $\bm{X}$ and $\bm{Y}$, of which 90\% are used for training and the remaining 10\% for testing. To investigate the effect of dataset size, we also create a smaller dataset consisting of 400 realizations, using the same 90\%--10\% split between training and test sets.
In total, this results in two distinct datasets. For each dataset, we train two conditional flow matching models corresponding to different choices of source distribution: a Gaussian distribution and the prior distribution.

\begin{figure}[!t]
    \centering
    \captionsetup[subfigure]{justification=centering}

    \begin{subfigure}[t]{0.48\textwidth}
        \centering
        \includegraphics[width=\linewidth]{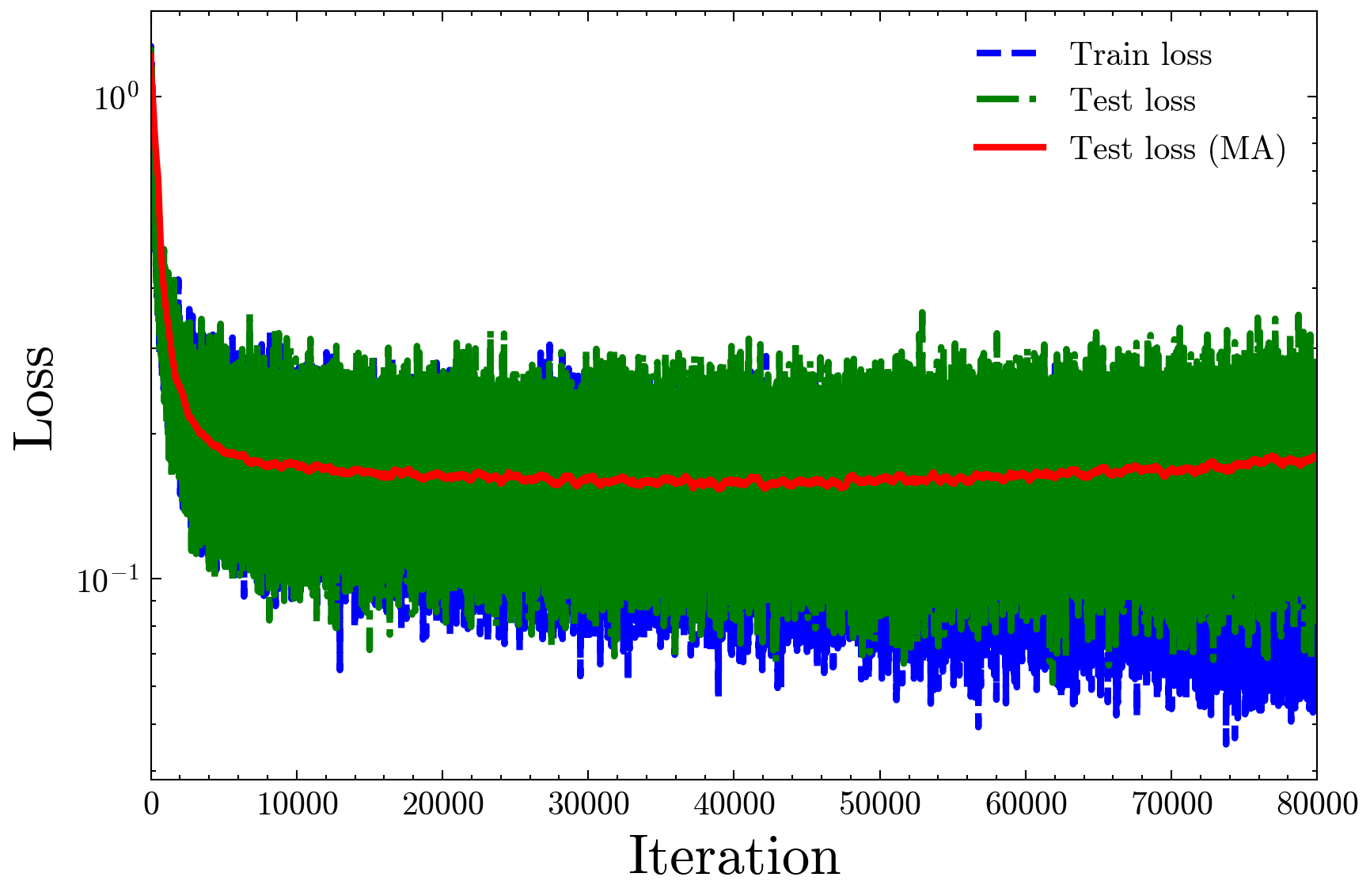}
        \caption{Gaussian source distribution}
    \end{subfigure}
    \hfill
    \begin{subfigure}[t]{0.48\textwidth}
        \centering
        \includegraphics[width=\linewidth]{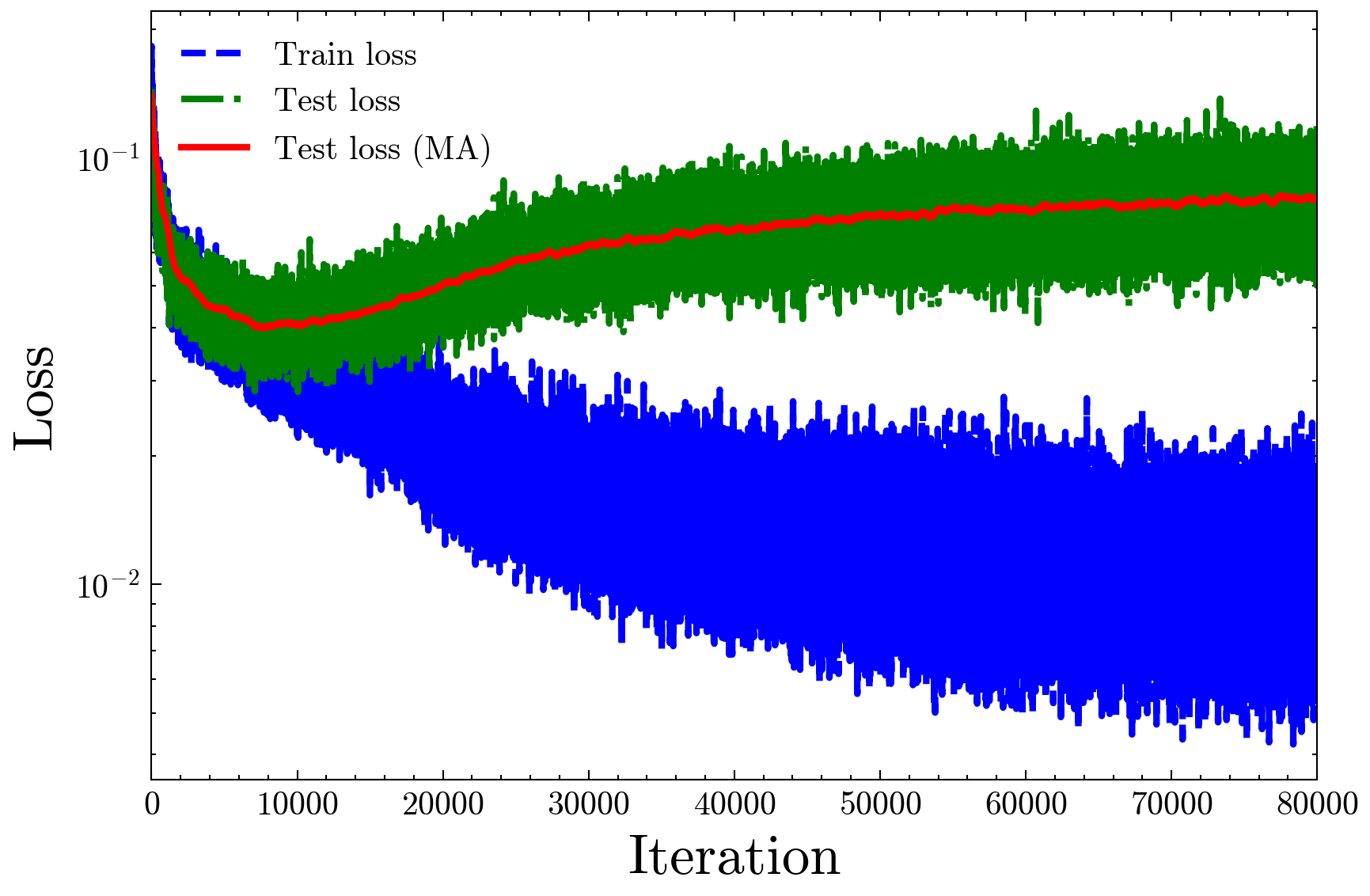}
        \caption{Prior source distribution}
    \end{subfigure}

    \caption{Training and test loss for the advection-diffusion-reaction problem with 400 training samples}
    
    \label{fig:training_400}
\end{figure}
\Cref{fig:training_400} presents the training and test loss curves for the dataset of size 400 using Gaussian and prior source distributions. In both cases, a plateau region in the test loss is observed while the training loss continues to decrease, indicating the onset of overfitting as discussed in Section~3. For the Gaussian source distribution, we choose the velocity network for 20,000 iterations, which lies within the plateau region, to sample the target posterior. When using the prior source distribution training samples, overfitting occurs earlier and the test loss exhibits a clear minimum around iteration 10,000, after which it begins to increase, indicating the onset of overfitting. We therefore select the velocity network trained for 10,000 iterations to sample the target posterior.

\begin{figure}[!t]
\centering

\includegraphics[width=\linewidth]{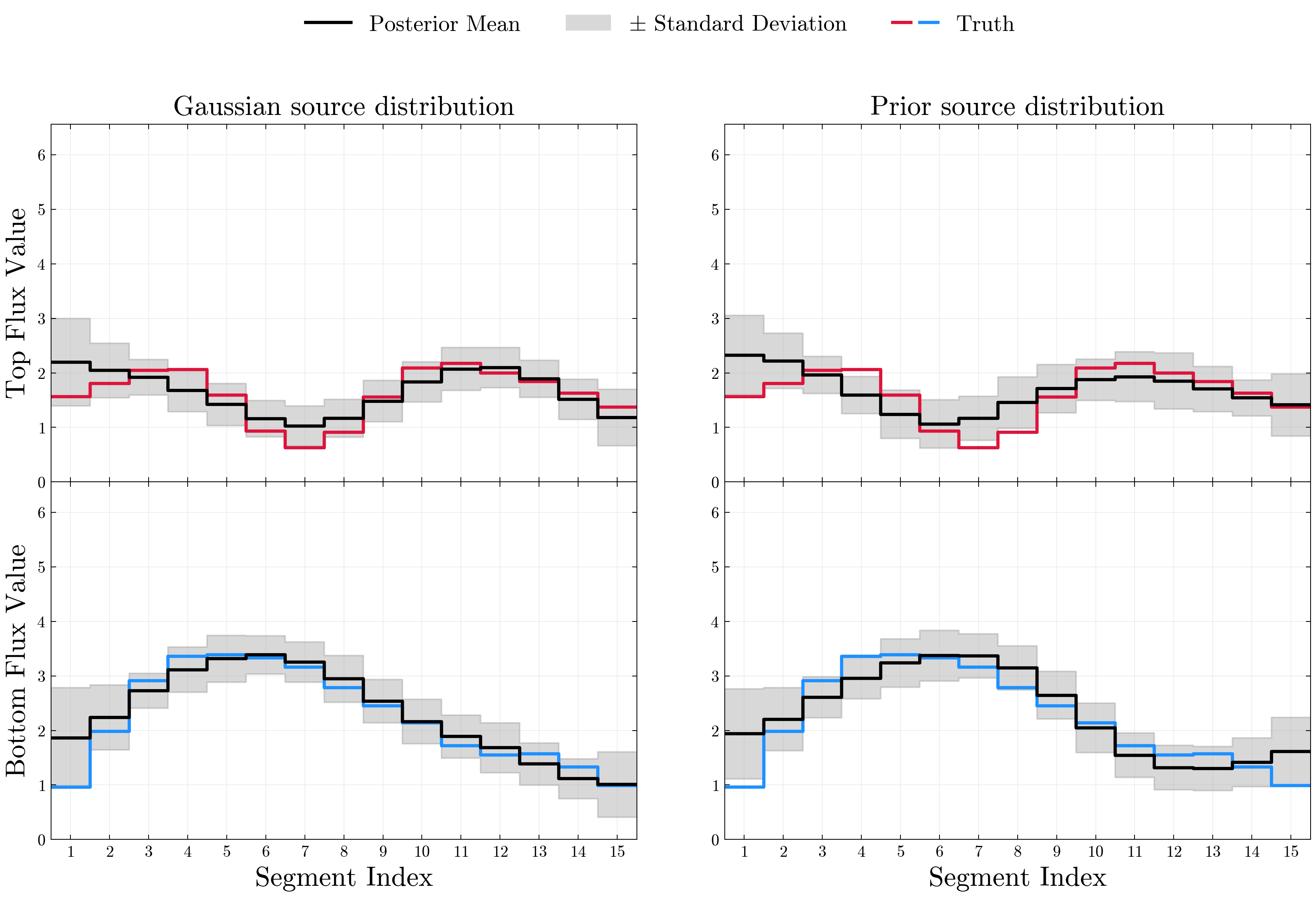}

\caption{Posterior mean and standard deviation of the inferred flux and the true flux for a test case with data size 400 and 1\% noise. Left: Gaussian source distribution. Right: Prior source distribution.}
\label{fig:posterior_mean_s400_n1}
\end{figure}
\Cref{fig:posterior_mean_s400_n1} shows the posterior mean and standard deviation of the inferred flux together with the true flux for a representative test sample  using a Gaussian or the prior as the source density when the training dataset's size is 400. For both source distributions, we observe that the mean flux is close to the true value, and that the true value is almost always contained within one standard deviation of the mean.

\Cref{fig:training_4000} shows the training and test loss curves for the larger dataset of size 4{,}000. Increasing the number of training samples delays the onset of overfitting and significantly stabilizes training (similar to \Cref{fig:conditional_memorization_data_and_loss_data} in \ref{app:finite-data}). The plateau region of the test loss becomes wider, allowing greater flexibility in selecting a stopping point. Based on the loss curves, the checkpoint at 20,000 iterations is selected for both the Gaussian and prior source distributions. 
\begin{figure}[!t]
    \centering
    \captionsetup[subfigure]{justification=centering}

    \begin{subfigure}[t]{0.48\textwidth}
        \centering
        \includegraphics[width=\linewidth]{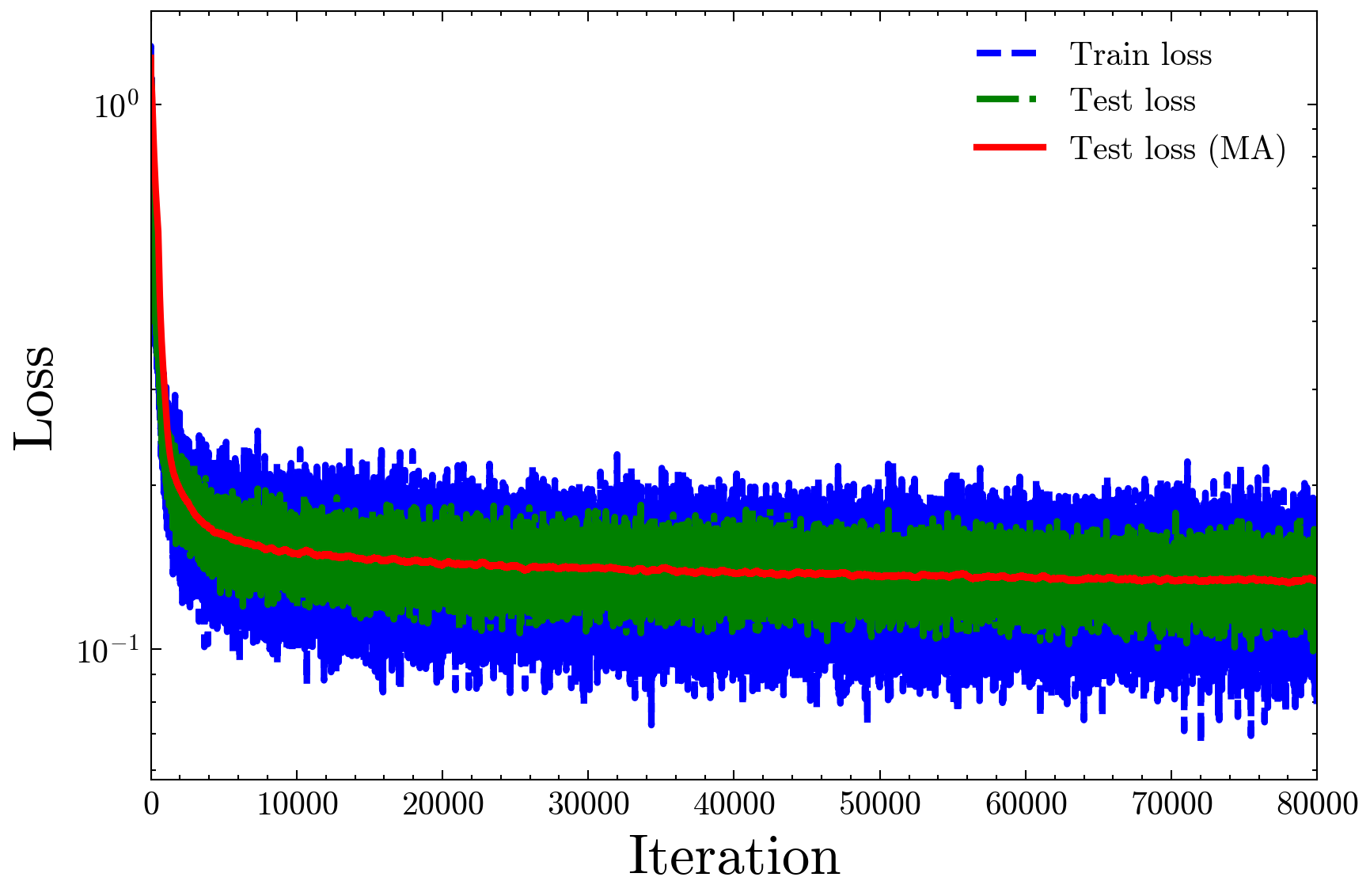}
        \caption{Gaussian source distribution}
    \end{subfigure}
    \hfill
    \begin{subfigure}[t]{0.48\textwidth}
        \centering
        \includegraphics[width=\linewidth]{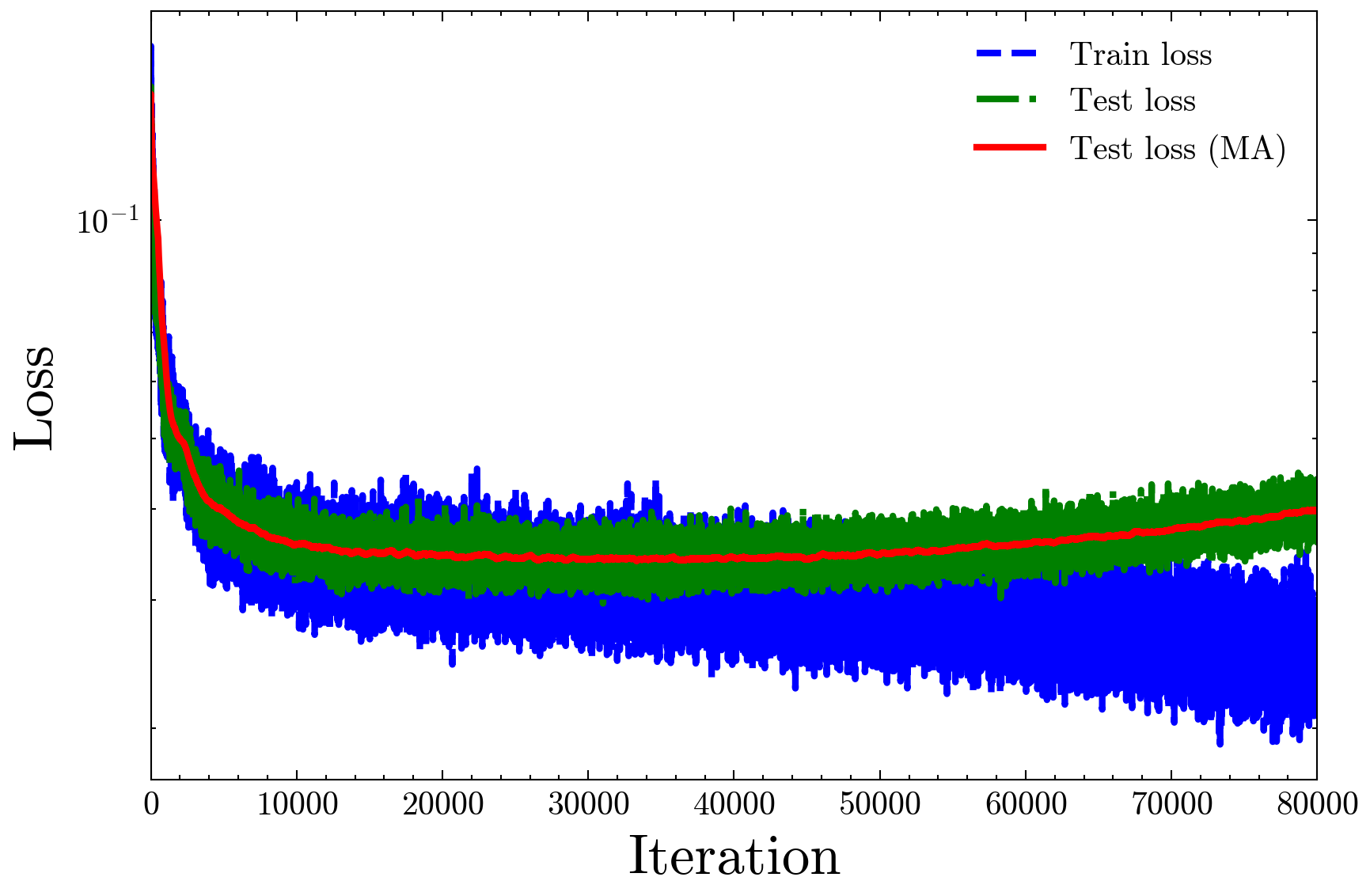}
        \caption{Prior source distribution}
    \end{subfigure}
    \caption{Training and test loss for the advection-diffusion-reaction problem with 4{,}000 training samples}
    \label{fig:training_4000}
\end{figure}

Next, we quantify the performance of the conditional flow matching models (two different source distributions and two training dataset sizes). Specifically, for each sensor measurement in the test set, we generate 100 posterior samples and compute their empirical mean and standard deviation. We then compute the mean relative error of the posterior mean with respect to the true flux, subsequently averaging these quantities for all 30 dimensions and all test samples, which results in a single scalar value for the error and standard deviation at each checkpoint.  We refer to these scalar quantities as the ``error'' and ``std'' in \Cref{tab:flowmatching_comparison}. Note that each experiment is performed using two different random seeds, and the results reported in \Cref{tab:flowmatching_comparison} are averaged over these runs. For the same amount of training data, we observe that model with the Gaussian source distribution incurs somewhat smaller error, when compared to the model with the prior as the source distribution. We also observe that as expected, increasing the amount of data reduces the error and the uncertainty, as seen in the reduced values of average error and average standard deviation.

\begin{table}[!b]
\centering
\caption{Comparison of conditional flow matching models with different types of source distributions and trained using different amounts of training data on the advection-diffusion-reaction problem}
\begin{tabular}{l|lcccc}
\toprule
\multicolumn{2}{l}{Source distribution} & \multicolumn{2}{c}{Gaussian} & \multicolumn{2}{c}{Prior} \\
\midrule
\multicolumn{2}{l}{Size of training data} & 400 & 4000 & 400 & 4000 \\
\midrule
\multirow{3}{*}{\rotatebox[origin=c]{90}{Metric}}
&Average $N_{\text{step}}$ & 13.0 & 14.7 & 9.0 & 9.4 \\
&Average error & 0.144 & 0.139 & 0.173 & 0.142 \\
&Average std & 0.061 & 0.055 & 0.066 & 0.053 \\
\bottomrule
\end{tabular}
\label{tab:flowmatching_comparison}
\end{table}
Beyond reconstruction accuracy, \Cref{tab:flowmatching_comparison} also reports the average number $N_{\text{step}}$ of sampling steps. Using the prior source distribution results in smaller number of integration steps compared to the Gaussian source. This may be explained by the closer proximity of the prior distribution to the posterior, when compared with the standard normal distribution. The cause for the closer proximity are as follows: \textit{(i)} just like the posterior distribution, samples from the prior distribution are normalized between $-1$ and $1$, where as samples from the standard normal distribution are not, and \textit{(ii)} just like the posterior, the prior distribution, which is essentially a Gaussian process, accounts for the correlation between flux values that are spatially close to each other, whereas the standard normal does not account for this correlation.


\begin{rem}
    From the experiments in \Cref{subsec:cond-density,subsubsec:one-step-DA}, we do not observe any empirical benefits from using the prior distribution $\rho_{\X}$ as the source unless it encodes useful information about the posterior such as in \Cref{subsubsec:ADR}. Hence, herein we will consider only the Gaussian source (\ie the multivariate standard normal distribution) for all the numerical examples.
\end{rem}

\subsubsection{Quasi-static elastography}\label{subsubsec:sci}


In this section, we consider the synthetic quasi-static elastography application adapted from~\cite{dasgupta2025conditional}. Quasi-static elastography is a medical imaging technique that uses ultrasound to measure tissue deformation under an external load \cite{barbone2009review}. These displacement measurements are then used to infer the spatially varying shear modulus of the tissue. In this study, the specimen consists of a stiffer circular inclusion of fixed radius embedded within a homogeneous soft background. Therefore, the objective of this inverse problem is to recover the shear modulus field from noisy measurements of the vertical displacement component.

We assume the specimen is linear and isotropic. Under quasi-static conditions and in the absence of body forces, the governing equations for the forward problem consist of the equilibrium equation
\begin{equation}\label{eq:equilibrium_equation}
    \nabla \cdot \bm{\sigma} = 0, 
\end{equation}
and the linear elastic constitutive law
\begin{equation}\label{eq:linear_constitutive_law}
    \bm{\sigma}= 2\mu \nabla^{\mathrm{s}}\uu + \lambda (\nabla\cdot\uu) \bm{I},
\end{equation}
assuming plane stress conditions. In these equations, $\bm{\sigma}$ denotes the Cauchy stress tensor, $\uu$ denotes the displacement field, $\nabla^{\mathrm{s}} \uu$ denotes the symmetric strain tensor, while $\mu$ and $\lambda$ are the Lam\'{e} parameters. The specimen measures $1 \times 1$ cm$^2$ and comprises a homogeneous background containing a uniformly stiff circular inclusion with fixed radius of 0.12 cm. The shear modulus values are fixed at 1.5 kPa for the inclusion and 1 kPa for the background. To simulate compression, the top edge is fixed in the vertical direction and is traction-free in the horizontal direction, while a downward vertical displacement of 0.01 cm is imposed on the bottom edge. The left and right edges are traction-free in both directions, and the top-left corner is pinned to prevent rigid-body motion.

The joint distribution $\prob{\X\Y}(\x, \y)$ represents pairs of shear modulus fields $\X$ and corresponding noisy vertical displacement field $\Y$. A realization of $\X$ is a $56 \times 56$ discretized spatial map of the shear modulus over the specimen. The parametric prior distribution controls the location of the inclusion’s center: the coordinates of the center are sampled independently from uniform distributions $\mathcal{U}(0.2, 0.8)$ cm along both spatial directions, ensuring that the inclusion remains fully inside the domain. Therefore, the parametric prior in this experiment is two-dimensional. For a given realization $\x$ of $\X$, the corresponding measurement $\y$ is generated by solving the forward elasticity problem using the finite element method. The resulting vertical displacement field is discretized on the same $56 \times 56$ Cartesian grid as the shear modulus field, so that $d = D = 56 \times 56$. Subsequently, homoscedastic zero-mean Gaussian measurement noise with standard deviation $\sigma_{\eta} \times u_{\max}$ is added to the displacement field, where $u_{\max}$ denotes the maximum vertical displacement across all training samples. The noise is truncated to the interval $[-3\sigma_{\eta} u_{\max},\, 3\sigma_{\eta} u_{\max}]$. We consider the noise level $\sigma_{\eta} = 0.05$. The training dataset consists of 10,000 paired samples drawn from the joint distribution $\prob{\X\Y}$. \Cref{fig:sci_train_samples} shows five data points randomly sampled from the training dataset. An additional 1,000 samples are used for testing. 
\begin{figure}[!t]
    \centering
    \includegraphics[width=0.85\linewidth]{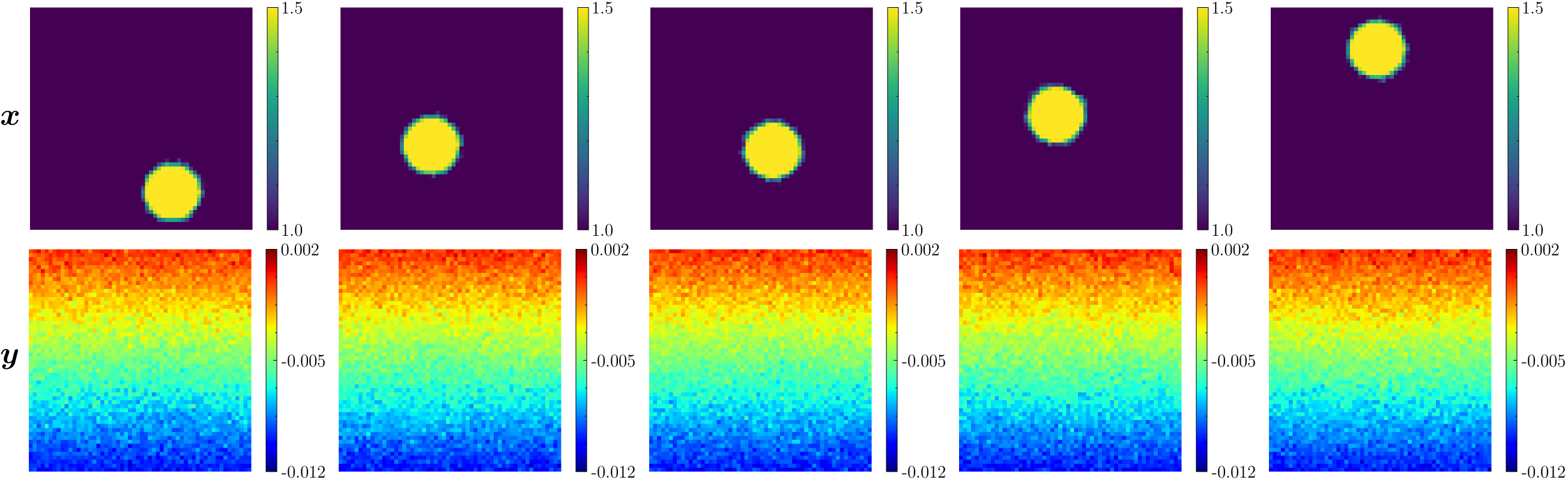}
    \caption{Five realizations of $\mathbf{X}$ and $\mathbf{Y}$ sampled from the joint distribution of the training dataset for the synthetic quasi-static elastography application. 
    The first row shows the shear modulus field, and the second row shows the corresponding noisy vertical displacement measurements}
    \label{fig:sci_train_samples}
\end{figure}

\begin{figure}[!b]
    \centering
    \includegraphics[width=0.70\linewidth]{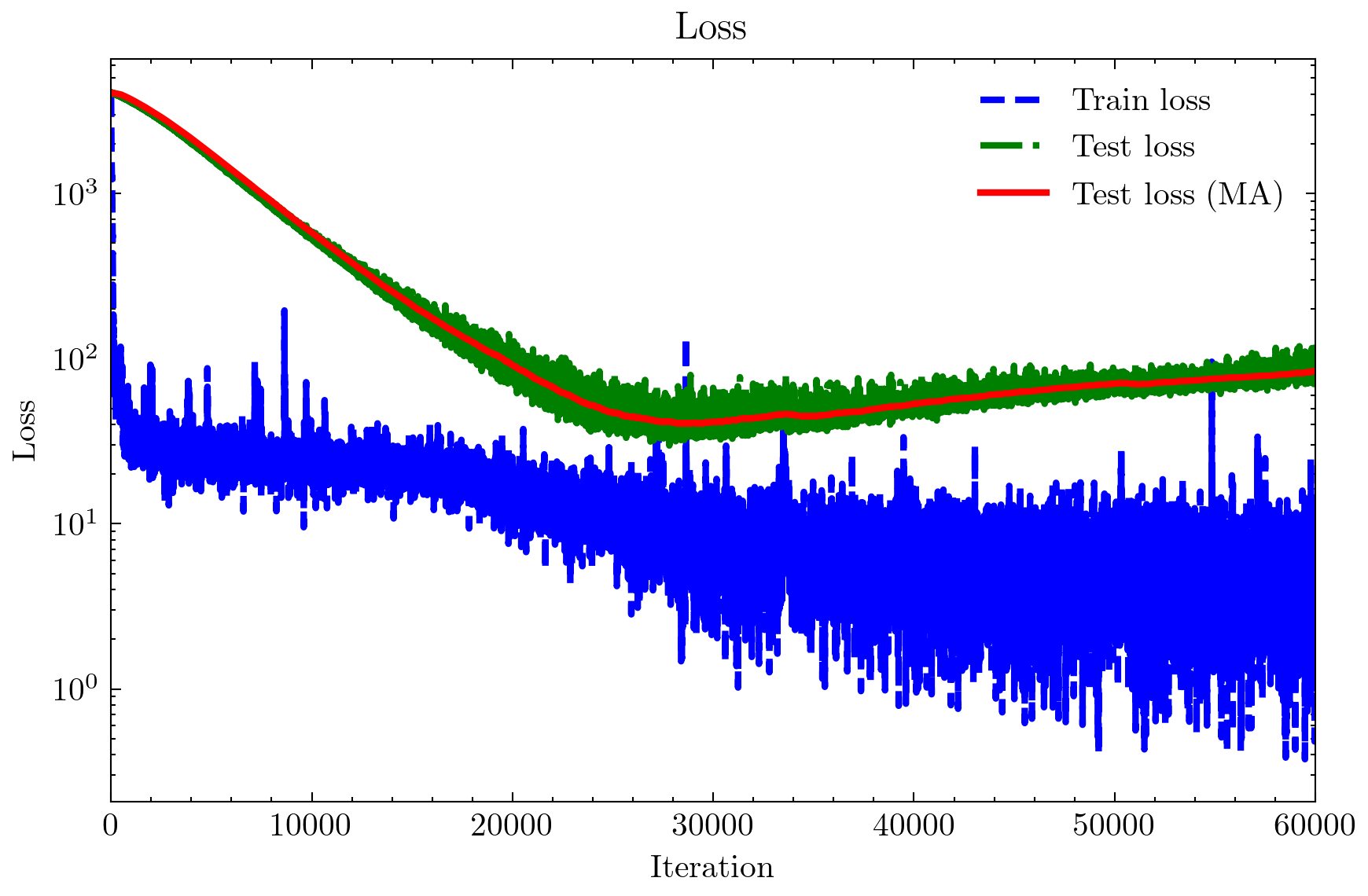}
    \caption{Training loss, test loss, and moving average of the test loss for the synthetic quasi-static elastography application}
    \label{fig:sci-loss_gaussian}
\end{figure}
Next, we train the velocity network with a Gaussian source distribution using the training dataset. \Cref{fig:sci-loss_gaussian} shows the training and test losses, together with the moving average of the test loss (averaged over the last 500 iterations). We use the velocity network trained for 24{,}000 iterations to generate samples from the posterior distribution because the moving average of the test loss reaches its minimum around this checkpoint.

\begin{figure}[!b]
    \centering
    \includegraphics[width=\linewidth]{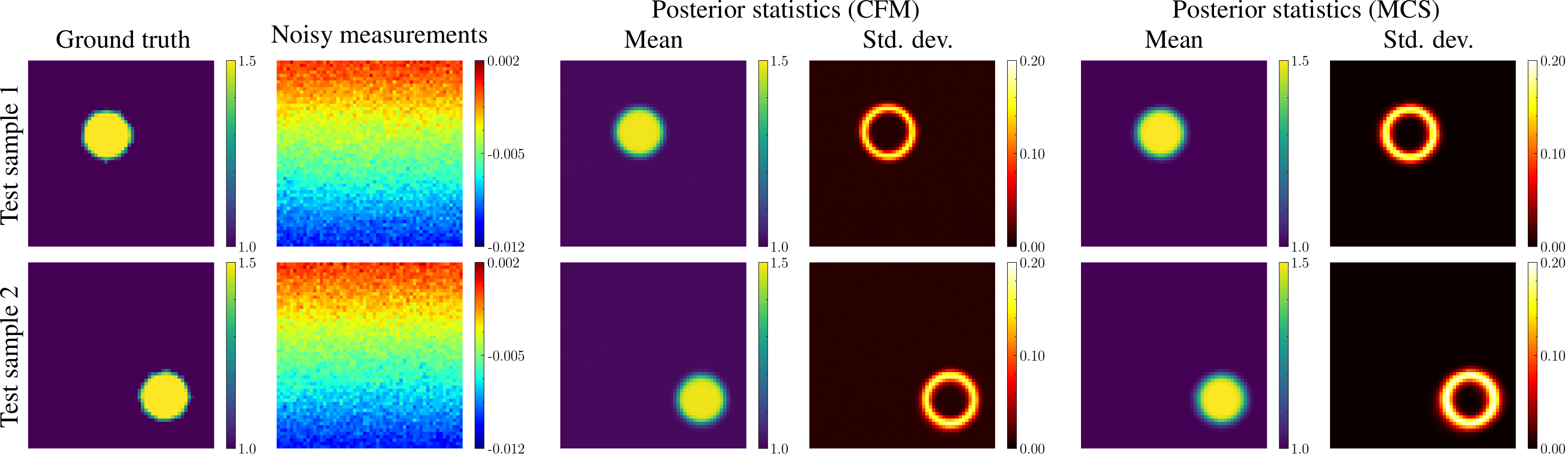}
    \caption{Posterior statistics estimated using the MCS and trained velocity network on two test samples for the synthetic quasi-static elastography application}
    \label{fig:sci-results}
\end{figure}
\begin{table}[!b]
\centering
\caption{RMSE between statistics estimated using tconditional flow matching and MCS, and the number of sampling steps for the test samples in \Cref{fig:sci-results}} 
\label{tab:rmse_samples}
\begin{tabular}{cccc}
\toprule
\multirow{2}{*}{\makecell{Test\\sample}}& \multicolumn{2}{c}{RMSE in posterior statistics} & \multirow{2}{*}{\makecell{Average number\\of sampling steps}}\\
\cline{2-3}
 & Mean & Std. dev & \\
\midrule
1 & 0.014 & 0.012 & 5 \\
2 & 0.012 & 0.011 & 7 \\
\bottomrule
\end{tabular}
\end{table}
Using the trained model at the selected checkpoint, we generate 4{,}000 realizations from the posterior distribution and compute the corresponding posterior statistics for two test samples that were not part of the training dataset. We also use Monte Carlo simulation (MCS) to compare the estimated posterior statistics with a reference solution. We use a sample of size 500{,}000, significantly larger than the training dataset, to control the variance in the MC estimates of the posterior statistics. The posterior statistics estimated using MCS serve as the reference statistics in this experiment. Additional details regarding the MCS procedure can be found in \cite{dasgupta2025conditional}. \Cref{fig:sci-results} compares the posterior statistics estimated using conditional flow matching and MCS on two test samples that were not part of the training dataset. The first and second columns in \Cref{fig:sci-results} shows the `true' shear modulus field and the corresponding noisy measurements, respectively. In \Cref{fig:sci-results}, the third and fourth column shows the posterior statistics (pixel-wise mean and standard deviation) estimated from the realizations generated using trained velocity network, while the last two columns show the reference posterior statistics estimated using MCS. These reuslts show that the conditional flow matching method with an optimally trained velocity network is able to capture the posterior statistics in both test cases. Although the conditional flow matching method yields a pixel-wise standard deviation in the posterior that is slightly lower than the reference, it correctly reflects greater uncertainty near the inclusion’s boundary. \Cref{tab:rmse_samples} assesses the quality of the statistics estimated by the conditional flow matching method by reporting their root mean squared error (RMSE) relative to the reference statistics computed via MCS.  \Cref{tab:rmse_samples} also reports the average number of sampling steps, indicating that only a few integration steps are necessary to sample the target posterior. 

\begin{rem}
    In earlier work by \citet{dasgupta2025conditional}, a discrete-time conditional diffusion model was used to solve the same inverse problem. Achieving comparable inference quality in \cite{dasgupta2025conditional} required 640 sampling steps, which is substantially more than the number of steps reported in \Cref{tab:rmse_samples}. This highlights an advantage of the continuous-time conditional flow matching formulation, whereby adaptive integration of \Cref{eq:ode-con} helps  significantly improve sampling efficiency.
\end{rem}

We also sampled from the trained velocity field after 60{,}000 iterations. \Cref{fig:sci-results-overfit} shows the estimated statistics computed using realizations from the target posterior in each case. As can be seen, the standard deviation is diminished. Comparing \Cref{fig:sci-results-overfit,fig:sci-results}, we find that the pixel-wise standard deviation estimated with an over-trained velocity network is substantially smaller than that obtained with an optimally trained velocity network. This further highlights the importance of early stopping of training to avoid degeneracy in the posterior approximation. 
\begin{figure}[H]
    \centering
    \includegraphics[width=0.85\textwidth]{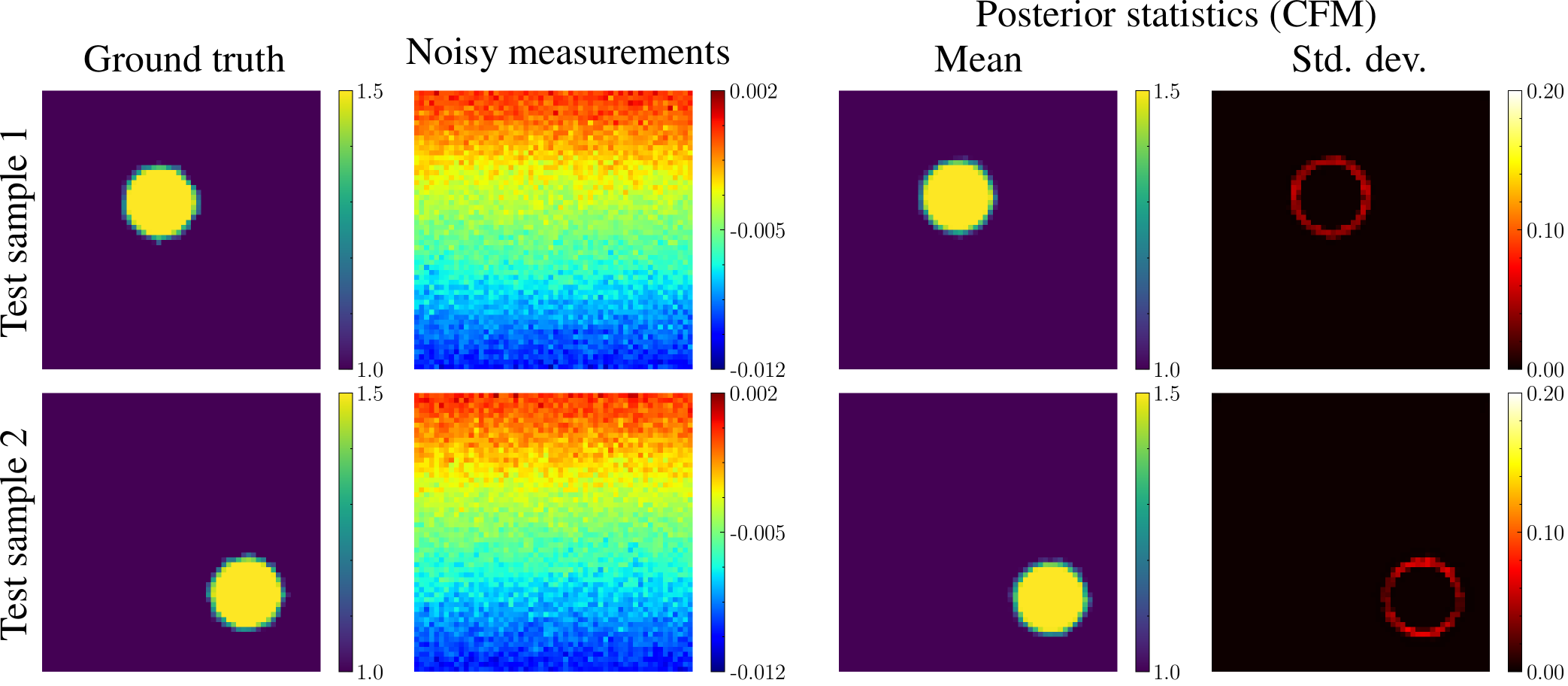}
    \caption{Posterior statistics estimated using a severely trained velocity network ($60{,}000$ iterations) on two test samples for the quasi-static elastography application}
    \label{fig:sci-results-overfit}
\end{figure}

\subsection{Applications with experimental data}

\subsubsection{Quasi-static elastography of a circular inclusion within homogeneous media}\label{subsubsec:elasto2}

In this example, we apply conditional flow matching to the inverse elasticity problem arising in quasi-static elastography. We adapt this example from~\cite{ray2022efficacy,dasgupta2025conditional}. Specifically, we wish to infer the spatial distribution of the shear modulus of a specimen using noisy full-field measurements of the vertical displacements. 
We assume the relation between the displacements and shear modulus for an elastic, isotropic, and incompressible medium in the absence of body forces to follow the equilibrium and constitutive equations, \Cref{eq:equilibrium_equation,eq:linear_constitutive_law}, respectively, %
under plane stress assumptions~\cite{ray2022efficacy}. 
The specimen is 34.608$\times$26.297 mm\textsuperscript{2}. We assume the left and right edges to be traction free, and the top and bottom surfaces to be traction free along the horizontal direction. We subject the top and bottom edges of the specimen to vertical displacements of 0.084 mm and 0.392 mm, respectively, to simulate compression of the specimen. 

\begin{table}[!t]
    \caption{Details of random variables comprising the parametric prior distribution for $\X$ in the quasi-static elastography application}
    \label{tab:elasto2_prior}
    \centering
    \begin{tabular}{lc}
        \toprule
        Random variable & Distribution \\
        \toprule
        Distance of the inclusion's center from the left edge (mm) & $\mathcal{U}(7.1,19.2)$ \\
        Distance of the inclusion's center from the bottom edge (mm) & $\mathcal{U}(7.1,27.6)$ \\
        Radius of the inclusion (mm)  & $\mathcal{U}(3.5,7.0)$ \\
        Ratio between the inclusion's and background's shear modulus (mm) & $\mathcal{U}(1,8)$ \\
        \bottomrule
    \end{tabular}
\end{table}
\begin{figure}[!t]
    \centering
    \includegraphics[width=0.84\textwidth]{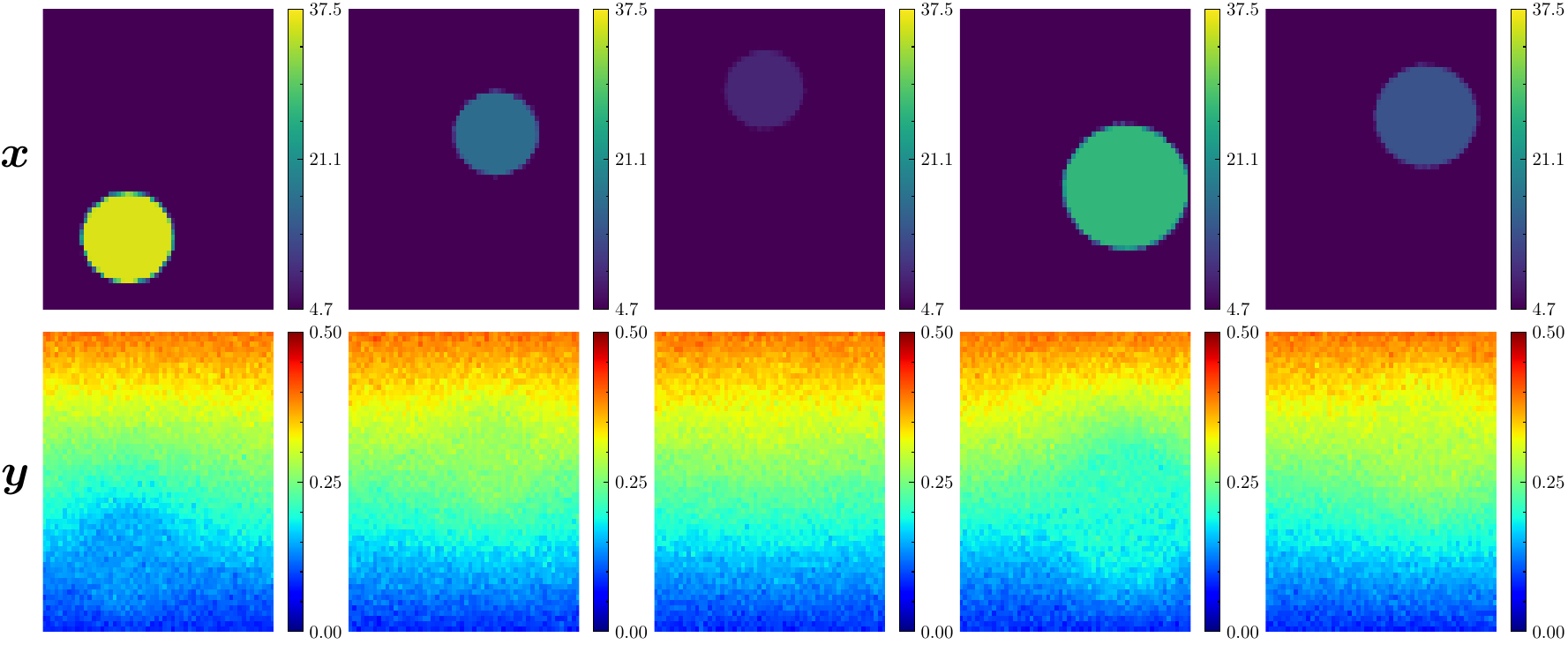}
    \caption{Five typical realizations of $\X$ and $\Y$ sampled from the training dataset for the quasi-static elastography application with experimental data. The first row shows the spatial distribution of the shear modulus field, and the second row shows the corresponding measurements of the noisy vertical displacement field}
    \label{fig:elasto2_train_samples}
\end{figure}
A realization of $\X$ corresponds to the spatial distribution of the shear modulus of a specimen, which we discretize over a 56$\times$56 Cartesian grid. These realizations are sampled from a suitable parametric prior designed to model stiff circular inclusions in a uniform background of 4.7 kPa, consisting of four random variables that control the coordinates of the center of the inclusion, the radius of the inclusion, and the ratio of the shear modulus of the inclusion with respect to the shear modulus of the background. \Cref{tab:elasto2_prior} provides details regarding these random variables. We propagate a realization of $\X$ through a finite element model with linear triangular elements, and add independent, homoscedastic Gaussian noise with standard deviation equal to 0.001 mm to the predicted discrete vertical displacement field to obtain the corresponding realization of $\Y$. The training dataset consists of 8000 such iid pairs of $\X$ and $\Y$ sampled from the joint $\prob{\X\Y}$. \Cref{fig:elasto2_train_samples} shows five data points randomly sampled from the training dataset. 

Next, we train the velocity network using the training dataset and a Gaussian source distribution. \Cref{fig:elasto2_losses} shows the training and test losses, together with the moving average of the test loss. We choose to sample the posterior distribution with the velocity network trained for 50,000 iterations, in the region where the test loss saturates. 
\begin{figure}[t]
    \centering
    \includegraphics{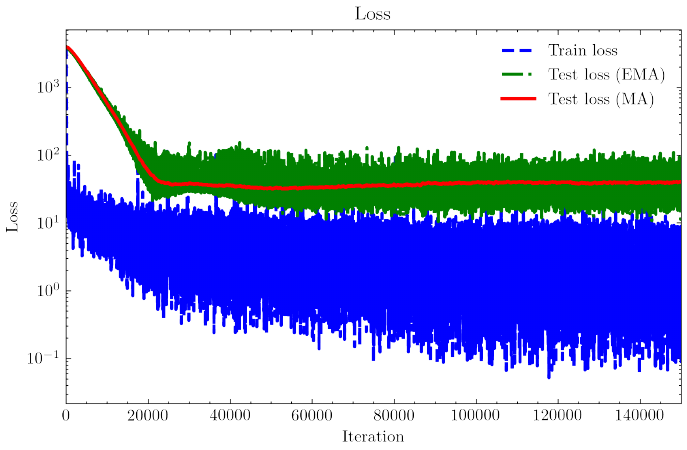}
    \caption{Training loss, test loss, and the moving average of the test loss for the velocity network in the quasi-static elastography application with experimental data}
    \label{fig:elasto2_losses}
\end{figure}

\paragraph{Validation on synthetic test data}  We use the trained velocity networks at the chosen checkpoint to generate 800 realizations from the posterior distribution corresponding to two test measurements not in the training dataset and infer the spatial distribution of the shear modulus in both test specimens. \Cref{fig:elasto2_gaussian_test_results} shows the results of these tests for the Gaussian source distribution. The first column in \Cref{fig:elasto2_gaussian_test_results} shows the `true' spatial distribution of the shear modulus in the two test specimens, whereas the second column in \Cref{fig:elasto2_gaussian_test_results} shows the corresponding full-field noisy measurements of vertical displacement. The third and fourth columns in \Cref{fig:elasto2_gaussian_test_results} show the estimated pixel-wise posterior mean and standard deviation of the posterior realizations for the two test cases. The last column in \Cref{fig:elasto2_gaussian_test_results} shows the absolute pixel-wise error between the estimated pixel-wise posterior mean and the corresponding ground truth. \Cref{tab:elasto2_test_results_rmse} tabulates the root mean squared error (RMSE) between the estimated pixel-wise posterior mean and the ground truth across the two test samples. From \Cref{tab:elasto2_test_results_rmse} and \Cref{fig:elasto2_gaussian_test_results}, we conclude that the inferred spatial distribution of the shear modulus of both test specimens is close to the ground truth. Moreover, we observe that the pixel-wise standard deviation is larger along the periphery of the inclusion rather than in its interior, which indicates greater uncertainty about the inclusion's location relative to its stiffness. \Cref{tab:elasto2_test_results_rmse} also provides the average number of sampling steps for both test cases. The average is taken across all the posterior samples in each test case. \Cref{tab:elasto2_test_results_rmse} shows that only a few steps is necessary for posterior sampling using the conditional flow matching model, which again confirms that generating posterior samples is efficient. In contrast,  \citet{dasgupta2025conditional} reports using 1280 steps to sample the posterior using a discrete version of the conditional diffusion model for the same problem. 
\begin{figure}[t]
    \centering
    \includegraphics[width=0.85\textwidth]{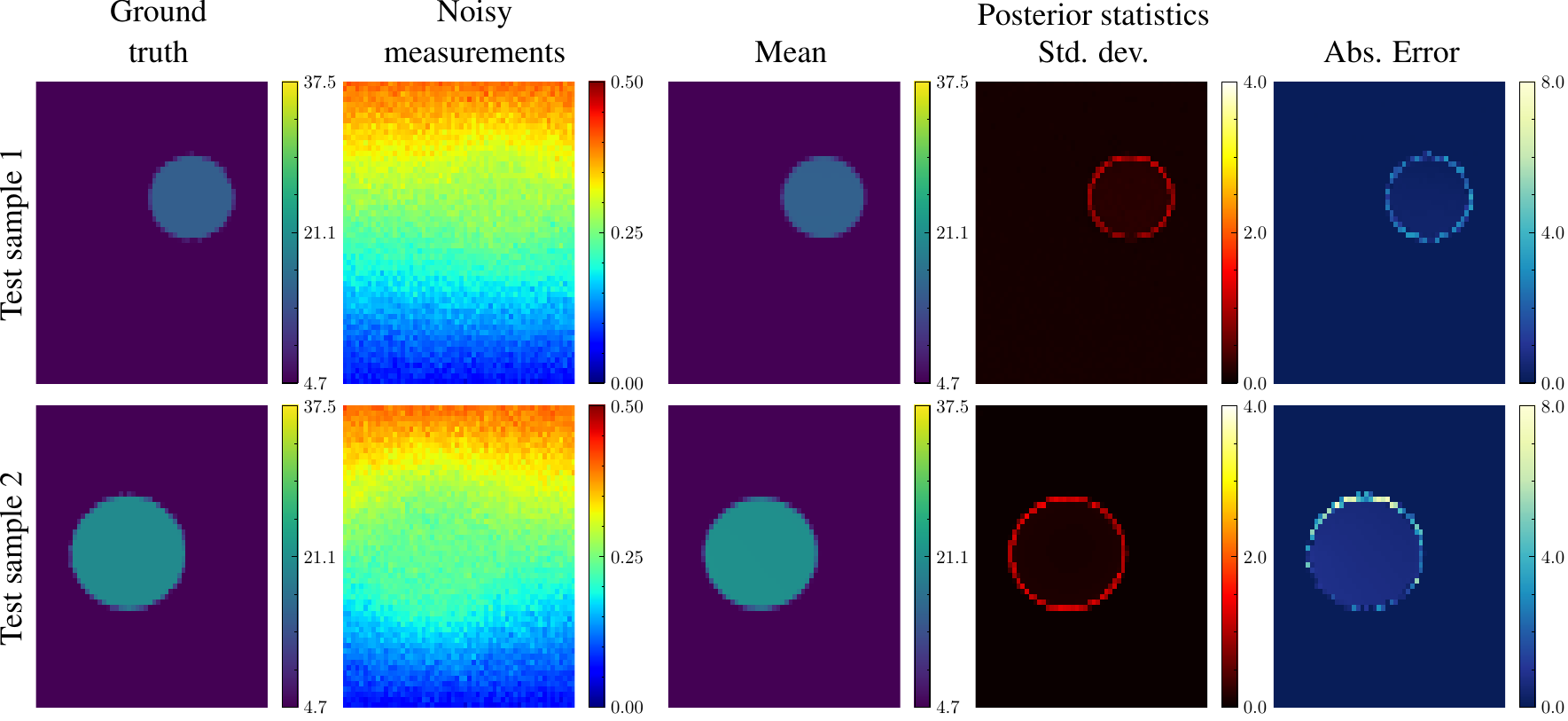}
    \caption{Posterior statistics estimated using the trained velocity network on two test samples for the quasi-static elastography application with experimental data}
    \label{fig:elasto2_gaussian_test_results}
\end{figure}
\begin{table}[!t]
    \caption{RMSE between the posterior mean and the ground truth and  average number of sampling steps for two test samples for the quasi-static elastography application with experimental data}
    \label{tab:elasto2_test_results_rmse}
    \centering
    \begin{tabular}{ccc}
    \toprule
    Test sample & RMSE & Avg. number of sampling steps \\
    \midrule
     1  &  0.300 & 9\\
     2  &  0.563 & 16\\
    \bottomrule
    \end{tabular}
\end{table}  
\begin{figure}[!t]
    \centering
    \includegraphics[width=0.85\textwidth]{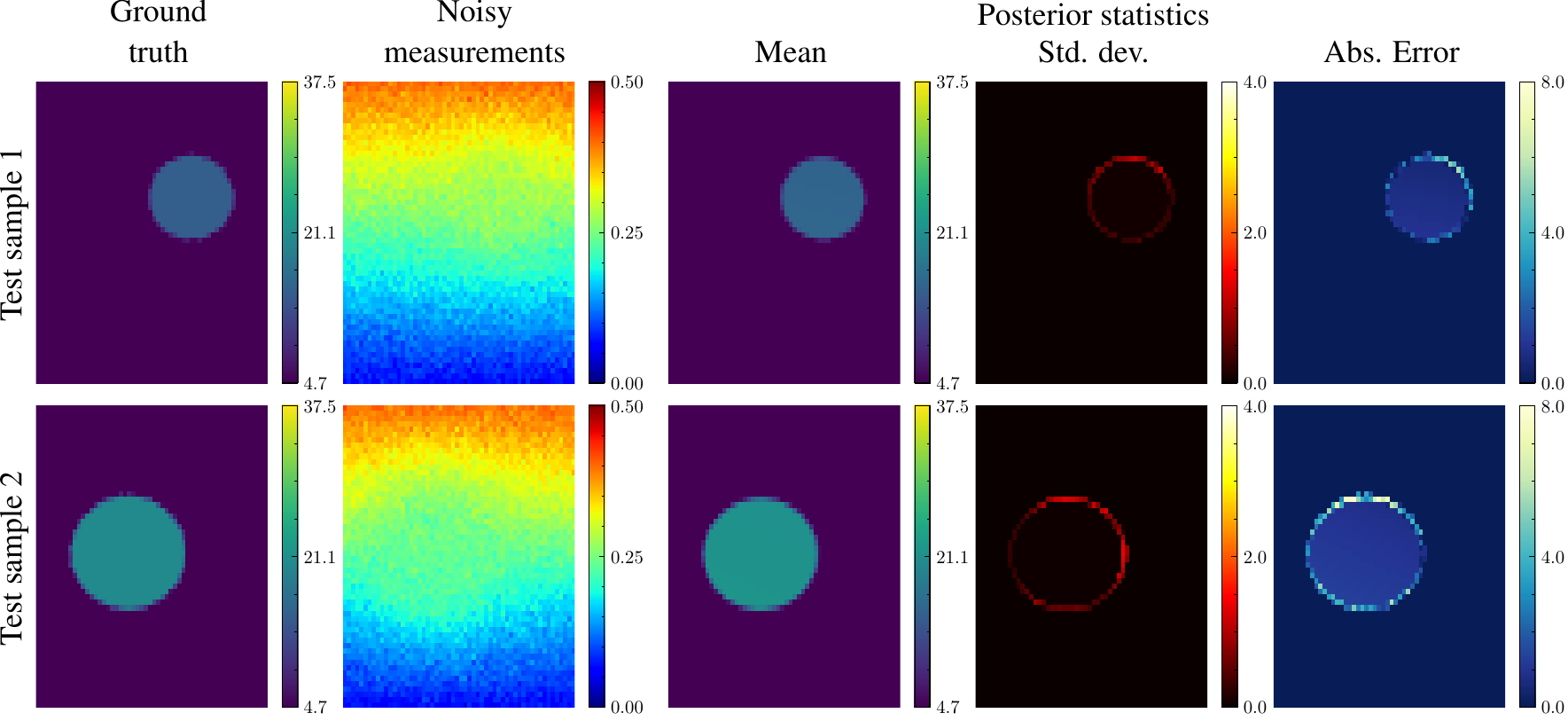}
    \caption{Posterior statistics estimated using a severely over-trained velocity network on two test samples shown in \Cref{fig:elasto2_gaussian_test_results}}
    \label{fig:elasto2_gaussian_test_results2}
\end{figure}


Further, \Cref{fig:elasto2_gaussian_test_results2} shows the results obtained using a velocity network trained for 300,000 iterations using a Gaussian source distribution. \Cref{fig:elasto2_gaussian_test_results2} further confirms the effect of overtraining, which manifests as a reduction in the pixel-wise standard deviation among the samples from the posterior. The average pixel-wise standard deviation for test samples 1 and 2 drops from 0.093 and 0.047, respectively, to 0.015 and 0.022, respectively, from the optimally trained case to the over-trained case. Moreover, over-training also leads to inferior inference quality: the RMSE between the posterior mean and ground truth increases to 0.375 and 0.723 for the two test cases.

\paragraph{Application on experimental test data} We use the trained velocity network to infer the spatially varying shear modulus of a tissue-mimicking phantom, consisting of a stiff inclusion embedded inside a softer substrate~\cite{pavan2012nonlinear}. A mixture of gelatin, agar, and oil was used to manufacture the phantom. In the laboratory experiment, the phantom was gently compressed, and the vertical displacement was measured using ultrasound scans. These measurements are shown in the first column of \Cref{fig:elasto2_gaussian_exp_results}. The second and third columns in \Cref{fig:elasto2_gaussian_exp_results} show the pixel-wise posterior mean and standard deviation estimated from 800 posterior realizations. From the estimated posterior mean, we extract the inclusion’s average stiffness and diameter to be 12.67 kPa and 10.8 mm, respectively. These estimates are close to the corresponding experimental measurements of 10.7 kPa and 10.0 mm, respectively, and consistent with the findings from a previous study by \citet{dasgupta2025conditional}, where a conditional diffusion model, trained on the same training data, estimated the average stiffness and diameter of the inclusion to be 12.94 kPa and 10.8 mm, respectively. In this case involving experimental data, the conditional flow matching models require 8 sampling steps on average. 
\begin{figure}[!t]
    \centering
    \includegraphics[width=0.50\textwidth]{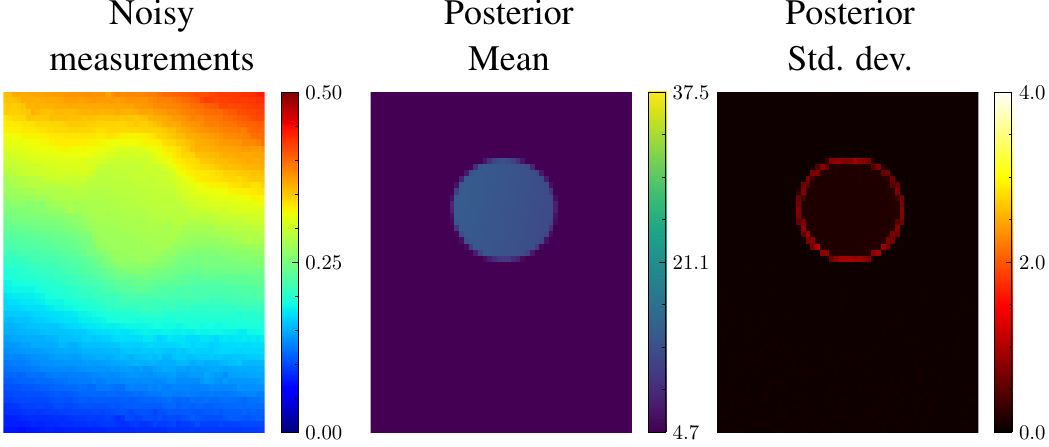}
    \caption{Posterior statistics estimated using the trained velocity network on experimental data for the quasi-static elastography application}
    \label{fig:elasto2_gaussian_exp_results}
\end{figure}

\subsubsection{Optical coherence elastography of tumor spheroids}\label{subsubsec:TS}

This application concerns the mechano-microscopy of tumor spheroids --- a collection of tumor cells. Mechano-microscopy is a type of phase-sensitive compression optical coherence elastography (OCE) that uses optical coherence microscopy (OCM), which is a high-resolution variant of optical coherence tomography  (OCT)~\cite{kennedy2014review,kennedy2014optical}. The objective is to infer the mechanical state of a tumor spheroid specimen from backscattered light as the specimen undergoes compression~\cite{kennedy2014review,kennedy2014optical}. We adapt this application from \cite{foo2024tumor,dasgupta2025conditional} and avoid providing an extensive background on OCE for brevity; instead, interested readers may refer to these references for relevant details. Briefly, this inverse problem involves inferring the spatial distribution of the Young's modulus in a tumor spheroid specimen, $\X$,  from the noisy phase difference measurements, $\Y$. For details on the physical experiment, specimen preparation, and data acquisition (omitted here for brevity), we refer readers to \cite[Appendix C]{dasgupta2025conditional}. We briefly discuss the forward physics and measurement model next. 

\begin{figure}[t]
	\centering
	\includegraphics[height=2in]{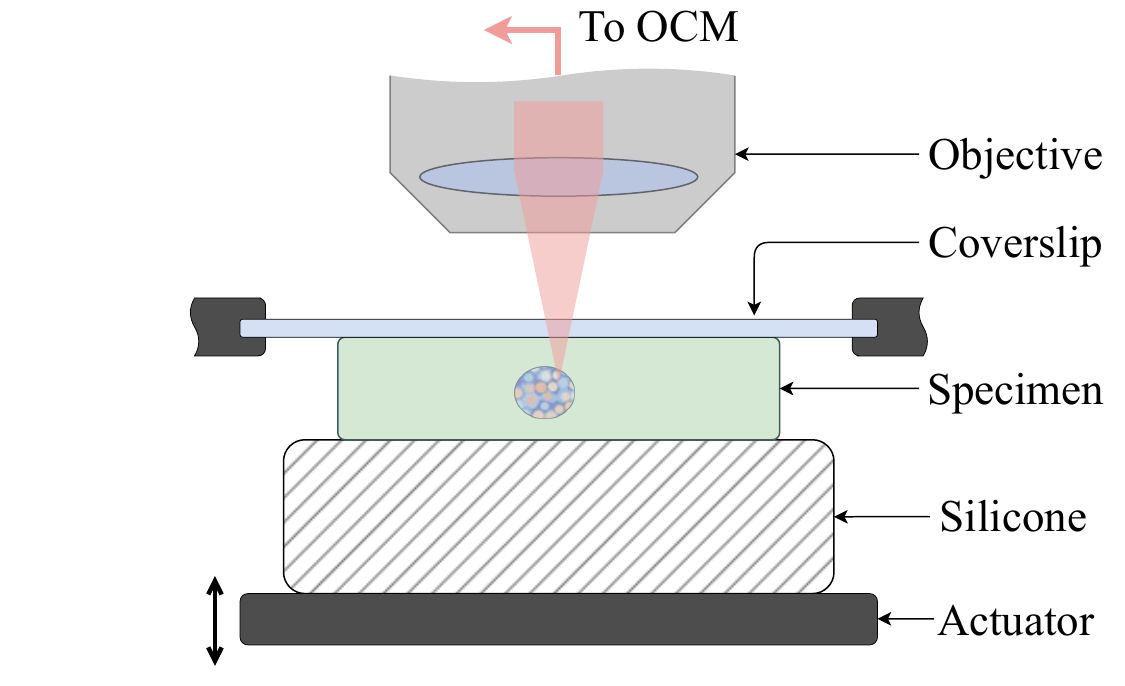}
	\caption{Experimental setup in the tumor spheroid application (adapted from \cite{dasgupta2025conditional})}
	\label{fig:ts_experimental_setup}
\end{figure}
\Cref{fig:ts_experimental_setup}, adapted from \cite{dasgupta2025conditional}, shows a schematic of the experimental setup. Under compression, the forward physics is governed by the equilibrium equation \Cref{eq:equilibrium_equation} for a linear, elastic, isotropic, nearly incompressible (the Poisson’s ratio is set to 0.49) solid undergoing finite deformation. Each specimen represents a tumor spheroid placed in hydrogel. The physical domain is 256$\times$256~$\mu$m\textsuperscript{2}. The Young's modulus field is obtained by sampling from a parametric prior (see \cite{dasgupta2025conditional} for details). The parametric prior consists of nine truncated Gaussian random variables controlling the width, height and location of the spheroid, the number of nuclei per square area, the radius of each nuclei, and the Young's modulus of each nuclei, cytoplasm and surrounding hydrogel. Moreover, the locations of the nuclei are decided using a custom algorithm described in \cite[Appendix C]{dasgupta2025conditional}. 

Using the Young's modulus field, realized following the above procedure, a CAD model is developed for the specimen. Then we conduct finite element analysis (FEA), using commercial software, to simulate uniaxial compression on the specimen under plain strain conditions. The boundary conditions are as follows: we specify the displacement along the top edge, the bottom-left corner of the specimen is fixed, and the vertical displacement along the bottom edge is constrained. The horizontal displacement of the top edge is zero, while the vertical displacement is sampled from a truncated Gaussian distribution. The FEA analysis yields the vertical component of the displacement field, which is manipulated to obtain the phase field (see \cite[Appendix C]{dasgupta2025conditional} for the relation between vertical displacement and phase). To this phase field, we add non-homogeneous, non-Gaussian measurement noise with a depth-dependent statistics. Finally, the noisy phase field is wrapped so that all measurements are between $(-\pi, \pi]$. See \cite[Appendix C]{dasgupta2025conditional} for more details regarding the FEA and measurement noise model.

In summary, a realization of $\X$ is the Young's modulus field discretized over a 256$\times$256 Cartesian grid and the corresponding realization of $\Y$ is the noisy wrapped phase difference field at the same grid points. Also note, $d = D = $ 256$\times$256 in this application. Following \cite{dasgupta2025conditional}, the synthetic dataset for this example consists of 30,000 paired realizations of $\X$ and $\Y$. Of these, 5000 realizations are simulated as described above. The rest are obtained via data augmentation. From this synthetic dataset, 24000 images are used for training and 6000 images for testing. Additionally, the Young's modulus field is rescaled as follows:
\begin{equation}
	E^\prime = \log_{10} \left( \frac{E}{E_{\text{hydrogel}}}\right),
\end{equation}
where $E^\prime$ is the rescaled Young's modulus. \Cref{fig:ts_sample} illustrates five randomly selected paired realizations from the dataset. 
\begin{figure}[t]
    \centering
    \includegraphics[width=0.85\linewidth]{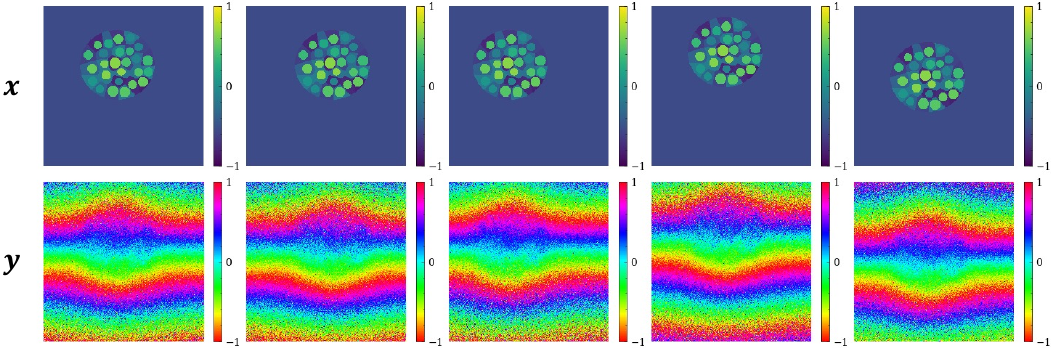}
    \caption{Five realizations of $\X$ and $\Y$ sampled from the joint distribution forming the training dataset for the tumor spheroid application. In the first row are instances of the log-normalized Young's modulus fields, and in the second row are corresponding instances of the noisy measurements. All values have been normalized to [-1,1]}
    \label{fig:ts_sample}
\end{figure}

The velocity network is trained using a Gaussian source distribution for 160,000 iterations. \Cref{fig:tumor_losses} presents both the training and test loss curves, along with the moving average of the test loss. We observe that the moving average test loss achieves its minimum value around 90,000 iterations and then increases slightly, reducing again after 140,000 iterations. We chose the velocity network trained for 90,000 iterations to sample the posterior because the moving averages of the test loss around 90,000 and 160,000 iterations are similar. 
\begin{figure}[!b]
   \centering
    \includegraphics[width=0.8\linewidth]{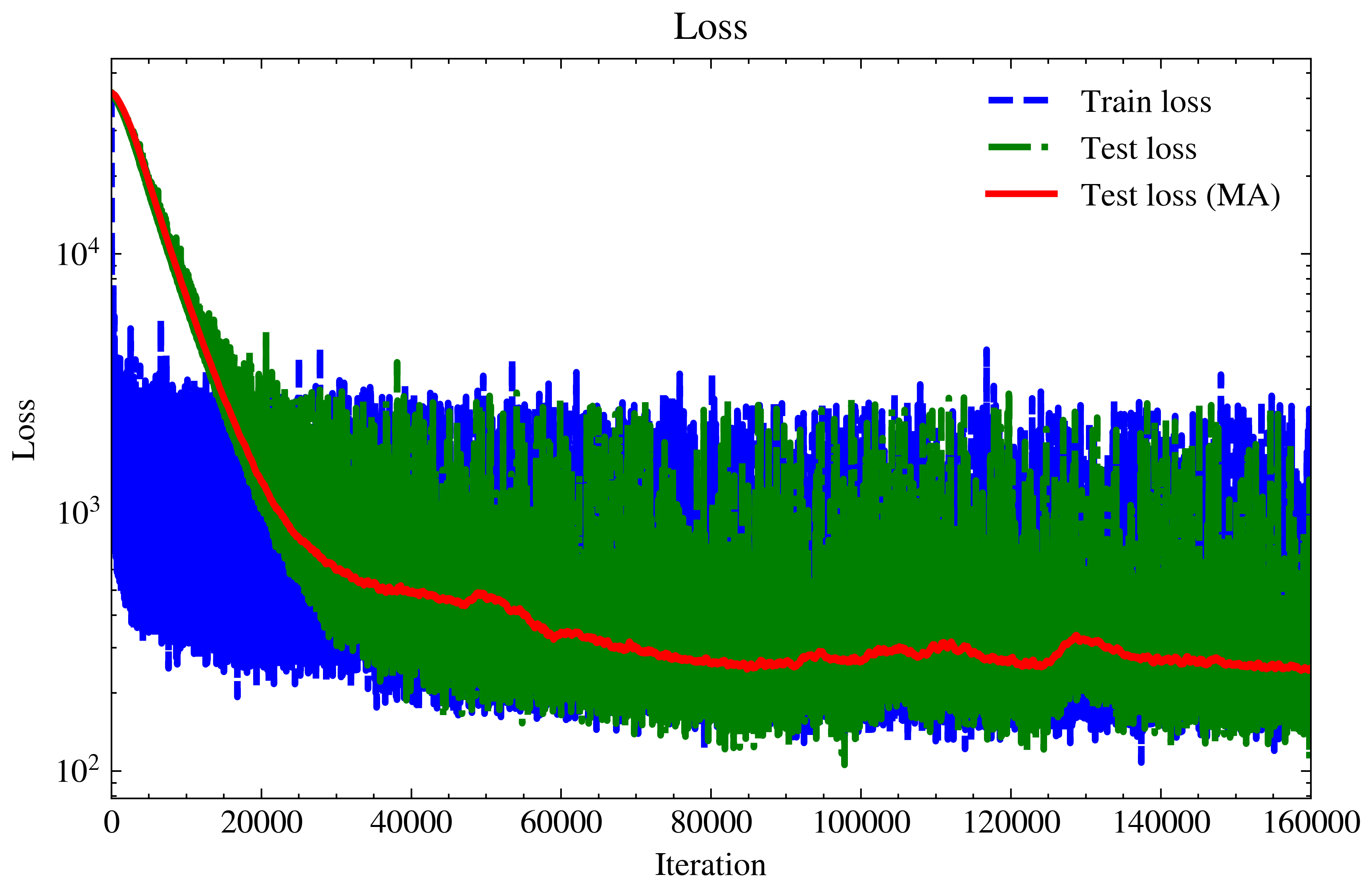}
   \caption{Training and test losses for the velocity network used in the tumor spheroid application, with moving average for the testing loss}
   \label{fig:tumor_losses}
\end{figure}

\paragraph{Validation on synthetic test data} We use the trained velocity network to sample 200 realizations of the Young's modulus field from the posterior distribution corresponding to two synthetic measurements from the test dataset, for which the corresponding true Young's modulus field (ground truth) is known. We note that it is possible to recover the Young's modulus field up to a multiplicative constant, the hydrogel modulus $E_{\text{hydrogel}}$, for which we assume a value of $10 \text{kPa}$ for all test cases.

The first two columns of \Cref{fig:tumor_results}, labeled as synthetic data, present the results for the examined synthetic cases, with the measurements shown in the first row, followed by three generated realizations of the Young's modulus field, followed by posterior statistics including the pixel-wise posterior mean and standard deviation, and finally the ground truth Young's modulus field shown in the last row. We observe that in both the synthetic test cases assessed, the estimated posterior mean captures the size and placement of the tumor spheroid. The estimated posterior mean also captures several of the stiffer nuclei present in the ground truth. Further examining the posterior standard deviation, we observe large uncertainties corresponding to the modulus field within each spheroid, and greater uncertainty associated with the boundaries of each nucleus. We can attribute such large uncertainties to the severely ill-posed nature of the inverse problem~\cite{ferreira2012uniqueness}. 
\begin{figure}
    \centering
    \includegraphics[width=0.88\linewidth]{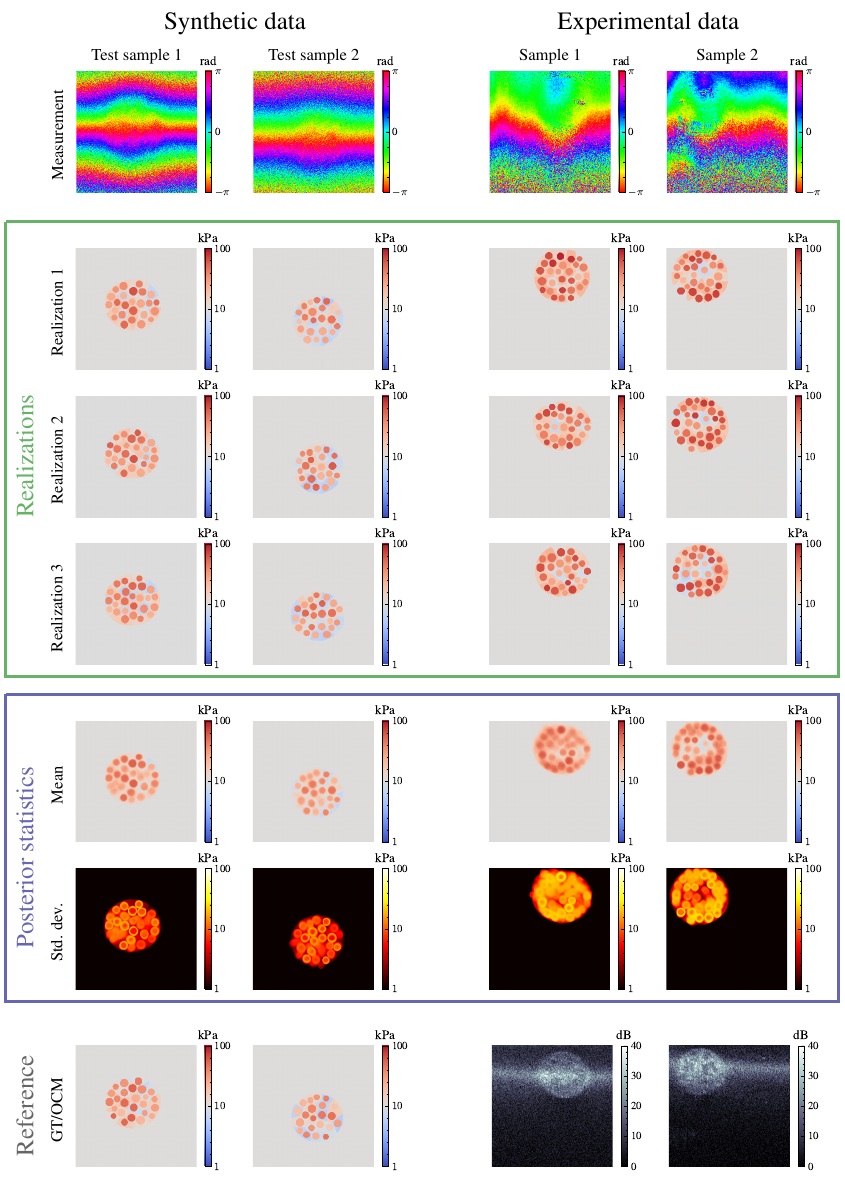}
    \caption{Posterior statistics estimated using the trained velocity network for select synthetic and experimental cases for the tumor spheroid application.}
    \label{fig:tumor_results}
\end{figure}

\paragraph{Application on experimental test data} The trained velocity network is then applied to experimentally-obtained measurements, for which we have access only to OCM images that were co-registered during the acquisition of the measured phase difference images for comparison. In this case, we note that the OCM images provide only coarse-scale information about the size and location of the tumor spheroids. Additionally, the experimental phase difference measurements are much noisier and more heterogeneous than the synthetically-generated measurements used for training. As a result of the increased noise in the measurement, we expect to observe higher uncertainty in the inferred Young's modulus field than in the synthetic cases. As with the synthetic cases, results are obtained up to the multiplicative constant $E_{\text{hydrogel}}$, for which we again use a value of $10 \text{kPa}$.

The third and fourth columns of \Cref{fig:tumor_results} present results for two experimental measurements. As before, the first row shows the conditional measurements, followed by three realizations, and the posterior statistics, including the pixel-wise mean and standard deviation. The final row shows the OCM images for each case, which we can use to assess the location and size of the tumor spheroids. We note that the generated samples are more diverse than in the synthetic cases, and as a result, the estimated posterior mean appears fuzzier, with less distinct nuclei indicated. This higher level of uncertainty can be attributed to the noisier measurements used for conditioning, and is also apparent in the posterior standard deviation images. However, comparing the OCM images with the estimated posterior mean, we remark that the conditional flow matching approach reproduces the locations and sizes of the tumor spheroids from experimental measurements despite the velocity network being trained on synthetic data.

\Cref{tab:tumor_test_results} additionally provides the average number of sampling steps required for the synthetic and experimental test cases examined here. For each case the average number of steps is computed across all posterior samples. We see that the number of steps is similar across both synthetic and experimental cases. Again, the number of steps is relatively small even though in this case $d=D=256\times256$, which indicates that posterior sampling can be very efficient for high-dimensional problems. In contrast,  \citet{dasgupta2025conditional} report using 3400 sampling steps to generate realizations from the target posterior using a discrete version of the conditional diffusion model for the same problem. 
\begin{table}[H]
    \caption{Average number of sampling steps for the tumor spheroid application}
    \label{tab:tumor_test_results}
    \centering
    \setlength\tabcolsep{0pt}
    \begin{tabular*}{0.6\textwidth}{@{\extracolsep{\fill}}*{3}{c}}
        \toprule
        \multirow{2}{*}{Measurement Type} & \multicolumn{2}{c}{Avg. number of sampling steps} \\
        \cline{2-3}
         & Test sample 1 & Test sample 2 \\
        \midrule
        Synthetic & 12 & 12 \\
        Experimental & 13 & 12\\
        \bottomrule
    \end{tabular*}
\end{table}

\section{Conclusions}\label{sec:conclusions}
In this work, we investigated the use of conditional flow matching for solving physics-constrained Bayesian inverse problems. By casting posterior inference as a conditional generative modeling task, the proposed framework learns a velocity field that transports samples from a chosen source distribution directly to the posterior distribution conditioned on observed measurements. The formulation requires only samples from the joint distribution of inferred variables and measurements and interfaces with the forward model in a black-box manner. As a result, it accommodates nonlinear, high-dimensional, and potentially non-differentiable forward operators without imposing restrictive assumptions on the likelihood.

An important contribution of this study is a theoretical and numerical examination of the behavior of the learned velocity field in the regime of finite training data. We demonstrated that, when trained beyond the point of optimal generalization, the velocity network can induce degenerate posterior distributions. Depending on the functional structure used to represent the conditioning variable, this degeneracy manifests either as variance collapse or as selective memorization of training samples. A simplified analysis elucidates the mechanism underlying this behavior and its connection to overfitting in regression problems. We also showed that monitoring the test loss and terminating training according to standard early-stopping criteria effectively mitigates these pathologies.

Through a series of benchmark and physics-based inverse problems, including multimodal conditional density estimation, a one-step data assimilation problem, and PDE-constrained parameter identification, we demonstrated that conditional flow matching can accurately approximate complex posterior distributions. We also examined the influence of different source distributions and found that their impact on accuracy and sampling efficiency can be problem dependent. Overall, the results establish conditional flow matching as a flexible, scalable, and practically effective approach for amortized Bayesian inference in physics-constrained inverse problems.

\section{Acknowledgments}
	
This work was initiated while AD was a postdoctoral researcher at the University of Southern California. The authors acknowledge support from ARO grant W911NF2410401 and ARO cooperative agreement W911NF-25-2-0183. The authors also acknowledge the Center for Advanced Research Computing (CARC, \href{https://carc.usc.edu}{carc.usc.edu}) at the University of Southern California for providing computing resources that have contributed to the research results reported within this publication. The authors also acknowledge Brendan Kennedy and Ken Foo from the University of Western Australia for their contribution to the results presented in \Cref{subsubsec:TS}. AD was additionally supported by the John von Neumann Fellowship at Sandia National Laboratories. This material is also based upon work supported by the U.S. Department of Energy, Office of Science, Office of Advanced Scientific Computing Research under Award Number 25-028431. 

Sandia National Laboratories is a multi-mission laboratory managed and operated by National Technology \& Engineering Solutions of Sandia, LLC (NTESS), a wholly owned subsidiary of Honeywell International Inc., for the U.S. Department of Energy’s National Nuclear Security Administration (DOE/NNSA) under contract DE-NA0003525. This written work is authored by an employee of NTESS. SAND2026-19587O. The employee, not NTESS, owns the right, title and interest in and to the written work and is responsible for its contents. Any subjective views or opinions that might be expressed in the written work do not necessarily represent the views of the U.S. Government. The publisher acknowledges that the U.S. Government retains a non-exclusive, paid-up, irrevocable, world-wide license to publish or reproduce the published form of this written work or allow others to do so, for U.S. Government purposes. The DOE will provide public access to results of federally sponsored research in accordance with the DOE Public Access Plan \url{https://www.energy.gov/downloads/doe-public-access-plan}.

\appendix
\setcounter{section}{0}
\renewcommand{\appendixname}{Appendix}
\renewcommand{\thesection}{Appendix~\Alph{section}}
\renewcommand{\thesubsection}{\Alph{section}\arabic{subsection}}
\renewcommand{\thefigure}{\Alph{section}\arabic{figure}}
\setcounter{figure}{0}
\renewcommand{\thetable}{\Alph{section}\arabic{table}}
\setcounter{table}{0}

\section{Details of the architecture for the velocity networks and training hyperparameters}\label{app:exp-settings}

\paragraph{Velocity networks modeled using MLPs} We use MLPs to model the velocity networks for the numerical examples in \Cref{subsec:cond-density,subsubsec:one-step-DA,subsubsec:ADR}. Following a previous study~\cite{dasgupta2026unifying}, we encode dependence on time $t$ using the Fourier features $[t - 0.5, \cos (2\pi t), \sin(2 \pi t), -\cos(4 \pi t)]$ which are concatenated with the spatial inputs $\X$ and $\Y$ before the first hidden layer. \Cref{tab:vel-network-MLP} provides additional details regarding the architectures of the velocity networks used for the various problems. 
\begin{table}[H]
    \caption{Details of the architecture of the velocity network for the numerical examples in \Cref{subsec:cond-density,subsubsec:one-step-DA,subsubsec:ADR}}
    \label{tab:vel-network-MLP}
    \centering
    \begin{tabular}{l c c c}
        \toprule
        \makecell[l]{Dataset/\\Inverse Problem} & \makecell{Width of\\hidden layers} & \makecell{Number of\\hidden layers} & \makecell{Activation\\function} \\
        \midrule
        Spiral (\Cref{subsec:cond-density}) & ~32 & 3 & ReLU \\
        One step data assimilation (\Cref{subsubsec:one-step-DA}) & 256 & 4 & ReLU \\ 
        Advection-diffusion-reaction (\Cref{subsubsec:ADR}) & 256 & 5 & ReLU \\
        \bottomrule
    \end{tabular}
\end{table}

\paragraph{Velocity networks modeled using DDPM-inspired U-Nets}
We model the velocity network using a DDPM-inspired U-Net~\cite{ho2020denoising} for the inverse problems discussed in 
\Cref{subsubsec:sci,subsubsec:elasto2,subsubsec:TS}. 
The U‑Net architecture consists of a encoder–decoder structure with skip connections between matching spatial scales. First, we append the input $\x_t$ with the measurements $\y$ along the channel dimension to introduce conditioning on the measurements. Next, an input block increases the number of channels in the input by a user-specified factor; we refer to the number of channels in the output of the input block to the number of model channels. The remainder of the encoder block contains multiple time-conditioned ResNet blocks that process spatial features at each resolution. Between every block the spatial resolution is halved, down to a resolution of 8 ultimately. These ResNet blocks are conditioned on the time $t$ via learned embeddings that are projected and then injected using scale and shift modulation. A bottleneck stage follows the encoder block at the lowest resolution and uses a ResNet block before upsampling begins. The decoder mirrors the encoder with upsampling and skip-feature concatenation, as is typical in U-Net architectures. The decoder ends with a convolutional head that predicts the velocity field with the appropriate number of channels. Additionally, we insert self-attention blocks at intermediate resolutions to help capture long-range spatial dependencies. \Cref{tab:vel-network-UNet} provides a few details regarding the U-Net architecture so that our experiments can be replicated. We refer interested readers to \cite{ho2020denoising} for additional details regarding the U-Net architecture. 
\begin{table}[t]
    \centering
    \caption{Details of the architecture of the velocity network for the numerical examples in 
    \Cref{subsubsec:sci,subsubsec:elasto2,subsubsec:TS}}
    \label{tab:vel-network-UNet}
    \begin{tabular}{l c c c c}
        \toprule
        \makecell[l]{Inverse Problem/\\Application} & \makecell{Number of\\residual blocks} & \makecell{Number of\\model channels} & \makecell{Channel\\multiplier} & \makecell{Attention\\resolution} \\
        \midrule
        \makecell[l]{Quasi-static\\elastography (\Cref{subsubsec:sci})} & 2& 128&[1, 2, 3, 4] &16\\ 
        \makecell[l]{Quasi-static\\elastography (\Cref{subsubsec:elasto2})} & 2& 128 & [1, 2, 3, 4] &8 \\
        Tumor spheroids (\Cref{subsubsec:TS}) & 2 & 128 & [1, 1, 2, 2, 4, 4] & 16 \\
        \bottomrule
    \end{tabular}
\end{table}

\paragraph{Additional details regarding training velocity networks} \Cref{tab:training-hyper-parameters} provides details of the training-related hyper-parameters we use for the various experiments. \Cref{tab:compute_cost_breakdown} provides an estimate of the wall times, which serves as a proxy for the computational cost, for completing 100 training iterations of the velocity network for the various examples presented in \Cref{sec:results} and the type of resource used for training. We only report the time necessary for training because sampling times are a very small fraction of the total wall time for training. 
\begin{table}[H]
    \centering
    \caption{Training hyper-parameters for training velocity networks for conditional flow matching}
    \label{tab:training-hyper-parameters}
    \begin{tabular}{l c c c c}
        \toprule
        \makecell[l]{Dataset/Inverse\\Problem/Application} & \makecell{Learning\\Rate} & \makecell{Batch\\Size} & \makecell{Number of\\ iterations} & \makecell{EMA\\coefficient}\\
        \midrule
        Spiral (\Cref{subsec:cond-density}) & 0.001~~ & 1000 & 100,000 & 0.9999 \\
        One step data assimilation (\Cref{subsubsec:one-step-DA})& 0.001~~ & ~500 & 100,000 & 0.9000 \\ 
        Advection-diffusion-reaction (\Cref{subsubsec:ADR}) & 0.001~~ & 1000 & ~~20,000 & 0.9999 \\
        Quasi-static elastography (\Cref{subsubsec:sci}) &0.0001 & ~256 & ~~60,000 & 0.9999 \\ 
        Quasi-static elastography (\Cref{subsubsec:elasto2}) & 0.0001& ~256 & 150,000 & 0.9999\\
        Tumor spheroids (\Cref{subsubsec:TS}) & 0.0001 & ~~16 & 160,000 & 0.9999 \\
        \bottomrule
    \end{tabular}
\end{table}
\begin{table}[H]
    \centering
    \caption{Approximate computational cost of training velocity networks for 100 iterations}
    \label{tab:compute_cost_breakdown}
    \begin{tabular}{@{\extracolsep{2pt}} l cc}
        \toprule
        \makecell[l]{Dataset/Inverse\\problem/Application} & \makecell{Approximate\\Wall time} & \makecell{Resource Type\\(NVIDIA GPU)} \\
        \midrule
        Spiral (\Cref{subsec:cond-density}) & 0.4 s & \makecell{Quadro RTX 8000} \\
        One step data assimilation (\Cref{subsubsec:one-step-DA}) & 0.4 s & \makecell{V100} \\
        Advection-diffusion-reaction (\Cref{subsubsec:ADR}) & 0.04 s & \makecell{L40s} \\
        Quasi-static elastography (\Cref{subsubsec:sci}) & 215 s&A100-80GB  \\
        Quasi-static elastography (\Cref{subsubsec:elasto2}) & 160 s & \makecell{A100-80GB} \\
        Tumor spheroids (\Cref{subsubsec:TS})& 100 s &  \makecell{A100-80GB} \\
        \bottomrule
    \end{tabular}
\end{table}

\section{Additional results on the effects of overfitting with finite training data}\label{app:finite-data}

We investigated the effects of overfitting the velocity network to a small amount of data in \Cref{subsec:overfitting-experiment}. In this appendix, we further investigate the effects of overfitting with a larger velocity network and more training data. 

\begin{figure}[!b]
    \centering
    \includegraphics[width=0.7\linewidth,trim={200 20 0 0},clip]{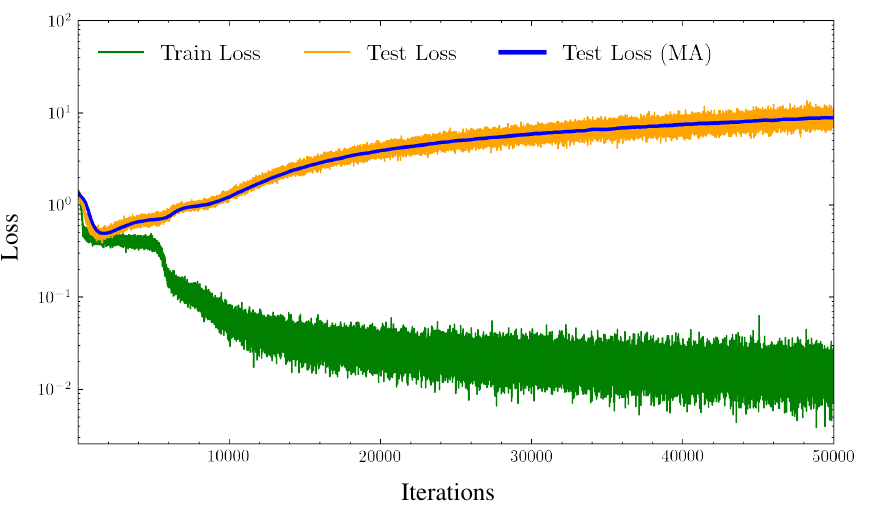}
    \caption{Train and test loss curve for the velocity network with 3 hidden layers each of width 64 trained using the training data shown in \Cref{fig:conditional_memorization_data_and_loss}(a). The blue curve above shows the moving average (MA) of the test loss}
    \label{fig:conditional_memorization_data_and_loss_size}
\end{figure}
\subsection{Increasing the size of the network}
First, we use the same training data as in \Cref{subsec:overfitting-experiment} but increase the size of the velocity network by doubling the width of each layer from 32 in \Cref{tab:vel-network-MLP} to 64. This yields a velocity network with more learnable parameters and, therefore, greater capacity. As before, we use the standard normal distribution, i.e., $Z \sim \mathcal{N}(0,1)$, as the source distribution. The training and test losses for the velocity network are shown in \Cref{fig:conditional_memorization_data_and_loss_size}. As in \Cref{subsec:overfitting-experiment}, we observe that the test loss initially decreases, attains a minimum at approximately 2{,}000 iterations, and subsequently increases. Therefore, the nature of the loss curve in \Cref{fig:conditional_memorization_data_and_loss_size} is similar to \Cref{fig:conditional_memorization_data_and_loss}(b). However, the inflection in the moving average of the test loss happens sooner for the larger network. This is consistent with overfitting behavior observed in regression problems, wherein an over-parameterized function exhibits stronger overfitting compared to an under-parameterized function; it is well known that under-parameterization acts as a source of regularization. 

\begin{figure}[!t]
    \centering
    \includegraphics[width=\linewidth]{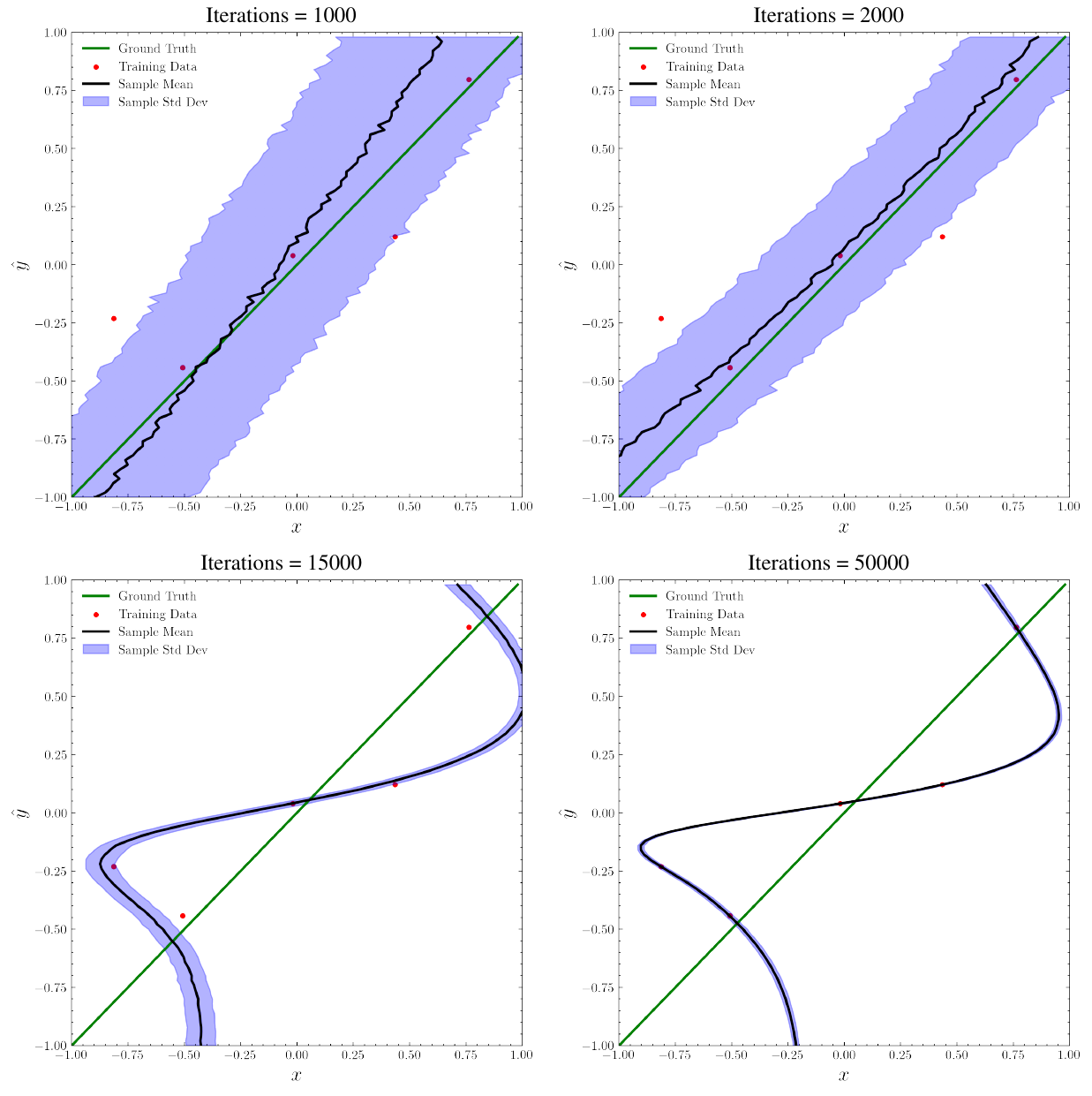}
    \caption{Mean and one-standard-deviation interval of the conditional distribution $\rho_{X|Y}$ estimated using samples generated by the trained velocity network with 3 hidden layers each of width 64 at different stages of training}
    \label{fig:conditional_memorization_posterior_size}
\end{figure}
Next, we use the trained velocity network (trained for a fixed number of iterations) to sample the conditional distribution  $\rho_{X \mid Y}$ for different values of $Y$. In each case, we estimate the mean and standard deviation from the samples. \Cref{fig:conditional_memorization_posterior_size} shows the estimated mean and the one-standard-deviation interval of the realizations as functions of $Y$. As in \Cref{fig:conditional_memorization_posterior}, we observe that as the number of training iterations increases, the standard deviation decreases for all values of $Y$. It is also noteworthy that the network learns a dependence of $\mathbb{E}[X \mid Y]$ on $Y$ similar to that shown in \Cref{fig:conditional_memorization_posterior}, though the reasons for this behavior warrant further investigation beyond the scope of the present work. We also note two differences between \Cref{fig:conditional_memorization_posterior,fig:conditional_memorization_posterior_size}. First, after 1000 iterations, the standard deviation of the posterior estimated using the velocity network with more parameters is smaller than that of the original velocity network. This is consistent with the earlier inflection in the test loss for the over-parameterized velocity network. The standard deviation estimated using the velocity networks trained for 15,000 iterations also shows similar behavior. Nonetheless, the estimated standard deviation vanishes in both cases when the networks are severely trained (50,000 iterations).

\begin{figure}[!b]
    \centering
    \includegraphics[width=\linewidth]{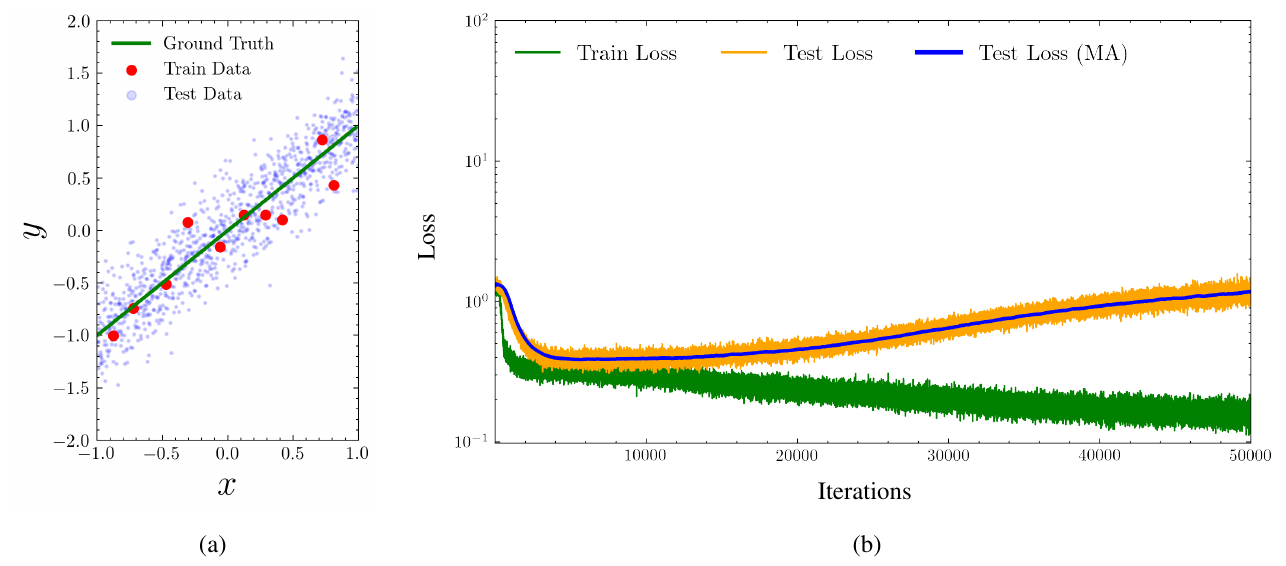}
    \caption{(a) Train and test data for the toy example used to illustrate the effects of overfitting with more data. (b) Train and test loss curve for the velocity network trained using the training data shown in \Cref{fig:conditional_memorization_data_and_loss_data}(a). The blue curve above shows the moving average (MA) of the test loss}
    \label{fig:conditional_memorization_data_and_loss_data}
\end{figure}
\subsection{Increasing the number of training data points}
Next, we use the same velocity network as in \Cref{subsec:overfitting-experiment} (meaning the size of the network is the same as we report in \Cref{tab:vel-network-MLP}) but increase the number of training data points from 5 to 10. \Cref{fig:conditional_memorization_data_and_loss_data}(a) shows the training data points, 1000 test samples, and the curve $Y=X$. Like \Cref{subsec:overfitting-experiment}, we use the standard normal distribution, i.e., $Z \sim \mathcal{N}(0,1)$, as the source distribution. The training and test losses for the velocity network are shown in \Cref{fig:conditional_memorization_data_and_loss_data}(b). In this case, we observe that the test loss initially decreases, attains a minimum around 4,000 iterations, saturates and then slowly increases up until to 10,000 iterations, and increases sharply beyond that. Significantly, the introduction of additional training data points delays the point at which the moving average of the test loss attains a minimum value. In this case, the moving average of the test loss attains its minimum value close to 4,000 iterations compared to 3,000 iterations with 5 training data points. Moreover, once the test loss reaches its minimum, it rises more sharply with five training data points than in the present case. Again, this behavior is expected as the introduction of additional training data points helps delay the onset of overfitting.

\begin{figure}[!t]
    \centering
    \includegraphics[width=\linewidth]{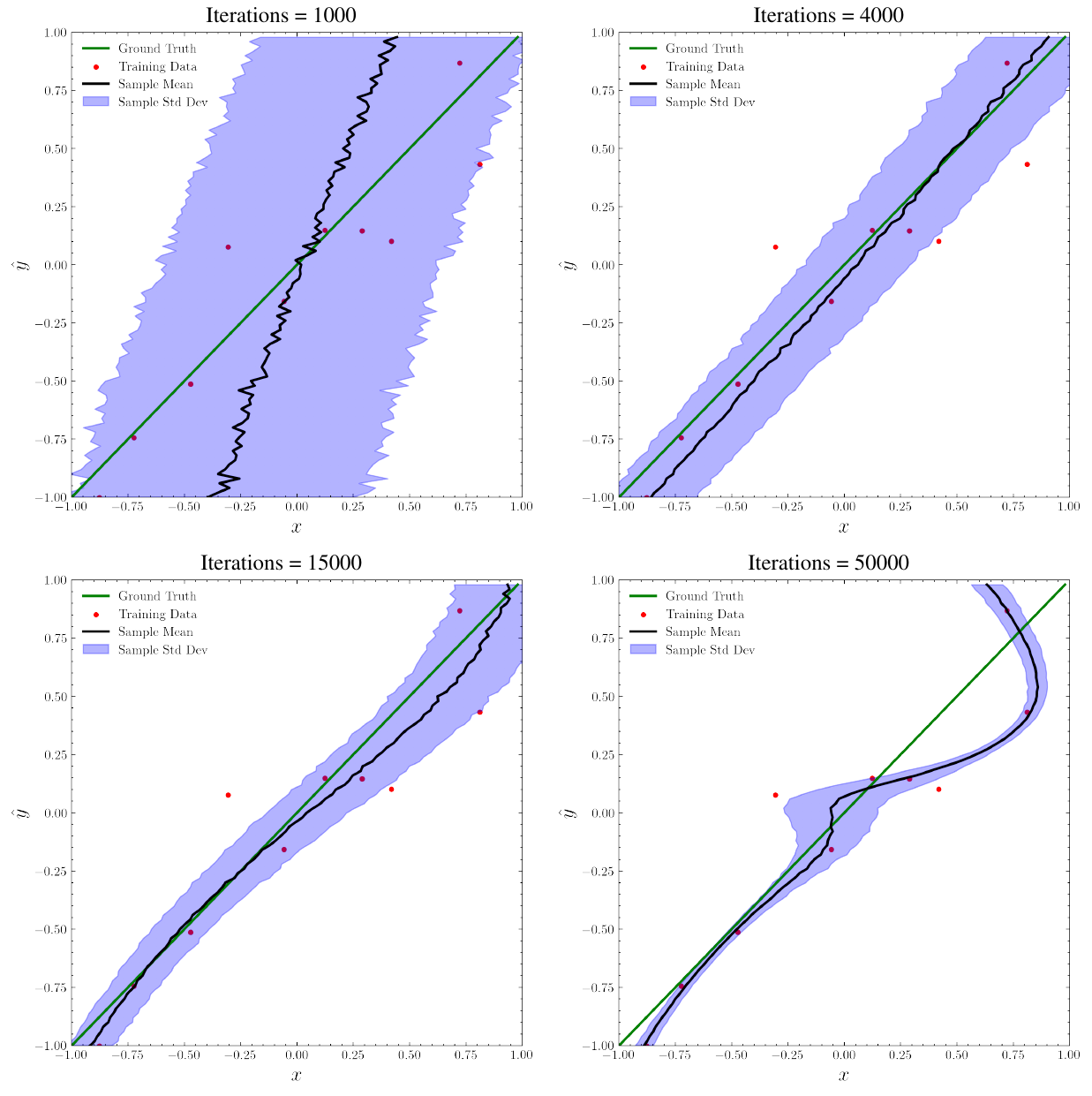}
    \caption{Mean and one-standard-deviation interval of the conditional distribution $\rho_{X|Y}$ estimated using samples generated by the trained velocity network at different stages of training and using more training data}
    \label{fig:conditional_memorization_posterior_data}
\end{figure}
Next, we use the trained velocity network (trained for a fixed number of iterations) to sample the conditional distribution  $\rho_{X \mid Y}$ for different values of $Y$, and subsequently estimate the mean and standard deviation from the samples. \Cref{fig:conditional_memorization_posterior_data} shows the estimated mean and the one-standard-deviation interval of the realizations as functions of $Y$. Like \Cref{fig:conditional_memorization_posterior,fig:conditional_memorization_data_and_loss_size}, we observe from \Cref{fig:conditional_memorization_posterior_data} that the standard deviation decreases for all values of $Y$ as the number of training iterations increases. A good approximation to the posterior is achieved at 4,000 iterations, near the point at which the moving average of the test loss attains its minimum value. Also, the bias in the estimated mean obtained using the velocity network trained for 4,000 iterations with 10 training data points is lower than the velocity network trained for 3,000 iterations with 5 training data points. Interestingly, the standard deviation does not disappear for all values of $Y$ when the velocity network is trained for 50,000 iterations in this case. The bottom-right subplot of \Cref{fig:conditional_memorization_posterior_data} reveals the local nature of overfitting: the estimated standard deviation is small in regions with sparse training data coverage, but comparatively larger near the cluster of training points in the range $0.25 \leq x \leq 0.50$. 

\section{Details of the data assimilation problem}\label{app:da-details}

We consider the discrete-time probabilistic state–space model 
\begin{equation}\label{eq:DA-SSM}
    \bm{x}_k \sim \prob{\X_k|\X_{k-1}}(\x_k|\x_{k-1}), \quad 
    \bm{y}_k \sim \prob{\Y_k|\X_k}(\y_k|\x_{k})
\end{equation}
where $\x_k \in \mathbb{R}^{d}$ denotes an $d$-dimensional state vector at data assimilation step $k \in \mathbb{Z}^{+}$, and $\y_k \in \mathbb{R}^{D}$ denotes the corresponding $D$-dimensional observation. The distribution $\prob{\X_k | \X_{k-1}}(\x_k|\x_{k-1})$ characterizes the evolution of the state vector, while $\prob{\Y_k|\X_k}(\y_k|\x_{k})$ defines the observation likelihood. Given the state–space model in \Cref{eq:DA-SSM}, Bayesian filtering provides the filtered or posterior distribution 
\begin{equation}\label{eq:Bayesian-filter-update}
    \prob{\X_k|\Y_{1:k}}(\x_k|\y_{1:k}) \propto \prob{\Y_k|\X_{k}}(\y_k|\x_{k}) \prob{\X_{k}|\Y_{1:k-1}}(\x_k | \y_{1:k-1}) 
\end{equation}
where $\y_{1:k}$ denotes the realized observations up to data assimilation step $k$. Also, in \Cref{eq:Bayesian-filter-update},
\begin{equation}\label{eq:Bayesian-filter-propagate}
    \prob{\X_{k}|\Y_{1:k-1}}(\x_k | \y_{1:k-1}) = \int \prob{\X_k|\X_{k-1}}(\x_k|\x_{k-1}) \prob{\X_k|\Y_{1:k-1}}(\x_k|\y_{1:k-1})\mathrm{d}\x_k
\end{equation}
acts as the prior distribution. For nonlinear and non-Gaussian systems, \Cref{eq:Bayesian-filter-update} is generally intractable, and practical data assimilation methods therefore approximate the posterior distribution $\prob{\X_k|\Y_{1:k}}(\x_k|\y_{1:k})$ in \Cref{eq:Bayesian-filter-update} and the prior distribution $\prob{\X_{k}|\Y_{1:k-1}}(\x_k | \y_{1:k-1})$ in \Cref{eq:Bayesian-filter-propagate} using an ensemble of realizations or particles. 

For this problem, we consider one step of the Lorenz-63 system~\cite{DeterministicNonperiodicFlow} --- a widely used benchmark in data assimilation due to its nonlinear, low-dimensional chaotic dynamics. The dynamics are governed by the following equations:
\begin{equation}\label{eq:DA-L63}
    \dot{x}_1 = \sigma (x_2-x_1), \; 
    \dot{x}_2 = \rho x_1 - x_2 - x_1 x_3, \;
    \dot{x}_3 = x_1 x_2 - \beta x_3,
\end{equation}
where $\dot{x}$ denotes the derivative derivative of the state variable with respect to physical time. We set the parameter values $\sigma=10$, $\rho=28$, and $\beta=8/3$, which ensures the system is in a chaotic regime~\cite{DeterministicNonperiodicFlow}. We use $[-1.27323174, -0.00702107,  0.74486393]$ as the initial condition, where each component was sampled independently from $\mathcal{N}(0, 1)$. We integrate \Cref{eq:DA-L63} using a simple forward Euler integration scheme with a step size of 0.01 s and zero-mean Gaussian process noise with covariance matrix $0.01^2\mathbb{I}_3$. We also consider the following observation operator (same as \Cref{eq:DA-obs_op}):
\begin{equation}\label{eq:DA-obs_op2}
    y = x_3 + \epsilon,
\end{equation}
where $\epsilon \sim \mathcal{N}(0, 0.5^2)$ denotes the measurement noise. We assume the observations are made at regular intervals of 0.1 s. Appropriate discretization of \Cref{eq:DA-L63} for forward Euler integration, coupled with the additive process noise, and the observation operator \Cref{eq:DA-obs_op2} yields a state-space model consistent with \Cref{eq:DA-SSM}. Note that observations are made every 10 integration steps for the dynamical system. 

Herein, we use $\x = [x_1, x_2, x_3]^T$ to denote a realization of the state vector $\X$ at data assimilation step $k$. So, $\X$ is three-dimensional, \ie $d=3$, while the observation $Y$ is a scalar, \ie $D=1$. Since the state is partially observed, the posterior distribution can exhibit significant non-Gaussian structures like bimodality. This study focuses on a single data assimilation step, at $k=3$, chosen to yield a nontrivial transformation where the prior distribution is unimodal while the conditioning on the observation induces a bimodal posterior; see \Cref{fig:DA-distributions}. 


\bibliographystyle{elsarticle-num-names} 
\bibliography{references}






\end{document}